\documentclass{article}
\usepackage[table]{xcolor}

\usepackage{microtype}
\usepackage{graphicx}
\usepackage{subcaption}

\usepackage{booktabs} %
\usepackage{multirow}
\usepackage{tabularx}
\usepackage{colortbl}
\usepackage{pgf}           %
\usepackage{hyperref}
\usepackage{enumerate}   
\usepackage{placeins}

\usepackage{xspace}

\newcommand{\stok}{\textsc{StochasTok}\xspace}
\newcommand{\stokuni}{\textsc{StochasTok-uni}\xspace}
\newcommand{\unik}{\textsc{Uniform-k}\xspace}
\newcommand{\uni}{\textsc{Uniform}\xspace}
\newcommand{\bpedropout}{\textsc{BPE-Dropout}\xspace}

\newcommand{\ppre}{\alpha_{\text{pre}}}
\newcommand{\pfine}{\alpha_{\text{fine}}}
\newcommand{\peval}{\alpha_{\text{eval}}}
\newcommand{\picl}{\alpha_{\text{ICL}}}

\newcommand{\tinyllm}{\textcolor{scGreen}{\bf\texttt{Tiny-LLM}}\xspace}
\newcommand{\gptxl}{\textcolor{scYellow}{\bf\texttt{GPT2-xl}}\xspace}
\newcommand{\llamaoneb}{\textcolor{scBlue}{\bf\texttt{Llama-1b}}\xspace}
\newcommand{\llamaeightb}{\textcolor{scPurple}{\bf\texttt{Llama-8b}}\xspace}
\newcommand{\gemmaoneb}{{\bf\texttt{Gemma-1b}}\xspace}
\newcommand{\qwenzerosixb}{\textcolor{scCyan}{\bf\texttt{Qwen-0.6b}}\xspace}

\newcommand{\langgame}{{\textsc{Language Game}}\xspace}
\newcommand{\cute}{{\textsc{Cute}}\xspace}

\newcommand{\norm}[1]{\left\lVert #1 \right\rVert}

\newcommand{\cV}{\mathcal{V}}
\newcommand{\R}{\mathbb{R}}

\usepackage[preprint]{icml2026}

\usepackage[utf8]{inputenc} %
\usepackage[T1]{fontenc}    %
\usepackage{url}            %

\usepackage{natbib}

\usepackage{xspace}
\newcommand{\eg}{\textit{e.g.}\xspace}
\newcommand{\ie}{\textit{i.e.}\xspace}
\newcommand{\cf}{\textit{cf.}\xspace}
\newcommand{\wrt}{w.r.t.\xspace}

\newcommand{\tok}[2]{%
  \tikz[baseline=(tok.base)]\node[fill=#1!40,inner sep=1pt,font=\strut](tok){#2};%
}

\renewcommand{\epsilon}{\varepsilon}

\usepackage{shortex-light}
\usepackage{fontawesome5}

\usepackage{safecolours}

\colorlet{scLightCyan}{scCyan!60}
\colorlet{scLightPurple}{scPurple!60}
\colorlet{scLightGreen}{scGreen!60}
\colorlet{scLightRed}{scRed!60}
\colorlet{scLightYellow}{scYellow!60}

\usepackage{tikz}          %
\usetikzlibrary{patterns}

\usepackage{pgfplots}      %
\pgfplotsset{compat=newest}
\pgfplotsset{unbounded coords=jump}
\usepgfplotslibrary{groupplots} %
\usepgfplotslibrary{colorbrewer}

\NewDocumentCommand{\printmynumber}{ m }
{%
  \pgfmathfloatifflags{#1}{3}
  {}%
  {\pgfmathprintnumber{#1}}%
}

\usepackage[table]{xcolor}        %
\usepackage{graphicx}      %

\usepgfplotslibrary{external}

\newlength\figureheight
\newlength\figurewidth
\icmltitlerunning{Stochasticity in Tokenisation Improves Robustness}

\begin{document}

\twocolumn[
\icmltitle{Stochasticity in Tokenisation Improves Robustness}

\begin{icmlauthorlist}
\icmlauthor{Sophie Steger}{tugraz}
\icmlauthor{Rui Li}{aalto}
\icmlauthor{Sofiane Ennadir}{microsoft}
\icmlauthor{Anya Sims}{oxford}
\icmlauthor{Arno Solin}{aalto}
\icmlauthor{Franz Pernkopf}{tugraz}
\icmlauthor{Martin Trapp}{kth}
\end{icmlauthorlist}

\icmlaffiliation{tugraz}{Institute of Signal Processing and Speech Communication, Graz University of Technology, Graz, Austria}
\icmlaffiliation{aalto}{ELLIS Institute Finland \& Aalto University, Espoo, Finland}
\icmlaffiliation{microsoft}{King AI Labs, Microsoft Gaming}
\icmlaffiliation{oxford}{University of Oxford, Oxford, United Kingdom}
\icmlaffiliation{kth}{KTH Royal Institute of Technology, Stockholm, Sweden}

\icmlcorrespondingauthor{Sophie Steger}{sophie.steger@tugraz.at}
\icmlkeywords{Tokenisation, Large-language models}

\vskip 0.3in
]

\printAffiliationsAndNotice{}  %

\begin{abstract}
    The widespread adoption of large language models (LLMs) has increased concerns about their robustness. Vulnerabilities in perturbations of tokenisation of the input indicate that models trained with a deterministic canonical tokenisation can be brittle to adversarial attacks. Recent studies suggest that stochastic tokenisation can deliver internal representations that are less sensitive to perturbations. In this paper, we analyse how stochastic tokenisations affect robustness to adversarial attacks and random perturbations. We systematically study this over a range of learning regimes (pre-training, supervised fine-tuning, and in-context learning), data sets, and model architectures. We show that pre-training and fine-tuning with uniformly sampled stochastic tokenisations improve robustness to random and adversarial perturbations. 
    Evaluating on uniformly sampled non-canonical tokenisations reduces the accuracy of a canonically trained Llama-1b model by 29.8\%. We find that training with stochastic tokenisation preserves accuracy without increasing inference cost. 
\looseness-1

\end{abstract}

\section{Introduction}
\begin{figure}[t]
    \centering\footnotesize

    \setlength{\figureheight}{0.4\linewidth}
    \setlength{\figurewidth}{0.8\linewidth}

    \pgfplotsset{
        x tick label style={font=\scriptsize}, 
        y tick label style={rotate=90, font=\scriptsize}, 
        ylabel={\small Avg.\ test accuracy $\rightarrow$},
        xlabel={\small Normalised number of splits (amount of perturbation) $\rightarrow$},
        scale only axis,
        tick align=outside,
        tick pos=left,
        every axis plot/.append style={mark size=3pt},
        axis x line*=bottom,
        axis y line*=left,
        axis line style={draw=none},
        grid style={solid},
        yticklabel={
          \pgfmathparse{100*\tick}%
          \pgfmathprintnumber[fixed,precision=0]{\pgfmathresult}\%
        },
        legend style={
          legend columns=4,
          legend cell align=left,
          draw opacity=1,
          text opacity=1,
          at={(0.03,1.03)},
          anchor=south west,
          draw=none
        },
        }    
    \begin{tikzpicture}

\begin{axis}[
  height=\figureheight,
  width=\figurewidth,
  ybar,
  bar width=6pt,
  enlarge x limits=0.1,
  symbolic x coords={0,0.1,0.5,1,3},
  xtick=data,
  xtick style={color=black},
  ytick style={color=black},
  ymajorgrids,
  grid style={scGrey},
  legend cell align={left},
  legend style={
    fill opacity=1,
    draw opacity=1,
    text opacity=1,
    at={(0.03,1.03)},
    anchor=south west,
    draw=none
  },
]

\addplot+[very thick, draw=black, fill=black, bar shift=-12pt]
coordinates { (0,0.94) ({0.1},0.9284) ({0.5},0.8861) (1,0.8170) (3,0.8165) };
\addlegendentry{\textsc{Canon}}

\addplot+[very thick, draw=scRed!30, fill=scRed!30, bar shift=-4pt]
coordinates { (0,0.95) ({0.1},0.9438) ({0.5},0.9442) (1,0.9484) (3,0.9498) };
\addlegendentry{\textsc{STok}}

\addplot+[very thick, draw=scRed!60, fill=scRed!60, bar shift=4pt]
coordinates { (0,0.95) ({0.1},0.9456) ({0.5},0.9522) (1,0.9554) (3,0.9586) };
\addlegendentry{\textsc{STok-uni}}

\addplot+[very thick, draw=scRed!90, fill=scRed!90, bar shift=12pt]
coordinates { (0,0.948) ({0.1},0.9498) ({0.5},0.96) (1,0.9664) (3,0.9656) };
\addlegendentry{\textsc{Uni-k}}

\end{axis}
\end{tikzpicture}

  \caption{Effect of stochastic tokenisation during testing. As the level of stochasticity increases (increasing number of splits), the accuracy of \llamaoneb trained on \langgame with canonical tokenisation (\textsc{Canon}) sharply drops while the same model fine-tuned with stochastic tokenisation (\textsc{Stok}, \textsc{Stok-uni}, or \textsc{uni-k}) remains robust to perturbations during testing. 
  }
  \vspace{-1em}
  \label{fig:teaser}
\end{figure}
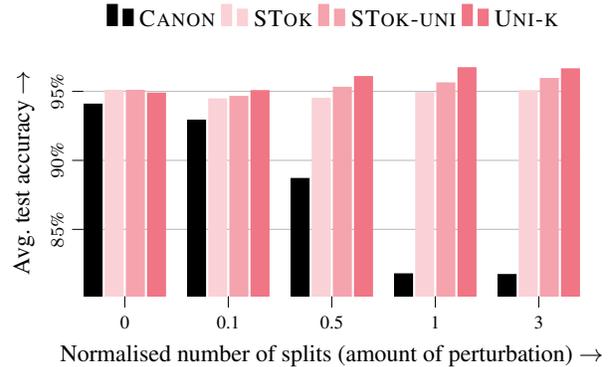
In recent years, large language models (LLMs) have demonstrated remarkable capabilities across a wide range of natural language processing tasks, leading to their rapid and widespread adoption \cite{achiam2023gpt, team2023gemini, grattafiori2024llama, brown2020language}.
As their use has increased, so too have concerns about their safety and robustness \cite{li2025security,dong2024attacks,geh2025adversarial}.
Various approaches have been proposed to tackle these safety concerns, including fine-tuning \cite{wei2022finetuned,agarwal2024onpolicy}, preference alignment \cite{rafailov2023direct,song2024preference}, and red teaming \cite{perez2022redteaming,deng2023attack}.
However, studying their robustness to perturbations in the inputs, for example, as adversarial attacks, remains underexplored.

Subword tokenisation, a standard component of modern LLMs, has recently attracted attention in the study of LLM robustness.
In LLMs, a tokeniser, such as Byte Pair Encoding (BPE) \citet{sennrich2016neural}, encodes text into a sequence of tokens through a typically deterministic encoding function.
Those token sequences then serve as the basis for all subsequent computations in LLMs.
This makes them a natural point of attack. 
Although decoding is a many-to-one mapping, text is typically encoded using a deterministic function that returns the canonical tokenisation. However, multiple token sequences can reconstruct the same string, referred to as non-canonical tokenisations. In standard training, the models are exposed to only the canonical tokenisation. 
Recently, \citet{zheng2025broken,kaplan2025from} investigated the robustness of pre-trained LLMs against non-canonical tokenisation and found that random non-canonical tokenisation generally significantly reduces performance. Further, \citet{geh2025adversarial} showed that alignment in LLMs can be broken by finding adversarial attacks in the space of tokenisations without changing the content of the query.
These results indicate that LLMs are surprisingly brittle with respect to non-canonical tokenisation, as also highlighted in \Cref{fig:teaser}. 
A potential remedy is to train with \emph{stochastic tokenisations}, which expose the model to alternative segmentations of the same string while keeping the vocabulary fixed. BPE, for example, retains all intermediate merges, enabling multiple valid token sequences per string. BPE-dropout \cite{provilkov2020bpe} introduces stochasticity by skipping merge rules at random. More recently, \citet{sims2026stochastok} proposed \stok, which randomly splits tokens into equivalent subtokens with some small probability. These works focus on improvements to subtoken-level understanding; consequently, the extent to which training under stochastic tokenisations affects robustness remains an open question.

In this work, we tackle the fundamental question: \textbf{Can stochastic tokenisation improve robustness?}
To this end, we investigate how stochastic tokenisation during training affects robustness against random and adversarial tokenisations.
We study various training regimes, including stochastic tokenisation during pre-training, fine-tuning, and in-context learning (ICL). 
In our study, we use \stok as the starting point, but find that \stok introduces a biased sampling distribution over tokenisations. 
Therefore, we devise \stokuni, which generates uniformly distributed tokenisations conditioned on a specific per-token split count, and \unik, which generates uniformly distributed tokenisations with a specified number of total splits. 
Both can represent a larger set of non-canonical tokenisations compared to \stok, \ie they have larger support. 
Our findings indicate that \stok helps generalise across random tokenisations during testing, but for worst-case (adversarial) evaluations, we observe a significant drop in test performance compared to \stokuni and \unik. 
Lastly, we provide theoretical insights into the adversarial robustness introduced by stochastic tokenisation during training.

Our contributions can be summarised as follows:
{\em(i)}~We study how training with stochastic tokenisation affects robustness against random and adversarial tokenisations (\cref{sec:robustness_analysis}).
{\em(ii)}~We analyse the sampling distribution of \stok, finding that it is biased towards certain tokenisations and introduce unbiased uniform sampling schemes (\cref{sec:methods}). 
{\em(iii)}~Lastly, we provide theoretical insights into adversarial robustness introduced by stochastic tokenisation (\cref{sec:adversarial}).

\section{Background}\label{sec:preliminaries}

We now introduce the necessary notation and formally define relevant concepts on tokenisation in LLMs.

\textbf{Notation.}\;
Let $\mcX$ denote an alphabet of characters $x_i$ (or bytes), and let $\bx=(x_1,\dots,x_n)\in\mcX^*$ be a string, \ie, a character sequence. The tokenisation of $\bx$ \wrt a token vocabulary $\mcV$ is a token sequence $\bv = (v_1,\dots,v_m)\in\mcV^*$ such that $v_1 \circ \dots \circ v_m = \bx$, where $\circ$ is the string concatenation. The vocabulary $\mcV$ is a finite set of string sequences $\mcV \subseteq \mcX^*$. 
We denote the encoding function $\tau\colon\mcX^*\to\mcV^*$ and the decoding function $\kappa\colon\mcV^*\to\mcX^*$, which satisfy $\kappa(\tau(\bx)) =\bx$ for all $\bx \in\mcX^*$.
\textbf{Non-canonical tokenisation.}\;
Subword tokenisation is the most prevalent tokenisation scheme in modern LLMs, mapping a sequence of characters to tokens. Which character sequences are merged is determined by the tokenisation scheme. For instance, byte-pair encoding \citep[BPE;][]{gage1994new} starts with a vocabulary that includes all individual characters. This vocabulary is then iteratively extended by merging the most frequent adjacent pairs of tokens in a given training corpus \citep{sennrich2016neural}. The procedure stops once a predetermined vocabulary size is reached.

However, this mapping is not unique. Given a token vocabulary $\mcV$, there may exist multiple token sequences that reconstruct the same underlying string sequence. 

\paragraph{Example.}\;
Consider the string $\bx={(\texttt{revolution})}$ and suppose the vocabulary
\begin{align*}
\mcV=\{\tok{scGrey}{\texttt{a}},\dots,\tok{scGrey}{\texttt{z}},\tok{scRed}{\texttt{re}},\tok{scYellow}{\texttt{v}},\tok{scPurple}{\texttt{ol}}, \tok{scCyan}{\texttt{ution}}, 
\tok{scBlue}{\texttt{\texttt{revolution}}}\}
\end{align*}
contains all single characters as well as the listed merged character sequences.
The canonical tokenisation is the single-token decomposition
$$
\bv^c=(\tok{scBlue}{\texttt{\texttt{revolution}}}).
$$
Other possible \textit{tokenisations} include
\begin{align*}
&\bv =(\tok{scRed}{\texttt{re}},\tok{scYellow}{\texttt{v}},\tok{scPurple}{\texttt{ol}},\tok{scCyan}{\texttt{ution}}),
\end{align*}
and, at the extreme, the character-level tokenisation
$$
\bv^{\text{char}}=(\tok{scGrey}{\texttt{r}},\tok{scGrey}{\texttt{e}},\tok{scGrey}{\texttt{v}},\tok{scGrey}{\texttt{o}},\tok{scGrey}{\texttt{l}},\tok{scGrey}{\texttt{u}},\tok{scGrey}{\texttt{t}},\tok{scGrey}{\texttt{i}},\tok{scGrey}{\texttt{o}},\tok{scGrey}{\texttt{n}}),
$$
all of which satisfy $v_1\circ\cdots\circ v_m=\bx$.

We denote the set of all \emph{valid tokenisations} \wrt{} a vocabulary $\mcV$ as $\mcT_\mcV(\bx) = \{ \bv : v_1 \circ \dots \circ v_m = \bx \}$, call $\bv^c=\tau(\bx)$ the \emph{canonical} tokenisation, and refer to any $\bv\in\mcT_{\mcV}(\bx)\setminus\{\bv^c\}$ as a \emph{non-canonical tokenisation}.

\textbf{\stok.}\;
Recently, \citet{sims2026stochastok} introduced a stochastic tokenisation scheme that is compatible with any tokeniser. It creates non-canonical tokenisations by randomly splitting tokens into pairs of smaller subtokens by following the merge rules reflected in the vocabulary.

Consider a string $\bx$ and its canonical tokenisation $\bv^c=(v_1,\dots,v_m)$, and define the  expansion proportion $\alpha$. Then, \stok aims to create a total of $N=\alpha \cdot m$ splits by iteratively splitting the canonical tokenisation, \ie, a subtoken $v_i$ is selected uniformly at random and then replaced with a pair of subtokens $(t_1,t_2)\in\mcV^2$. This procedure is repeated on the extended sequence $(v_1,\dots,v_{i-1},t_1,t_2,v_{i+1},\dots,v_m)$ for $N$ iterations. 

\textbf{Adversarial tokenisation.}\;
The goal of adversarial tokenisation is to search for tokenisations of a string $\bx$ that change an LLM's correct prediction, \ie, break an alignment of an LLM. As the number of non-canonical tokenisations grows exponentially with the length of the string $|\bx|$, an exhaustive enumeration is infeasible. Instead, \citet{geh2025adversarial} proposed performing an adversarial search over local neighbourhoods in the tokenisation space. 

Let $d(\cdot,\cdot)$ be a token-level edit distance (Levenshtein distance \cite{levenshtein1965binary}) with insertion cost $1$ and deletion cost $0$. 
For a reference tokenisation $\bv$, define the set of tokenisations at edit distance $i$ as $\mcT_\mcV^i(\bx,\bv)=\{ \bu: \bu\in\mcT_\mcV(\bx) \land d(\bv,\bu)=i\}$. The neighbourhood of $\bv$ is defined as $\mathrm{Ne}(\bv)=\mcT^2_\mcV(\bx,\bv)$ whose cardinality is bounded by $\mcO(|\bv|^2)$. 

In this work, we focus on multiple-choice question answering tasks. 
Let $\bx$ be a question string with tokenisation $\bv$, and $\{\ba_c\}_{c=1}^C$ be a set of answer strings with tokenisation $\{\bv_{\ba_c}\}_{c=1}^C$. The model assigns a score, \ie, 
\begin{equation}
  z_{c}(\bv)=p(\bv_{\ba_c} \mid \bv) ,
\end{equation}
to each answer option. When attacking, our aim is to maximise the margin of violation:
\begin{equation}
    \bv^\star = \argmax_{\bu\in\mcT_\mcV(\bx)}\left(\max_{c\neq y} z_c(\bu)-z_y(\bu)\right).
\end{equation}
where $y$ denotes the true class and $c$ the most probable wrong class.
To approximate this, \citep{geh2025adversarial} proposed a greedy search over a local neighbourhood. \Cref{alg:adv_tok} (\Cref{app:pseudocodes}) summarises the procedure used in this work.

\section{Related Work}

\textbf{Robustness to tokenisations.}\;
Typically, LLMs are trained with deterministic subword tokenisation \cite{sennrich2016neural, kudo2018subword}, which has been argued to obscure sub-token and character-level information \cite{chai2024tokenization,edman2024cute}. Recent work has demonstrated that models still encode within-word information without direct access to its characters \cite{itzhak2022models, hiraoka2025spelling}, with larger models performing better in general \cite{kaushal2022tokens}. Evidence suggests that models convert subword tokens into an implicit vocabulary, which allows models to retrieve character-level information \cite{feucht2024token, kaplan2025from}. In addition, \citet{zheng2025broken} demonstrate the robustness of canonically trained LLMs to unseen non-canonical tokenisations at test time, especially after instruction-tuning. Post-training on curated datasets with canonical tokenisation, however, has been argued to render models vulnerable to adversarial tokenisation attacks \cite{geh2025adversarial}.

\textbf{Training with stochastic tokenisations.}\; 
Several works have proposed to train with stochastic tokenisations to improve the model's subword understanding. Among the first, \citet{provilkov2020bpe} introduce stochasticity by randomly dropping merge rules. \citet{cognetta2024distributional} demonstrate that BPE-dropout \cite{provilkov2020bpe} and MaxMatch/Unigram \cite{wu2016google} result in non-uniform distribution over tokenisations, and show improvements in machine translation quality with uniform sampling. \citet{bauwens2025grampa} argue that uniform sampling leads to tokens with reduced character count and propose length-aware sampling. Most recently, \citet{sims2026stochastok} randomly split tokens into pairs of sub-tokens, showing improved sub-word understanding. In addition, results suggest improved robustness to non-canonical tokenisations.

\section{Does \stok Improve Robustness?} \label{sec:robustness_analysis}
In the following section, we investigate whether training with \stok instead of canonical tokenisation improves robustness to non-canonical tokenisations. Our aim is to provide a practical overview of the benefits of using stochastic tokenisation in different training regimes.
For our evaluation setup, we use the \langgame and \cute data sets from \citet{sims2026stochastok}. These data sets test sub-token-level understanding, such as counting or detecting letters within words, see \Cref{app:datasets} for examples. 
We assess the effect of stochastic tokenisation during (i) pre-training, (ii) supervised fine-tuning, and (iii) in-context learning. In particular, we evaluate the performance on canonical tokenisation during testing (clean case) and on stochastic tokenisation at various levels during testing, to assess the model's ability to solve the task in the ideal setting and its robustness to input perturbations.

\subsection{Does Pre-training Improve Robustness?} \label{subsec:pre-training}
First, we test whether \stok pre-training can instil both sub-token understanding and robustness. While choosing the pre-training scheme offers the highest flexibility, it also imposes a high computational burden. For this reason, we pre-train a small \tinyllm (50\,M-parameter) \cite{hillier2024super} model from scratch on OpenWebText and focus on fine-tuning and in-context learning for larger models. In \Cref{app:pretraining}, we report validation perplexity, accuracy on general multiple-choice benchmarks, and TokSuite \cite{altintacs2025toksuite} results after pre-training under both canonical and stochastic tokenisation.
Below, we follow \citet{sims2026stochastok} and perform canonical fine-tuning on \langgame to test whether sub-token understanding and robustness to alternative tokenisations persist. 
In particular, we evaluate the average accuracy for stochastic tokenisations sampled with \stok at increasing edit distances, normalised by the length of the canonical token sequence (see \Cref{eq:normalised_edit_distance}, \Cref{app:normalised_edit_distance}). Specifically, we used the range $\{0, 0.1, 0.5, 1, 3\}$, where zero corresponds to canonical, deterministic tokenisation. 
Note that all test-time stochastic tokenisations in this experiment are sampled with \stok. Hence, tokenisations are not sampled uniformly but reflect sampling biases induced by \stok. In \Cref{sec:methods}, we test the robustness to uniformly sampled tokenisations.

In \Cref{fig:robustness_finetuning_tiny_langgame} \textit{(left)}, we see that canonical pre-training, \ie, expansion proportion of $\ppre=0.0$, is insufficient to produce useful representations for solving these subtoken-level tasks, which is consistent with findings in \citet{sims2026stochastok}, and pre-training with stochastic tokenisation ($\ppre>0.0$) substantially improves accuracy.
Further, across all tokenisation schemes, we observe that accuracy degrades as the normalised number of splits increases. 
However, accuracy with stochastic pre-training remains substantially higher compared to canonical pre-training. 
Results for \cute exhibit the same trend and are shown in \Cref{app:fig:avg_acc_tinyllm_cute}, \Cref{app:exp-pretraining-finetuning}.  We note that although stochastic pre-training performs favourably in the clean case, it is not clear that it provides strong robustness benefits against alternative tokenisations. 

\begin{figure}[t!]
    \centering\footnotesize
    \setlength{\figureheight}{0.3\linewidth}    
    \pgfplotsset{
        x tick label style={font=\scriptsize}, 
        y tick label style={rotate=90, font=\scriptsize}, 
        ylabel={\small Avg.\ test accuracy $\rightarrow$},
        ylabel style={yshift=+1pt},
        xlabel={\small Normalised number of splits},
        scale only axis,
        tick align=outside,
        tick pos=left,
        axis x line*=bottom,
        axis y line*=left,
        axis line style={draw=none},
        grid style={solid},
        }    
    \begin{subfigure}[t]{0.48\linewidth}
        \setlength{\figurewidth}{0.7\linewidth}
        \centering
        \begin{tikzpicture}
\begin{axis}[
  height=\figureheight,
  width=\figurewidth,
  ybar,
  bar width=3pt,
  enlarge x limits=0.1,
  symbolic x coords={0,0.1,0.5,1,3},
  xtick=data,
  xtick style={color=black},
  ytick style={color=black},
  ymajorgrids,
  grid style={scGrey},
  tick label style={font=\small},
  ymin=0.1,ymax=.85,
    yticklabel={
      \pgfmathparse{100*\tick}\pgfmathprintnumber[fixed,precision=0]{\pgfmathresult}\%
    },
  axis y line*=left,
  legend cell align={left},
  legend style={font=\scriptsize,legend columns=2,at={(0.4,1.05)},anchor=south,draw=none,/tikz/every even column/.append style={column sep=4pt}},
  legend image post style={xscale=0.6,yscale=1},
]
\addlegendimage{empty legend}
\addlegendentry{\small\textit{Pre-training:}}
\addlegendimage{empty legend}
\addlegendentry{}

\addplot+[very thick,draw=black,fill=black,bar shift=-5pt] coordinates {(0,0.411) ({0.1},0.3822) ({0.5},0.3084) (1,0.2803) (3,0.2622)};
\label{plt:pretraining:canon}
\addlegendentry{$\ppre=0.0$}
\addplot+[very thick,draw=black!60,fill=black!60,bar shift=0pt] coordinates {(0,0.787) ({0.1},0.7642) ({0.5},0.6879) (1,0.6040) (3,0.4364)};
\addlegendentry{$\ppre=0.1$}
\addplot+[very thick,draw=black!30,fill=black!30,bar shift=5pt] coordinates {(0,0.750666666666667) ({0.1},0.769533333333333) ({0.5},0.729666666666667) (1,0.652666666666667) (3,0.493333333333333)};
\addlegendentry{$\ppre=0.5$}
\end{axis}
\end{tikzpicture}
    \end{subfigure}
    \hfill
    \begin{subfigure}[t]{0.49\linewidth}
        \setlength{\figurewidth}{\linewidth}
        \centering
        \pgfplotsset{
            ylabel={},
            ylabel style={yshift=0pt},
            yticklabel=\empty,
            ymajorticks=false
        } 
        \begin{tikzpicture}
\begin{axis}[
  height=\figureheight,
  width=\figurewidth,
  ybar,
  bar width=3pt,
  enlarge x limits=0.12,
  symbolic x coords={0,0.1,0.5,1,3},
  xtick=data,
  xtick style={color=black},
  ymajorgrids,
  grid style={scGrey},
  tick label style={font=\small},
  ymin=0.1,ymax=.85,
  legend cell align={left},
  legend style={font=\scriptsize,legend columns=2,at={(0.5,1.05)},anchor=south,draw=none,/tikz/every even column/.append style={column sep=4pt}},
    legend image post style={xscale=0.6,yscale=1},
]

\addlegendimage{empty legend}
\addlegendentry{\small\textit{Fine-tuning:}}
\addlegendimage{empty legend}
\addlegendentry{\small\;}

\addplot+[very thick,draw=black!60,fill=black!60,bar shift=-7.5pt] coordinates {(0,0.787) ({0.1},0.7642) ({0.5},0.6879) (1,0.6040) (3,0.4364)};
\addlegendentry{$\pfine=0.0$}
\addplot+[very thick,draw=scRed!90,fill=scRed!90,bar shift=-2.5pt] coordinates {(0,0.814) ({0.1},0.8114) ({0.5},0.8084) (1,0.7926) (3,0.7388)};
\addlegendentry{$\pfine=0.1$}
\addplot+[very thick,draw=scRed!60,fill=scRed!60,bar shift=2.5pt] coordinates {(0,0.78) ({0.1},0.7852) ({0.5},0.7878) (1,0.7860) (3,0.7810)};
\addlegendentry{$\pfine=0.5$}
\addplot+[very thick,draw=scRed!30,fill=scRed!30,bar shift=7.5pt] coordinates {(0,0.818) ({0.1},0.8234) ({0.5},0.8282) (1,0.8280) (3,0.8232)};
\addlegendentry{$\pfine=1.0$}
\end{axis}
\end{tikzpicture}
    \end{subfigure}
    \caption{\textit{(Left)} Average test accuracy of \tinyllm \emph{pre-trained} on \langgame with canonical tokenisation~\ref{plt:pretraining:canon} and \stok tokenisation. Stochastic tokenisation performs favourably in all settings and provides mild robustness to input perturbations (normalised number of splits $>0$).
    \textit{(Right)} Additionally fine-tuning with \stok tokenisation ($\pfine>0$) after pre-training with $\ppre=0.1$ substantially improves robustness.
}
\label{fig:robustness_finetuning_tiny_langgame}
\end{figure}
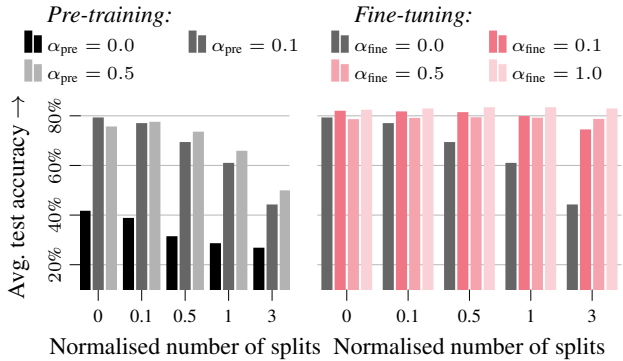

\subsection{Does Fine-tuning Improve Robustness?} \label{subsec:stochastok-fine-tuning}
As pre-training with \stok alone provides mild robustness against noisy inputs, we now address the following two questions: (i) does fine-tuning with \stok after pre-training with \stok instil additional robustness, and if so (ii) is fine-tuning with \stok sufficient to provide robustness against noisy inputs?
Previously, \citet{geh2025adversarial} found that large models, pre-trained on massive, possibly noisy, text corpora, exhibit a certain robustness to non-canonical tokenisations. However, fine-tuning on smaller, curated data sets was argued to introduce vulnerabilities. Hence, question (ii) examines if fine-tuning with \stok prevents those vulnerabilities.

To address the first question, we compare canonical fine-tuning against fine-tuning with \stok using varying expansion proportions $\pfine\in\{0.1,0.5,1\}$. 
In all cases, we start with stochastic pre-training using $\ppre=0.1$.
In \Cref{fig:robustness_finetuning_tiny_langgame} (right), we observe that additional stochastic fine-tuning substantially improves robustness against input perturbations and does not degrade accuracy on canonical tokenisations. 
In \Cref{fig:app-langgame_tiny-p0} (\Cref{app:exp-pretraining-finetuning}), we provided additional ablations on canonical pre-training. %

Next, we test whether this effect translates to larger canonically pre-trained models. 
For this, we use a publicly available model checkpoint of \llamaoneb and fine-tune a low-rank adapter (LoRA) \cite{hu2022lora}. 
\Cref{fig:finetuning_langgame_llama1b} demonstrates that, while the \llamaoneb model is able to solve the subword-level tasks without stochastic pre-training, it is brittle to non-canonical tokenisations during evaluation. Again, a moderate expansion proportion ($\pfine=0.1$) significantly increases robustness, with a slight drop at a normalised number of splits of $\peval=3$. Increased levels of stochasticity ($\pfine\in\{0.1, 0.5\}$) lead to almost constant accuracy at all levels of $\peval$.
This suggests that stochastic tokenisation during fine-tuning is an effective way to improve robustness to non-canonical tokenisation, even in the case of canonical pre-training, and effectively reduces potential vulnerabilities to non-canonical tokenisation.

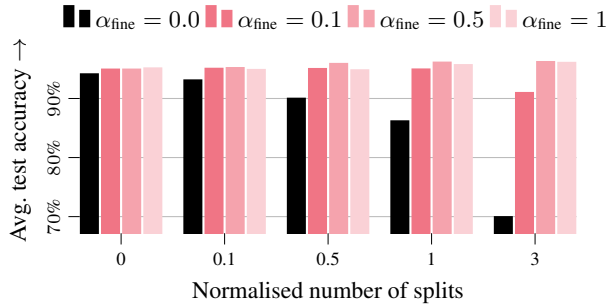
\begin{figure}[t]
    \centering

\centering\footnotesize
\setlength{\figureheight}{0.3\linewidth}
\setlength{\figurewidth}{0.8\linewidth}

\pgfplotsset{
    x tick label style={font=\scriptsize}, 
    y tick label style={rotate=90, font=\scriptsize}, 
    ylabel={\small Avg.\ test accuracy $\rightarrow$},
    xlabel={\small Normalised number of splits},
    scale only axis,
    tick align=outside,
    tick pos=left,
    every axis plot/.append style={mark size=3pt},
    axis x line*=bottom,
    axis y line*=left,
    axis line style={draw=none},
    grid style={solid},
    }    

    \begin{tikzpicture}
\begin{axis}[
  height=\figureheight,
  width=\figurewidth,
  ybar,
  bar width=6pt,
  enlarge x limits=0.1,
  symbolic x coords={0,0.1,0.5,1,3},
  xtick=data,
  xtick style={color=black},
  ytick style={color=black},
  ymajorgrids,
    yticklabel={
      \pgfmathparse{100*\tick}\pgfmathprintnumber[fixed,precision=0]{\pgfmathresult}\%
    },
  grid style={scGrey},
  legend cell align={left},
  legend style={fill opacity=1,legend columns=4,draw opacity=1,text opacity=1,at={(-0.05,1.03)},anchor=south west,draw=none},
]

\addplot+[very thick, draw=black, fill=black, bar shift=-12pt]
coordinates { (0,0.94) ({0.1},0.9297) ({0.5},0.8987) (1,0.8604) (3,0.6982) };
\addlegendentry{$\pfine=0.0$}

\addplot+[very thick, draw=scRed!90, fill=scRed!90, bar shift=-4pt]
coordinates { (0,0.948) ({0.1},0.9494) ({0.5},0.9490) (1,0.9482) (3,0.9084) };
\addlegendentry{$\pfine=0.1$}

\addplot+[very thick, draw=scRed!60, fill=scRed!60, bar shift=4pt]
coordinates { (0,0.948) ({0.1},0.9506) ({0.5},0.9576) (1,0.9598) (3,0.9608) };
\addlegendentry{$\pfine=0.5$}

\addplot+[very thick, draw=scRed!30, fill=scRed!30, bar shift=12pt]
coordinates { (0,0.95) ({0.1},0.9474) ({0.5},0.9468) (1,0.9556) (3,0.9594) };
\addlegendentry{$\pfine=1$}

\end{axis}
\end{tikzpicture}
    \caption{Average test accuracy for canonically pre-trained \llamaoneb model and stochastic tokenisation during fine-tuning on \langgame.
    Stochastic fine-tuning with \stok consistently improves robustness to random perturbations during testing without affecting performance on canonical test data. \label{fig:finetuning_langgame_llama1b}
    }
    \vspace*{-1em}    
\end{figure}

\newcommand{\AccPrec}{1}   %
\newcommand{\DeltaPrec}{1} %

\newcommand{\Acc}[1]{%
  \pgfmathparse{100*(#1)}%
  \pgfmathprintnumber[fixed,fixed zerofill,precision=\AccPrec]{\pgfmathresult}%
}

\newcommand{\HeatVMax}{0.3}   %
\newcommand{\HeatMinInt}{5} %
\newcommand{\HeatSpanInt}{70} %

\newcommand{\HeatCell}[1]{%
  \begingroup
  \pgfmathsetmacro{\d}{#1}%
  \pgfmathsetmacro{\a}{min(abs(\d)/\HeatVMax,1)}%
  \pgfmathtruncatemacro{\I}{round(\HeatMinInt + \HeatSpanInt*\a)}%
  \ifdim \d pt < 0pt\relax
    \edef\HeatColorSpec{scPurple!\I}%
  \else
    \edef\HeatColorSpec{scCyan!\I}%
  \fi
  \expandafter\cellcolor\expandafter{\HeatColorSpec}%
  \pgfmathprintnumber[fixed,fixed zerofill,precision=\DeltaPrec]{\d}%
  \endgroup
}

\newcommand{\HeatCellColorBy}[2]{%
  \begingroup
  \pgfmathsetmacro{\d}{#1}%
  \pgfmathsetmacro{\p}{#2}%
  \pgfmathsetmacro{\a}{min(abs(\d)/\HeatVMax,1)}%
  \pgfmathtruncatemacro{\I}{round(\HeatMinInt + \HeatSpanInt*\a)}%
  \ifdim \d pt < 0pt\relax
    \edef\HeatColorSpec{scPurple!\I}%
  \else
    \edef\HeatColorSpec{scCyan!\I}%
  \fi
  \expandafter\cellcolor\expandafter{\HeatColorSpec}%
  \pgfmathparse{100*(\p)}%
  \pgfmathprintnumber[fixed,fixed zerofill,precision=\DeltaPrec]{\pgfmathresult}%
  \endgroup
}

\newcommand{\PairDelta}[2]{%
  \Acc{#1} &
  \begingroup
  \pgfmathsetmacro{\delta}{(#2)-(#1)}%
  \HeatCellColorBy{\delta}{#2}%
  \endgroup
}

\begin{table*}[t]
\centering\small
\renewcommand{\arraystretch}{1.1}
\setlength{\tabcolsep}{4pt}

\caption{Accuracy (\%) under canonical versus \stok ($\peval=3$) query tokenisation across data sets, $K=10$ few shot examples. }
\label{tab:avg_icl}
\resizebox{\textwidth}{!}{%

\begin{tabular}{@{}l ll *{7}{cc} @{}}
\toprule
& \multirow{2}{*}{\shortstack{\textbf{Context}\\\textbf{tokenisation}}} &
  \multirow{2}{*}{\textbf{$\picl$}} &
  \multicolumn{2}{c}{\textsc{Language Game}} &
  \multicolumn{2}{c}{\textsc{ARC-E}} &
  \multicolumn{2}{c}{\textsc{COPA}} &
  \multicolumn{2}{c}{\textsc{CSQA}} &
  \multicolumn{2}{c}{\textsc{HellaSwag}} &
  \multicolumn{2}{c}{\textsc{PIQA}} &
  \multicolumn{2}{c}{\textsc{SocialIQA}} \\
\cmidrule(l){4-5}\cmidrule(l){6-7}\cmidrule(l){8-9}\cmidrule(l){10-11}\cmidrule(l){12-13}\cmidrule(l){14-15}\cmidrule(l){16-17}
& & &
\textbf{canon.} & \textbf{\textsc{STok}} &
\textbf{canon.} & \textbf{\textsc{STok}} &
\textbf{canon.} & \textbf{\textsc{STok}} &
\textbf{canon.} & \textbf{\textsc{STok}} &
\textbf{canon.} & \textbf{\textsc{STok}} &
\textbf{canon.} & \textbf{\textsc{STok}} &
\textbf{canon.} & \textbf{\textsc{STok}} \\
\midrule

& Canonical & -  & \PairDelta{0.838}{0.723} & \PairDelta{0.782}{0.738} & \PairDelta{0.920}{0.882} & \PairDelta{0.684}{0.652} & \PairDelta{0.548}{0.512} & \PairDelta{0.814}{0.802} & \PairDelta{0.524}{0.472} \\
& \stok & 0.1  & \PairDelta{0.846}{0.764} & \PairDelta{0.770}{0.747} & \PairDelta{0.920}{0.883} & \PairDelta{0.684}{0.647} & \PairDelta{0.552}{0.518} & \PairDelta{0.818}{0.801} & \PairDelta{0.526}{0.502} \\
& \stok & 0.5  & \PairDelta{0.870}{0.802} & \PairDelta{0.776}{0.760} & \PairDelta{0.920}{0.882} & \PairDelta{0.674}{0.657} & \PairDelta{0.542}{0.513} & \PairDelta{0.820}{0.805} & \PairDelta{0.540}{0.507} \\
& \stok & 1.0  & \PairDelta{0.864}{0.807} & \PairDelta{0.776}{0.760} & \PairDelta{0.910}{0.877} & \PairDelta{0.666}{0.654} & \PairDelta{0.542}{0.515} & \PairDelta{0.828}{0.808} & \PairDelta{0.530}{0.511} \\
\bottomrule
\end{tabular}
} 

\end{table*}

\subsection{Does In-context Learning Improve Robustness?} \label{subsec:stock-ict}
Lastly, we study the most restrictive scenario, where the model parameters are kept frozen.  We test whether robustness to stochastic tokenisations can be improved \emph{without parameter updates}, using in-context learning (ICL) \cite{yin2024deeper}. In ICL, the prediction for a test input $\bx^*$ is obtained by prepending solved examples $(\bx_1, \br_1, \dots, \bx_K, \br_K, \bx^*)$, while keeping model parameters fixed. Specifically, we analyse whether stochastic tokenisation of the context questions, $\bv_i \sim P(\cdot\mid \bx_i;\alpha_{\mathrm{ICL}})$, increases robustness to stochastic tokenisations of the query, $\bv^* \sim P(\cdot\mid \bx^*;\alpha_{\mathrm{eval}})$.

We evaluate robustness to non-canonical tokenisations of the query for the \langgame data set and general multiple-choice data sets. \Cref{tab:avg_icl} reports results for a \llamaeightb model using $K=10$ context examples. We compare accuracy for a canonical query and report the drop in accuracy when evaluated on non-canonical tokenisations at a normalised number of splits of $\peval=3$. 
Using a canonical context \textit{(first row)} consistently yields lower accuracy under non-canonical queries compared to the canonical tokenisation. Using \stok for context tokenisation yields notable robustness gains on \langgame, with modest improvements on the other benchmark datasets. %
While ICL imposes the lowest computational cost, it also yields lower robustness gains than full fine-tuning. In fine-tuning, the model is trained with 10k training samples, each tokenised with 10 random draws by \stok. In contrast, in ICL the model is presented with $K=10$ examples, each with a single \stok sample. %

Having established that \stok does improve robustness to non-canonical tokenisations, we investigate whether further gains are possible. We explore the limitations of \stok's distribution over tokenisations and propose two alternative sampling schemes (\stokuni and \unik), considering both randomly sampled and adversarially selected non-canonical tokenisations.

\section{Improving with Uniform Sampling} \label{sec:methods}
While training with \stok generally improves robustness, it comes with two key limitations.
Firstly, the sampling distribution $P_\textsc{STok}(\bv\mid \bx;\alpha)$ generated by \stok is non-uniform.  
Among segmentations with the same number of splits, it is biased towards those constructed from subtokens with fewer alternative split pairs. 
Secondly, the sampling process yields a distribution with incomplete support. We discuss these two limitations in detail below.

\textbf{Bias.}\;
First, let us call a subtoken pair $(v, v') \in \mcV^2$ \emph{valid} if, when merged via a two-way merge $v \circ v' = u$, they form a token $u \in \mcV$.
Now, let us denote the set of valid subtoken pairs as: $\mcS(v_i)=\{ (t,t') \in\mcV^2 : t\circ t' = v_i  \}$. 

Recall, \stok recursively splits tokens $v_i$ into (valid) token pairs $(t,t')$. Let us consider tokenisations obtained by performing exactly one further split. At this point, \stok selects the token to be split uniformly from $\{ t,t'\}$. Importantly, the two subtokens may have an unequal number of valid subtoken pairs, $|\mcS(t)|\neq|\mcS(t')|$. As the subsequent split pair is sampled uniformly from $\mcS(\cdot)$, each individual tokenisation is obtained via splitting $t$ or $t'$ with probability $\frac{1}{|\mcS(t)|}$ or $\frac{1}{|\mcS(t')|}$, respectively. Thus, the induced distribution over the tokenisations is inherently biased.

\textbf{Incomplete Support.}\;
Let's assume we have the canonical token \tikz[baseline=-.3em]\protect\node[fill=scGrey!40] {\texttt{revolution}};
Then, a potential non-canonical tokenisation with edit distance of $3$ by \stok is:
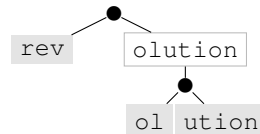
\begin{center}
    \begin{tikzpicture}[scale=0.95, transform shape]
    \node[fill=scGrey!40] (rev) at (3,-1.5) {\texttt{rev}};
        \node[draw=scGrey] (olution) at (5,-1.5) {\texttt{olution}};
        \node[fill=scGrey!40] (ol) at (4.5,-2.5) {\texttt{ol}};
        \node[fill=scGrey!40] (ution) at (5.5,-2.5) {\texttt{ution}};

        \node[inner sep=0pt, outer sep=0pt] (m1) at (4,-1) {\tiny\faCircle};
        \node[inner sep=0pt, outer sep=0pt] (m2) at (5,-2) {\tiny\faCircle};

        \draw (rev) -- (m1);
        \draw (olution) -- (m1);

        \draw (ol) -- (m2);
        \draw (ution) -- (m2);
        \draw (m2) -- (olution);
    \end{tikzpicture}    
\end{center}
In this example, we iteratively apply 2 two-way merges to obtain the initial canonical token from the respective subtokens.
Another potential tokenisation is by splitting the initial token directly into three subtokens, \ie, 
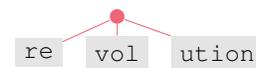
\begin{center}
    \begin{tikzpicture}[scale=0.95, transform shape]
    \node[fill=scGrey!40] (re-m) at (3,-4) {\texttt{re}};
    \node[fill=scGrey!40] (vol-m) at (4.1,-4) {\texttt{vol}};
    \node[fill=scGrey!40] (ution-m) at (5.5,-4) {\texttt{ution}};

    \node[inner sep=0pt, outer sep=0pt, scRed] (m21) at (4.1,-3.5) {\tiny\faCircle};

    \draw[scRed] (re-m) -- (m21);
    \draw[scRed] (vol-m) -- (m21);
    \draw[scRed] (ution-m) -- (m21);
\end{tikzpicture}    
\end{center}
where we note that all three subtokens, \tikz[baseline=-.3em]\protect\node[fill=scGrey!40] {\texttt{re}};, \tikz[baseline=-.3em]\protect\node[fill=scGrey!40] {\texttt{vol}};, and \tikz[baseline=-.3em]\protect\node[fill=scGrey!40] {\texttt{ution}};; are contained in the vocabulary.
However, neither \tikz[baseline=-.3em]\protect\node[draw=scGrey!40] {\texttt{revol}}; nor \tikz[baseline=-.3em]\protect\node[draw=scGrey!40] {\texttt{volution}}; are tokens in the vocabulary $\mcV$.
Thus, \stok cannot generate the tokenisation through a series of two-merges.
This example illustrates the second key issue: \stok does not have full support for all possible k-edit distance tokenisations of a token.

\begin{figure}[t]
\centering\scriptsize
    \pgfplotsset{
    axis x line*=bottom,
    axis y line*=left,
    grid style={solid},
    }    

\begin{subfigure}[t]{0.98\linewidth}
\centering

\setlength{\figurewidth}{0.85\linewidth}
\setlength{\figureheight}{0.2\linewidth}

\begin{tikzpicture}
    \begin{axis}[
      ybar,
      legend cell align={left},
      bar width=4pt,
      width=\figurewidth,
      height=\figureheight,
      ymin=0,
      ylabel={Probability (\%)},
      scale only axis,
      yticklabel={
      \pgfmathparse{100*\tick}%
      \pgfmathprintnumber[fixed,precision=0]{\pgfmathresult}%
      },
      xlabel={},
      symbolic x coords={
        rev-ol-u-tion,
        re-vol-u-tion,
        r-e-v-olution,
        rev-o-l-ution,
        re-vo-l-ution,
        rev-o-lu-tion,
        re-vol-uti-on,
        rev-ol-ut-ion,
        re-v-ol-ution,
        re-vo-lu-tion,
        rev-ol-uti-on,
        r-e-vol-ution,
        re-vol-ut-ion,
        r-ev-ol-ution,
      },
      xtick=data,
      xticklabel style={font=\ttfamily\tiny, rotate=35, anchor=east},
      enlarge x limits=0.05,
    ]
    \addplot+[bar shift=-2pt, fill=scRed, draw=scRed] table[
      col sep=comma,
      x=Segmentation,
      y=Frequency
    ] {
    Segmentation,Frequency
    rev-ol-u-tion,0.07471943295924395
    re-vol-u-tion,0.07412876550502068
    r-e-v-olution,0.0759007678676905
    rev-o-l-ution,0.0637920850561134
    re-vo-l-ution,0.0670407560543414
    rev-o-lu-tion,0.06556408741878322
    re-vol-uti-on,0.06585942114589487
    rev-ol-ut-ion,0.07531010041346722
    re-v-ol-ution,0.07147076196101594
    re-vo-lu-tion,0.07560543414057885
    rev-ol-uti-on,0.06556408741878322
    r-e-vol-ution,0.07058476077968104
    re-vol-ut-ion,0.0759007678676905
    r-ev-ol-ution,0.07855877141169522
    };
    \label{plt:seg_histogram:uniform}
    \addplot+[bar shift=+2pt,fill=black!60, draw=black!60] table[
      col sep=comma,
      x=Segmentation,
      y=Frequency] {
    Segmentation,Frequency
    r-ev-ol-ution,0.17307692307692307
    rev-o-l-ution,0.21875
    re-v-ol-ution,0.18509615384615385
    rev-ol-ut-ion,0.052884615384615384
    rev-ol-uti-on,0.04807692307692308
    rev-ol-u-tion,0.05048076923076923
    r-e-v-olution,0.27163461538461536
    };
    \label{plt:seg_histogram:stok}
    \end{axis}
\end{tikzpicture}
\end{subfigure}

\caption{Histogram of segmentation probabilities for two sampling schemes, \ref{plt:seg_histogram:stok} \stok and  \ref{plt:seg_histogram:uniform} \stokuni. \stok induces a biased distribution over a subset of all valid segmentations with limited support. \stokuni (ours), on the other hand, generates samples with equal probability and has full support over the set of non-canonical tokenisations.
}

\label{fig:seg_histograms}

\end{figure}
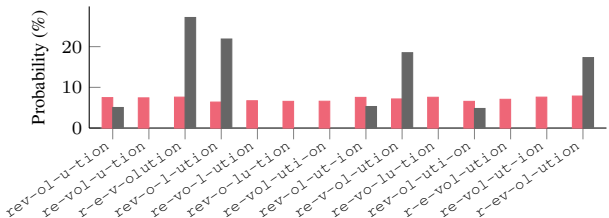

\subsection{Removing Bias and Increasing Support} 
To address these limitations and understand their influence on the robustness to stochastic tokenisations, we propose stochastic tokenisation schemes with progressively increasing uniformity and support.

\textbf{\stokuni.}\;
As a first step, we introduce a uniform extension of \stok that reduces sampling bias while preserving the per-token split-count distribution induced by \stok, which we call \stokuni. For this, we devise a two-stage procedure where (i) we generate the per-token split-count vector $\bS=(S_1, \dots, S_m)$ for a sequence of $m$ tokens following the distribution induced by \stok, and (ii) for each token independently, we uniformly sample non-canonical tokenisations with $S_i$ splits, that is, with $S_i+1$ segments.

First, note that \stok introduces a split-count distribution for a canonical token following a Beta-Binomial distribution, \cf, \Cref{app:sec:splitcounts}.
By doing so, \stok has a preference towards small split counts.
To sample $\bS=(S_1, \dots, S_m)$ with $\sum_i S_i = N$ we use a Dirichlet-multinomial model, $\bS \sim \distDirMult(N, \balpha)$, with symmetric concentration parameter $\balpha = \b1$.

For the second step, let $\bv^c=(v_1^c, \dots, v_m^c)$ denote the canonical token sequence. Recall that each token $v_i^c\in\mcV$ corresponds to a character sequence which we denote by $\tilde{\bx}_i \coloneqq \kappa(v_i^c)\in\mcX^*$.
Then, conditional on the split counts $\bS$, we sample tokenisations uniformly from $\mcT^{S_i+1}_\mcV(\tilde{\bx}_i, v_i^c)$, that is, from the set of tokenisations consisting of $S_i+1$ segments. 

Inspired by \citet{zheng2025broken}, we recursively construct a tree, where each root-to-leaf path corresponds to a non-canonical tokenisation $\bv\in\mcT_\mcV(\tilde{\bx}_i)$.
To ensure that the path length equals the number of segments $S_i+1$, we remove all paths longer or shorter than $S_i+1$. 
Next, we perform uniform sampling by descending the tree, selecting each outgoing edge with probability proportional to the number of leaves in its subtree. The sampling cost scales linearly with the number of segments, while constructing the tree scales exponentially with the sequence length of the token, $\mcO\left(2^{|\Tilde{\bx}|-1}\right)$. In our case, the cost remains manageable as we only construct the tree for each token $v_i^c$ independently.
\Cref{fig:seg_histograms} compares the empirical distribution over segmentation probabilities for \stok and \stokuni. 
Note that \stok covers only a subset of all non-canonical tokenisations, while \stokuni (i) provides full support and (ii) results in a uniform distribution over all tokenisations. 

\textbf{\unik.}\;
Note that, while \stokuni removes the key limitations of \stok, it remains restricted to per-token level splits. Also, the tokenisations are sampled uniformly \textit{conditioned} on the per-token split count. 
Henceforth, we now introduce \unik, which samples \emph{marginally} uniformly at edit distance $k$ from the set of non-canonical tokenisations over the \emph{entire input sequence}, \ie
$\bv \sim \distUnif\!\left(\mcT_\mcV^k(\bx,\bv^c)\right)$.
Creating the tokenisation tree recursively for the entire string $\bx$, is generally infeasible due to exponential scaling \wrt the length $|\bx|$. To address this, we encode the tokenisations with a Multi-valued Decision Diagram (MDD) \cite{geh2024signal}. By decomposing and reusing subsequences, the MDD encodes all tokenisations with a construction cost scaling linearly \wrt the tokeniser vocabulary and the sequence length, $\mcO(|\bx| \cdot|\mcV|)$. Sampling is performed by following edges proportional to the number of tokenisations of the sub-MDD. To sample uniformly at edit distance $k$, \citet{geh2025adversarial} proposed to construct a multi-rooted MDD (MRMDD), where the construction cost is upper bounded by $\mcO(|\bx|^2 \cdot|\mcV|)$. We adopt their MRMDD algorithm and select the edit distance $k$ to match the total split count $N$ of \stok and \stokuni. Note, however, that uniform sampling induces a different per-token split-count distribution (see \Cref{app:distribution}).

By construction, \unik can create subtokens across canonical token boundaries. Thus, providing even higher flexibility than \stokuni.
Lastly, we include full uniform sampling as a baseline, denoted \uni. This offers full support over $\mcT_\mcV(\bx)$, but substantially increases computational cost and the number of splits generated. In \Cref{app:sec:splitcounts} we present empirical results for the number of splits generated by each variant.

\subsection{Fine-tuning Robustness}\label{subsec:uniform-fine-tuning}
We now analyse if increased uniformity over non-canonical tokenisations improves robustness during fine-tuning.

\Cref{tab:finetune_schemes_langgame_llama1b} reports the drop in accuracy when switching from canonical queries to uniformly sampled tokenisations. 
As in \Cref{subsec:stochastok-fine-tuning}, canonical fine-tuning results in the largest drop in accuracy ($-29.8\%$). 
Further note that increasing the number of splits, \ie, increasing $\pfine$, improves robustness for all methods. 
Next, recall that both \stokuni and \unik produce uniform samples conditioned on the number of splits, but with different distributions over split-counts and different support.
Despite those differences, we observe that both approaches generally result in a smaller drop in accuracy than \stok, indicating that increased uniformity is beneficial.

In addition to the results shown in \Cref{tab:finetune_schemes_langgame_llama1b}, we note that \Cref{fig:teaser}, which assesses the performance for increasing levels of perturbations (normalised edit distance), further highlights the benefits of a uniform sampling distribution.

\begin{table}[t]
\centering\small
\setlength{\tabcolsep}{3pt}

\caption{Accuracy (\%) of \llamaoneb for (i) canonical query tokenisation ($\bv^*=\tau(\bx)$) or (ii) a randomly sampled non-canonical query tokenisation ($\mathrm{Uniform}(\mathcal{T}_{\mathcal{V}}(\bx))$) trained on \langgame under various fine-tuning strategies.}
\label{tab:finetune_schemes_langgame_llama1b}

\renewcommand{\HeatVMax}{0.30} %
\renewcommand{\HeatMinInt}{5} %
\renewcommand{\HeatSpanInt}{70} %

\newcommand{\HeatCellFinetuning}[1]{%
  \begingroup
  \pgfmathsetmacro{\d}{#1}%

  \pgfmathsetmacro{\a}{(abs(\d) / \HeatVMax)}%
  \pgfmathsetmacro{\a}{min(\a, 1)}%

  \pgfmathtruncatemacro{\I}{round(\HeatMinInt + \HeatSpanInt*\a)}%

  \ifdim \d pt < 0pt\relax
    \edef\HeatColorSpec{scPurple!\I}%
  \else
    \edef\HeatColorSpec{scCyan!\I}%
  \fi

  \expandafter\cellcolor\expandafter{\HeatColorSpec}%
  \pgfmathparse{100*(\d)}%
  \pgfmathprintnumber[fixed,showpos,precision=3]{\pgfmathresult}%
  \endgroup
}

\newcommand{\tableitem}[4]{%
  #1 & #2 & \Acc{#3} & \HeatCellFinetuning{#4}%
}

\begin{tabular}{@{}llcc@{}}
\toprule
\multirow{2}{*}{\textbf{Fine-tuning}} &
\multirow{2}{*}{\textbf{$\pfine$}} &
\multicolumn{2}{c}{\textbf{Accuracy}} \\
\cmidrule(l){3-4}
& & \textbf{$\bv^*$ canonical} & \textbf{$\bv^*$ uniform $\Delta$} \\
\midrule
\tableitem{Canonical}{0}{0.940}{-0.298} \\ \midrule

\tableitem{\stok{}}{0.1}{0.948}{-0.117} \\
\tableitem{}{0.5}{0.948}{-0.005} \\
\tableitem{}{1}{0.950}{-0.003} \\ \midrule

\tableitem{\stokuni{}}{0.1}{0.948}{-0.088} \\
\tableitem{}{0.5}{0.936}{-0.001} \\
\tableitem{}{1}{0.950}{0.014} \\ \midrule

\tableitem{\unik{}}{0.1}{0.938}{-0.048} \\
\tableitem{}{0.5}{0.962}{-0.006} \\
\tableitem{}{1}{0.948}{0.028} \\ \midrule

\tableitem{\uni{}}{--}{0.928}{0.060} \\

\bottomrule
\end{tabular}
\end{table}

\subsection{In-context Learning Robustness}\label{subsec:stochastic_scheme}
Having established that inducing a uniform sampling distribution improves robustness in the case of stochastic tokenisation, we now investigate the ICL setting.
For this, we expose the model to $K=10$ in-context examples for few-shot learning, with each example having a non-canonical tokenisation generated by \stok, \stokuni, \unik, or \uni, respectively.
Note that this setup is consistent with \cref{subsec:stock-ict} and that, in contrast to pre-training of fine-tuning experiments, only one non-canonical tokenisation per training example is provided.
Results in \cref{tab:uniformity_avg_icl} show that, while stochastic tokenisation during ICL generally provides mild improvements over canonical tokenisation, little difference between the stochastic tokenisation schemes can be observed. However, the most noticeable improvements are observed with \uni, which also provides the highest degree of stochasticity.
It may be necessary to provide a substantially larger number of in-context examples to induce robustness via ICL.

\begin{table*}[t]
\centering\small
\renewcommand{\arraystretch}{1.1}
\setlength{\tabcolsep}{4pt}

\renewcommand{\HeatVMax}{0.3} %
\renewcommand{\HeatMinInt}{0} %
\renewcommand{\HeatSpanInt}{70} %

\caption{Accuracy (\%) under canonical versus uniformly sampled query tokenisation across data sets, with $K=10$ few shot examples. 
}
\label{tab:uniformity_avg_icl}
\resizebox{\textwidth}{!}{%

\begin{tabular}{@{}l ll *{7}{cc} @{}}
\toprule
& \multirow{2}{*}{\shortstack{\textbf{Context}\\\textbf{tokenisation}}} &
  \multirow{2}{*}{\textbf{$\picl$}} &
  \multicolumn{2}{c}{\textsc{Language Game}} &
  \multicolumn{2}{c}{\textsc{ARC-E}} &
  \multicolumn{2}{c}{\textsc{COPA}} &
  \multicolumn{2}{c}{\textsc{CSQA}} &
  \multicolumn{2}{c}{\textsc{HellaSwag}} &
  \multicolumn{2}{c}{\textsc{PIQA}} &
  \multicolumn{2}{c}{\textsc{SocialIQA}} \\
\cmidrule(l){4-5}\cmidrule(l){6-7}\cmidrule(l){8-9}\cmidrule(l){10-11}\cmidrule(l){12-13}\cmidrule(l){14-15}\cmidrule(l){16-17}
& & &
\textbf{canon.} & \textbf{unif.} &
\textbf{canon.} & \textbf{unif.} &
\textbf{canon.} & \textbf{unif.} &
\textbf{canon.} & \textbf{unif.} &
\textbf{canon.} & \textbf{unif.} &
\textbf{canon.} & \textbf{unif.} &
\textbf{canon.} & \textbf{unif.} \\
\midrule

& Canonical & -  & \PairDelta{0.839}{0.703} & \PairDelta{0.781}{0.713} & \PairDelta{0.920}{0.846} & \PairDelta{0.683}{0.630} & \PairDelta{0.548}{0.502} & \PairDelta{0.814}{0.795} & \PairDelta{0.524}{0.451} \\ %
& \stok & 1.0  & \PairDelta{0.877}{0.818} & \PairDelta{0.774}{0.740} & \PairDelta{0.913}{0.870} & \PairDelta{0.671}{0.639} & \PairDelta{0.544}{0.509} & \PairDelta{0.819}{0.800} & \PairDelta{0.538}{0.475} \\ %
& \stokuni & 1.0  & \PairDelta{0.871}{0.813} & \PairDelta{0.773}{0.738} & \PairDelta{0.916}{0.873} & \PairDelta{0.667}{0.639} & \PairDelta{0.542}{0.510} & \PairDelta{0.819}{0.801} & \PairDelta{0.534}{0.477} \\ %
& \unik & 1.0  & \PairDelta{0.872}{0.808} & \PairDelta{0.771}{0.737} & \PairDelta{0.911}{0.872} & \PairDelta{0.676}{0.640} & \PairDelta{0.543}{0.509} & \PairDelta{0.819}{0.800} & \PairDelta{0.528}{0.489} \\ \midrule
& \uni & -  & \PairDelta{0.850}{0.817} & \PairDelta{0.776}{0.752} & \PairDelta{0.902}{0.870} & \PairDelta{0.623}{0.638} & \PairDelta{0.541}{0.510} & \PairDelta{0.819}{0.804} & \PairDelta{0.514}{0.498} \\
\bottomrule
\end{tabular}
} 
\end{table*}

\section{Robustness to Adversarial Tokenisation}\label{sec:adversarial}

We will now provide an empirical assessment of adversarial robustness induced by stochastic tokenisation, utilising the adversarial attack mechanism presented in \cite{geh2025adversarial}, and following the approach outlined in \cref{sec:preliminaries}.

\Cref{fig:adv_robustness_llama1b} shows the results for the \llamaoneb model fine-tuned canonically (\textsc{Canon.}) or using a stochastic tokenisation scheme and evaluated either on canonical tokens or under adversarial attacks.
We see that non-canonical tokenisation during fine-tuning substantially improves adversarial robustness, \ie, accuracy under canonical fine-tuning drops from $94\%$ to $6.1\%$ (almost $90$ percentage points), while under \unik we only observe a drop to $69.6\%$.
Further, we observe that increased uniformity at the same split budget provides increased robustness to adversarial tokenisations.
In \Cref{app:add_results_tinyllm} we provide additional results for \tinyllm and for the \cute dataset. In \Cref{app:bpedropout} we provide a comparison with \bpedropout for \llamaoneb and \qwenzerosixb models. In both cases, the findings are consistent with those presented here.

Lastly, we evaluate whether stochastic tokenisation during ICL provides any robustness to adversarial attacks. 
\Cref{tab:adv_icl_langgame} summarises the experiments and indicates that, surprisingly, ICL with stochastic tokenisation can provide mild robustness to adversarial tokenisation.

\subsection{Theoretical Analysis}\label{sec:theo_analysis}

Beyond the empirical results outlined in the previous section, we aim to theoretically characterise how tokenisation affects the robustness of the resulting model. In this perspective, we focus on evasion attacks, \ie, adversarial perturbations applied at test time. 
In this direction, let's consider a trained classifier $f\colon \mcV \to \mcY$ and let $v \in \mcV $ be an input with its associated label $y \in \mcY$, such that $f(v) = y$. The goal of an attacker is to craft a small additional perturbation to the input, such as to generate a point $\tilde{v}$ whose prediction $f(\tilde{v})$ is different from the original one. We note that in our setting, the function $f$ represents the LLM \emph{after} the tokeniser, comprising the embedding function and subsequent transformer blocks with self-attention.

As the generated adversarial perturbation should be ``similar'' to the original input, we consider a similarity budget $\epsilon$ (the maximum possible changes to the input), comprising the set of valid adversarial perturbations defined by a neighbourhood $\mcB(x, \epsilon)$. 
In our setting $\mcB(x, \epsilon) = \mcT^k_{\mcV}(x)$ with edit distance at most $k$. 
Now, the adversarial aim is to find, within $\mcB(x, \epsilon)$, points that 
result in the worst prediction (the model's \textit{adversarial risk}),
which 
can be formulated as:
\begin{equation}\label{equation:robustness_definition}
\mathcal{R}_{\epsilon}[f] = \mathop{\mathbb{E}}_{x \in \mathcal{D}_{\mathcal{X}}} \left[\sup_{\tilde{x} \in \mathcal{B}(x, \epsilon)} d_{\mathcal{Y}}\left(f\left(\tilde{x}\right), f\left(x\right)\right)\right] .
\end{equation}
From a defence perspective, the objective is to ensure that the risk remains small, thereby keeping predictions within the same output region and making it harder to find a valid perturbation. This goal can be formulated as:
\begin{definition}[Adversarial Robustness]\label{def:adv_robustness}
    The classifier $f$ is said to be $(\epsilon, \gamma)$-\emph{robust} if its adversarial risk with respect to the classifier $f$ satisfies: $\mathcal{R}_{\epsilon}[f] \leq \gamma.$
\end{definition}

\Cref{def:adv_robustness} uses an upper-bound perspective ($\gamma$) as computing the exact risk is generally intractable. 
A smaller value of $\gamma$ indicates stronger robustness guarantees. Consequently, comparing the resulting bounds provides a principled way to assess and contrast the model's robustness. 
Without loss of generality, we consider a $1$-layer Transformer-based model in which all activation functions are assumed to be $1$-Lipschitz, which is the case for many common activations, \eg,  \texttt{tanh} and \texttt{sigmoid} \cite{virmaux18}. We additionally assume the output of the embedding function to be bounded in terms of its norm (\eg, $[0, 1]^{n \times d}$, with $n$ being the number of tokens and $d$ the hidden dimension). In this setting, the model's underlying robustness, $\gamma$, can be derived and is given by the following. 

\begin{theorem}\label{theo:upper_bound_robustness}
    Let $f\colon \mathcal{X}\rightarrow\mathcal{Y}$ be a classifier following our considered problem setup, then $f$ is $(\epsilon, \gamma)$-\emph{robust}, with: 
    \begin{equation}
        \gamma = \sqrt{2} \lVert W\lVert
    \left(\frac{d}{d-1}\right)^2
    C_1 C_2 \sqrt{k} ,
    \end{equation}
    where $W$ is the weight matrix of the embedding layer and $C_1C_2$ is proportional to the Lipschitz constant of the attention block \cite{ennadir2025pool}. 
\end{theorem}

Details and proof of \cref{theo:upper_bound_robustness} are given in \cref{app:proof_theorem}. 
The resulting upper bound, as with any Lipschitz-like analysis, depends on the weight norms and, additionally, on the factor $k$, which is linked to the tokenisation.
Thus, it underscores the importance of tokenisation for the model’s robustness. 
Additionally, the bound 
explicitly on the perturbation neighbourhood, defined through the token-level edit distance, linking it to the support of the sampling distribution. 
Specifically, canonical tokenisation concentrates the model’s sensitivity on a single tokenisation, making it vulnerable to adversarial attacks.
Stochastic tokenisation, however, 
provides the model with a distribution over tokenisation, effectively smoothing the embedding space and reducing its vulnerability. 

In connection to this, we empirically evaluated the sensitivity of each layer by comparing unit-normalised embeddings for $50$ neighbouring non-canonical tokenisations of
the top $10$k most frequent canonical tokens.
\Cref{fig:app:embeddings} shows the average distance for canonical and stochastic tokenisation schemes.
We find that stochastic tokenisation results in smaller mean normalised distances, indicating smaller local Lipschitz constants and suggesting reduced vulnerability.

\begin{figure}[t!]
\centering

        \resizebox{1\linewidth}{!}{%
            \pgfplotsset{
  colormap={scScheme}{
    color(0pt)=(white)
color(1pt)=(scCyan!60)
  }
}

\begin{tikzpicture}
\begin{groupplot}[
  group style={
    group size=2 by 1,
    horizontal sep=.15cm,
  },
  width=0.3\textwidth,
  tick label style={font=\small},   colormap name=scScheme,
  point meta min=0,
  point meta max=100,
  x grid style={gray!30},
  y grid style={gray!30},
  xlabel={$\pfine$},
  xmin=-0.5, xmax=3.5,
  ymin=-0.5, ymax=4.5,
  xtick={0,...,3},
  xticklabels={-, 0.1, 0.5, 1.0},
  ytick={0,...,4},
  yticklabels={
    \textsc{Canon.},
    \textsc{Uni.}, 
    \textsc{STok}, 
    \textsc{STok-uni},
    \textsc{Uni-k}
  },
  yticklabel style={rotate=0, anchor=east, font=\small},
  y dir=reverse,
  tick align=outside,
  tick pos=left,
  every axis label/.append style={font=\small},
  every node near coord/.append style={
    font=\footnotesize, color=black, anchor=center,
    /pgf/number format/.cd, fixed, fixed zerofill, precision=1,
  },
]

\nextgroupplot[
  title={Canonical Tokenisation},
  title style={yshift=-1.0ex},
]
\addplot [
  matrix plot,
  point meta=explicit,
  mesh/cols=4,
  point meta=\thisrow{Value} * 100,
  nodes near coords={\printmynumber\pgfplotspointmeta},
]
table [meta=Value] {
x y Value
0 0 0.940
1 0 nan
2 0 nan
3 0 nan
0 1 0.928
1 1 nan
2 1 nan
3 1 nan
0 2 nan
1 2 0.950
2 2 0.948
3 2 0.950
0 3 nan
1 3 0.948
2 3 0.938
3 3 0.950
0 4 nan
1 4 0.938
2 4 0.964
3 4 0.948
};

\nextgroupplot[
  title={Adversarial Tokenisation},
  title style={align=center,yshift=-1.0ex},
  colorbar style={
    ylabel={},
    y dir=normal,
    width=5pt, 
  },
yticklabels={
},
  y dir=reverse,
]
\addplot [
  matrix plot,
  point meta=explicit,
  mesh/cols=4,
  point meta=\thisrow{Value} * 100,
nodes near coords={\printmynumber\pgfplotspointmeta},
]
table [
meta=Value
] {
x y Value
0 0 0.061
1 0 nan
2 0 nan
3 0 nan
0 1 0.618
1 1 nan
2 1 nan
3 1 nan
0 2 nan
1 2 0.266
2 2 0.518
3 2 0.620
0 3 nan
1 3 0.328
2 3 0.592
3 3 0.616
0 4 nan
1 4 0.350
2 4 0.644
3 4 0.696
};

\end{groupplot}
\end{tikzpicture}
        }

    \caption{Accuracy (\%) under canonical and adversarial tokenisation for \llamaoneb trained on \langgame under various fine-tuning strategies.}
    \vspace*{-1em}\label{fig:adv_robustness_llama1b}
\end{figure}
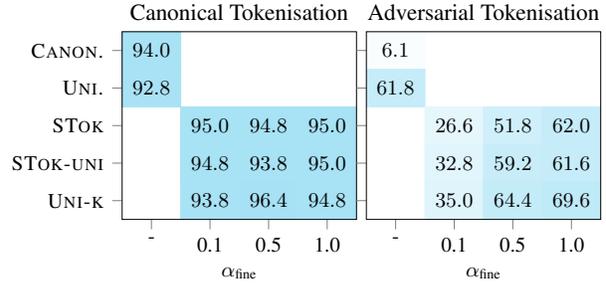

\section{Conclusion \& Discussion}

In this paper, we address the fundamental question: \textbf{Can stochastic tokenisation improve robustness?} and find that stochastic tokenisation during training can provide robustness against random perturbations and adversarial attacks on tokens.
To this end, we first examined the robustness of \stok \cite{sims2026stochastok} and found that it provides some robustness, but the sampling distribution is both biased and incomplete, making it vulnerable to adversarial attacks.
Hence, we develop alternative sampling mechanisms, specifically \stokuni and \unik, both of which ensure a uniform sampling distribution (conditional in the case of \stokuni) with support spanning all non-canonical tokenisations within some neighbourhood.
We find uniform sampling often performs favourably compared to \stok and provides stronger adversarial robustness.
Lastly, we provided a theoretical analysis of the expected adversarial risk, further confirming that stochastic tokenisation reduces it.

\textbf{Practitioner Guide.}\;
We found that canonical tokenisation results in a brittle model that is vulnerable to non-canonical tokenisations or attacks.
Further, we found that fine-tuning with stochastic tokenisation substantially increases robustness, without additional cost during inference.
This indicates that fine-tuning should generally be performed with stochastic tokenisation instead of canonical tokenisation.

\textbf{Limitations.}\;
Our analysis is currently limited to question-answering tasks, and an interesting next step would be to extend the assessment to natural language generation tasks. 
While our theoretical analysis provides links to stochastic tokenisation, further analysis is needed to fully characterise the optimal design of a stochastic tokenisation scheme.

\section*{Acknowledgements}

This work was partially supported by the Wallenberg AI, Autonomous Systems and Software Program (WASP) funded by the Knut and Alice Wallenberg Foundation.
RL and AS acknowledge funding from the Research Council of Finland (grant numbers 339730 and 362408). 
We acknowledge the computational resources provided by the Aalto Science-IT project. 

\section*{Impact Statement}

This work is concerned with the robustness of large-language models (LLMs) to non-canonical tokenisations and adversarial attacks on the tokenisation of the input.
Henceforth, our results and findings directly impact not only the community in machine learning research, such as indicating that stochastic tokenisation during training is a promising direction in providing robustness against such settings, but also raise awareness of the vulnerability to non-canonical tokenisations existing in LLMs.
Consequently, while we do not believe that our work has direct ethical implications, it does provide societal impact in the long run. 
In particular, it provides a guide that stochastic tokenisation during training of LLMs can substantially improve the robustness of LLMs to various attacks and therefore reduces the risk imposed on society by LLMs.

\bibliographystyle{icml2026}

\newpage
\appendix
\onecolumn

\newcommand{\appheader}[1]{%
  \noindent\makebox[\linewidth][l]{%
    \colorbox{gray!10}{%
      \begin{minipage}{0.98\linewidth}
        \emph{#1}%
      \end{minipage}%
    }%
  }\par\medskip
}

\FloatBarrier
\section{Training Details} \label{app:training}
\noindent \faGithub\xspace The code for the experiments is available at:
\url{https://github.com/stegsoph/stochastic-tokenisation-robustness} .

For LLMs trained from scratch, we use the architecture proposed in \citet{hillier2024super} and used in \citet{sims2026stochastok}. The architectural details are summarised in \Cref{tab:tinyllm_architecture}. The pre-training setup is summarised in \Cref{tab:tinyllm_pre-training}. The fine-tuning step uses full parameter updating, with details specified in \Cref{tab:tinyllm_instruction_tuning}.

Additionally, we use pre-trained weights and architectures from Hugging Face for \gptxl (\url{https://huggingface.co/openai-community/gpt2-xl}), \llamaoneb (\url{https://huggingface.co/meta-llama/Llama-3.2-1B}), \llamaeightb (\url{https://huggingface.co/meta-llama/Llama-3.1-8B}), 
\qwenzerosixb (\url{https://huggingface.co/Qwen/Qwen3-0.6B}), 
and 
\gemmaoneb (\url{https://huggingface.co/google/gemma-3-1b-pt}). For the fine-tuning step, we train LoRA adapters with the training setup as specified in \Cref{tab:hf_instruction_tuning}. 

\begin{table}[h]
\centering
\caption{Model architecture of \tinyllm. }
\label{tab:tinyllm_architecture}
\begin{tabular}{ll}
\toprule
\multicolumn{2}{c}{\textbf{Model architecture (\tinyllm)}} \\ \midrule

Number of layers 
& 8 \\

Hidden dimension 
& 512 \\

Attention heads 
& 16 (grouped, group size 4) \\

Feed-forward network 
& SwiGLU, dimension 1320 \\

Normalization 
& RMSNorm (no bias) \\

Positional encoding 
& Rotary positional embeddings (RoPE) \\

Context length 
& 512 \\
Tokeniser type
& GPT \\
Vocabulary size 
& 50{,}257 \\

Embedding tying 
& Input/output embeddings tied \\
\bottomrule
\end{tabular}
\end{table}

\begin{table}[h]
\centering
\caption{Pre-training setup for \tinyllm on OpenWebText.}
\label{tab:tinyllm_pre-training}
\begin{tabular}{ll}
\toprule
\multicolumn{2}{c}{\textbf{Pre-training (\tinyllm)}} \\
\midrule
Training objective 
& Autoregressive language modeling \\

Data set 
& OpenWebText \\

Tokenisation 
& Canonical, \stok ($\ppre \in \{0.1, 0.5\}$) \\

Training steps 
& 30{,}000 \\

Batch size 
& 480 \\

Optimizer 
& AdamW \\

Learning rate 
& $6\times10^{-4}$ (min.\ $6\times10^{-5}$) \\

Adam parameters 
& $\beta_1 = 0.9$, $\beta_2 = 0.95$ \\

Weight decay 
& 0.1 \\

Learning-rate schedule 
& Linear warmup (5k steps) + linear decay \\

Gradient clipping 
& 1.0 \\

Dropout 
& 0.1 \\
\bottomrule
\end{tabular}

\end{table}

\begin{table}[t]
\centering
\caption{Supervised fine-tuning setup for \tinyllm using stochastic tokenisation.}
\label{tab:tinyllm_instruction_tuning}
\begin{tabular}{l p{0.6\linewidth}}
\toprule
\multicolumn{2}{c}{\textbf{Supervised fine-tuning (\tinyllm)}} \\
\midrule
Training objective 
& Supervised fine-tuning \\

Tokenisation mode 
& Canonical, \stok, \stokuni, \unik, \uni, \bpedropout \\

Training steps 
& 3{,}000 \\

Batch size 
& 480 \\

Optimizer 
& AdamW \\

Learning rate 
& $1\times10^{-4}$ (min.\ $1\times10^{-5}$) \\

Adam parameters 
& $\beta_1 = 0.9$, $\beta_2 = 0.95$ \\

Weight decay 
& 0.05 \\

Warmup steps 
& 500 \\

Dropout 
& 0.1 \\
\bottomrule
\end{tabular}
\end{table}

\begin{table}[t]
\centering
\caption{Supervised fine-tuning configuration for Hugging Face models.}
\label{tab:hf_instruction_tuning}
\begin{tabular}{l p{0.6\linewidth}}
\toprule
\multicolumn{2}{c}{\textbf{Supervised fine-tuning (\gptxl, \llamaoneb, \llamaeightb, \qwenzerosixb, \gemmaoneb)}} \\
\midrule
Training objective
& Supervised fine-tuning \\

Tokenisation mode 
& Canonical, \stok, \stokuni, \unik, \uni, \bpedropout \\

Training steps
& 1{,}500 \\

Optimizer
& AdamW  \\

Learning rate 
& $1\times10^{-4}$ (min.\ $1\times10^{-5}$) \\

Adam parameters 
& $\beta_1 = 0.9$, $\beta_2 = 0.95$ \\

Weight decay 
& 0.01 \\

Warmup steps 
& 150 \\

LoRA rank ($r$) 
& 8 \\

LoRA scaling ($\alpha$) 
& 32 \\

LoRA dropout 
& 0.05 \\

Task type 
& Causal language modeling \\
\bottomrule
\end{tabular}
\end{table}

\FloatBarrier

\newpage
\section{Experimental Details} \label{app:exp_details}
In this section, we describe our evaluation setup in detail. 

\subsection{Normalised Number of Splits} \label{app:normalised_edit_distance}
Assume a string sequence $\bx=(x_1,\dots,x_s)$ and its canonical token sequence $\bv^c=(v_1^c,\dots,v_m^c)$. Let $\bv\sim P(\cdot\mid\bx)$ be a non-canonical tokenisation of $\bx$, where $\bv=(v_1,\dots,v_p) \in \mcT_\mcV(\bx)$. We define the per-token split-count vector $\bS = (S_1, \dots, S_m)$ that describes the number of splits induced in each canonical token. 

We define the normalised number of splits for a token sequence as
\[
\alpha(\bv,\bv^c) :=\frac{\sum_iS_i}{|\bv^c|},
\label{eq:normalised_edit_distance}
\]
which corresponds to the average number of splits per canonical token.

\subsection{Average Tokenisation Robustness}

We quantify average accuracy given stochastic tokenisations in a multiple-choice
setting. For each input question string $\bx$, we have answer options
$\{\ba_c\}_{c=1}^C$ and correct index $y\in\{1,\dots,C\}$. Assume the stochastic tokenisation of the query $\bv\sim P(\cdot\mid \bx;\peval)$, and canonical tokenisation of the answer options $\{\bv_{\ba_c}\}_{c=1}^C$. The model predicts the answer with the highest conditional probability
\[
\hat y\left(\bv, \{\bv_{\ba_c}\}_{c=1}^C \right) = \argmax_{c \in \{1,\dots,C\}} \; p(\bv_{\ba_c}\mid\bv).
\]
We approximate the expected accuracy under stochastic tokenisation by Monte Carlo sampling. We draw $M=10$ tokenisations $\{\bv^{(m)}\}_{m=1}^{M}$ and compute
\[
\mathrm{Acc}
=
\frac{1}{M}
\sum_{m=1}^{M}
\mathbf{1}\!\left(
\hat y\left(\bv^{(m)}, \{\bv_{\ba_c}\}_{c=1}^C \right) = y
\right).
\]

We report the average accuracy over $N=500$ data samples,
\[
\overline{\mathrm{Acc}}
=
\frac{1}{N}
\sum_{n=1}^{N}
\frac{1}{M}
\sum_{m=1}^{M}
\mathbf{1}\!\left(
\hat y\left(\bv_n^{(m)}, \{\bv_{\ba_c,n}\}_{c=1}^C \right) = y_n
\right).
\]

\subsection{Adversarial Tokenisation Robustness}
To evaluate adversarial tokenisation robustness, we use the greedy search algorithm by \citet{geh2025adversarial}, summarised in \Cref{alg:adv_tok}. For a given input string $\bx$, the attacker selects an initial tokenisation. Unless stated otherwise, we use the canonical tokenisation as a default initialisation. In ablation experiments, we report results for randomly drawn initial tokenisation. Let $\bv_n^\mathrm{adv}$ denote the adversarial tokenisation for input string $\bx_n$. 

We compare the clean average accuracy on the canonical tokenisation
\[
\overline{\mathrm{Acc}}_{\mathrm{clean}}
=
\frac{1}{N}
\sum_{n=1}^{N}
\ind\!\left(
\hat{y}\left(\tau(\bx_n), \{\bv_{\ba_c,n}\}_{c=1}^C \right) = y_n
\right)
\]
and the adversarial accuracy 
\[
\overline{\mathrm{Acc}}_{\mathrm{adv}}
=
\frac{1}{N}
\sum_{n=1}^{N}
\ind\!\left(
\hat{y}\left(\bv_n^{\mathrm{adv}}, \{\bv_{\ba_c,n}\}_{c=1}^C  \right) = y_n
\right).
\]

\subsection{In-context Learning}

Let $(\bv_i, \bv_{\ba_i})_{i=1}^K$ be $K$ tokenised context examples and let $\bv^*$ be a test query with answer options $\{ \bv_{\ba_c}^* \}_{c=1}^C$. 
To evaluate in-context learning robustness, we sample stochastic tokenisations for each context example $\bv_i \sim P(\bv_i\mid \bx_i;\alpha_{\mathrm{ICL}})$, and test whether this improves robustness to stochastic tokenisations of the query $\bv^* \sim P(\bv^*\mid \bx^*;\alpha_{\mathrm{eval}})$.
We denote the number of solved examples presented in the context by $K$. We report the average accuracy over $N=500$ data samples
\[
\overline{\mathrm{Acc}}
=
\frac{1}{N}
\sum_{n=1}^{N}
\ind\!\left(
\hat{y}\left( (\bv_1, \bv_{\ba_1}, \dots, \bv_K, \bv_{\ba_K}, \bv_n^{*}), \{\bv^*_{\ba_c,n}\}_{c=1}^C  \right) = y^*_n
\right)
\]
where $y_n^*$ denotes the ground-truth index and $\hat{y}(\cdot)$ the predicted answer index. 

For the adversarial attack, we define the adversarial objective as maximising the margin of the conditional probability assigned to each answer option $c$
\[
\bar{z}_c(\bv^*)=
p\!\left(
 \bv_{\ba_c} \mid \bv_1, \bv_{\ba_1}, \dots, \bv_K, \bv_{\ba_K}, \bv^*
\right).
\]

\subsection{Distances of Representations}

We analyse how training with different stochastic tokenisation schemes affects the internal representations of the same word with different non-canonical tokenisations by measuring the distance between its hidden states. We consider the 1,000 most common English words. For each word $\bx$, we draw 50 tokenisations uniformly from $\mcT_\mcV(\bx)$. 

Let $h_\ell(\bv) \in \reals^d$ denote the hidden representation at model layer
$\ell$ for tokenisation $\bv$, obtained via last-token pooling. For each layer $\ell$, we report the $L^2$ distance between the canonical tokenisation $v^c$ and a sampled tokenisation $\bv$, where the length of both vectors has been scaled to have unit norm: 
\[
d_\ell(\bv, \bv^c)
=
\left\|
\frac{h_\ell(\bv)}{\lVert h_\ell(\bv)\rVert_2}
-
\frac{h_\ell(\bv^c)}{\lVert h_\ell(\bv^c)\rVert_2}
\right\|_2
\]
We report the average of $d_\ell(\bv, \bv^c)$ across words and sampled tokenisations.

\FloatBarrier
\section{Data Sets} \label{app:datasets}

\newenvironment{ttitems}
{\begingroup\ttfamily\small\bitems}
{\eitems\endgroup}

\pgfmathsetseed{\number\pdfrandomseed} %

\newcommand{\randtoken}[1]{
    \pgfmathsetmacro{\thenuma}{int(random(50,200))}%
    \pgfmathsetmacro{\thenumb}{int(random(50,200))}%
    \pgfmathsetmacro{\thenumc}{int(random(50,200))}%
    \definecolor{randomcolor1}{RGB}{\thenuma,\thenumb,\thenumc}%
    \textcolor{randomcolor1}{\token}--%
}

\newcommand{\tokensequence}[1]{%
    \foreach \token in {#1} {%
        \!\!\!\randtoken{\token}%
    }%
}

\subsection{Language Games}

\textbf{\langgame} \cite{sims2026stochastok} consists of six question types that require subword understanding (count letter, contains letter, starts with string, ends with string, longest word, shortest word). For the training set we randomly sample from the \href{https://github.com/powerlanguage/word-lists/blob/master/1000-most-common-words.txt}{top 1k English words}. 

Below, we illustrate the effect of our stochastic tokenisation schemes for questions of the \langgame data set tokeniser with a Llama 3 tokeniser. Hyphens (``-``) indicate token boundaries.

\textsc{Canonical:}
\begin{ttitems}
\item \tokensequence{{Which},{ choice},{ has},{ the},{ most},{ letter},{ '},{b},{'s},{?},{ These},{ are},{ the},{ possible},{ option},{ words},{:},{ [},{ up},{,},{ baby},{,},{ rain},{,},{ glad},{].},{ Answer},{:},{ },{ baby}}
\item \tokensequence{{Which},{ option},{ ends},{ with},{ 'o'},{?},{ The},{ possible},{ option},{ strings},{:},{ [},{ no},{,},{ hard},{,},{ by},{,},{ happen},{].},{ Answer},{:},{ },{ no}}
\item \tokensequence{{What},{ word},{ is},{ the},{ shortest},{?},{ The},{ available},{ option},{ strings},{ are},{:},{ [},{ row},{,},{ suffix},{,},{ gather},{,},{ energy},{].},{ Answer},{:},{ },{ row}}
\end{ttitems}

\stok\ ($\alpha=1$):
\begin{ttitems}
\item \tokensequence{{Wh},{i},{ch},{ choice},{ has},{ the},{ most},{ le},{t},{t},{er},{ '},{b},{'s},{?},{ These},{ are},{ the},{ possible},{ option},{ words},{:},{ [},{ up},{,},{ baby},{,},{ rain},{,},{ g},{lad},{].},{ A},{n},{sw},{er},{:},{ },{ baby}}
\item \tokensequence{{Which},{ option},{ end},{s},{ with},{ 'o'},{?},{ The},{ possible},{ option},{ s},{tring},{s},{:},{ [},{ no},{,},{ hard},{,},{ b},{y},{,},{ happen},{].},{ Answer},{:},{ },{ no}}
\item \tokensequence{{Wh},{at},{ word},{ i},{s},{ the},{ shortest},{?},{ The},{ available},{ option},{ string},{s},{ are},{:},{ [},{ r},{o},{w},{,},{ suffix},{,},{ g},{ather},{,},{ e},{n},{ergy},{].},{ Ans},{wer},{:},{ },{ row}}
\end{ttitems}

\stokuni\ ($\alpha=1$):
\begin{ttitems}
\item \tokensequence{{W},{hic},{h},{ choice},{ has},{ the},{ mo},{s},{t},{ letter},{ '},{b},{'},{s},{?},{ The},{se},{ are},{ the},{ possible},{ opt},{ion},{ words},{:},{ [},{ u},{p},{,},{ baby},{,},{ rain},{,},{ g},{la},{d},{].},{ Answer},{:},{ },{ baby}}
\item \tokensequence{{Which},{ op},{ti},{on},{ en},{ds},{ w},{i},{t},{h},{ 'o'},{?},{ The},{ possible},{ op},{tion},{ strings},{:},{ [},{ no},{,},{ h},{a},{r},{d},{,},{ by},{,},{ happen},{].},{ Answer},{:},{ },{ no}}
\item \tokensequence{{What},{ wo},{rd},{ is},{ th},{e},{ shortest},{?},{ The},{ available},{ o},{p},{tion},{ strings},{ a},{r},{e},{:},{ [},{ row},{,},{ su},{ff},{ix},{,},{ gather},{,},{ ener},{gy},{].},{ Answer},{:},{ },{ row}}
\end{ttitems}

\unik\ ($\alpha=1$):
\begin{ttitems}
\item \tokensequence{{Which},{ c},{h},{o},{ic},{e},{ has},{ th},{e},{ mos},{t},{ letter},{ '},{b},{'s},{?},{ The},{se},{ are},{ the},{ poss},{ible},{ op},{tion},{ wo},{r},{ds},{:},{ [},{ up},{,},{ ba},{by},{,},{ rain},{,},{ glad},{].},{ An},{s},{we},{r},{:},{ },{ baby}}
\item \tokensequence{{W},{h},{ic},{h},{ option},{ end},{s},{ w},{ith},{ 'o'},{?},{ T},{he},{ possible},{ option},{ str},{ing},{s},{:},{ [},{ no},{,},{ h},{ard},{,},{ by},{,},{ happen},{].},{ An},{sw},{er},{:},{ },{ no}}
\item \tokensequence{{What},{ w},{o},{rd},{ is},{ the},{ s},{hort},{est},{?},{ The},{ a},{vailable},{ option},{ s},{t},{r},{i},{n},{gs},{ are},{:},{ [},{ row},{,},{ suffix},{,},{ ga},{ther},{,},{ energy},{].},{ Ans},{w},{er},{:},{ },{ row}}
\end{ttitems}

\uni:
\begin{ttitems}
\item \tokensequence{{Wh},{i},{c},{h},{ cho},{i},{c},{e},{ h},{as},{ t},{h},{e},{ m},{o},{st},{ le},{tte},{r},{ '},{b},{'s},{?},{ Th},{es},{e},{ a},{r},{e},{ t},{he},{ po},{ssi},{b},{l},{e},{ op},{t},{ion},{ words},{:},{ [},{ u},{p},{,},{ b},{a},{b},{y},{,},{ r},{ain},{,},{ g},{lad},{].},{ A},{ns},{w},{e},{r},{:},{ },{ baby}}
\item \tokensequence{{W},{h},{i},{ch},{ op},{ti},{on},{ ends},{ wi},{t},{h},{ 'o'},{?},{ T},{h},{e},{ po},{s},{s},{ible},{ option},{ s},{tr},{i},{n},{g},{s},{:},{ [},{ no},{,},{ ha},{rd},{,},{ b},{y},{,},{ ha},{p},{pe},{n},{].},{ Ans},{w},{er},{:},{ },{ no}}
\item \tokensequence{{W},{hat},{ word},{ i},{s},{ th},{e},{ s},{ho},{r},{t},{e},{st},{?},{ Th},{e},{ a},{v},{a},{i},{l},{abl},{e},{ op},{t},{i},{o},{n},{ s},{t},{ring},{s},{ a},{re},{:},{ [},{ row},{,},{ s},{u},{ff},{i},{x},{,},{ g},{a},{th},{er},{,},{ e},{n},{er},{g},{y},{].},{ An},{s},{wer},{:},{ },{ row}}
\end{ttitems}

\textbf{\cute} \cite{edman2024cute} tests the Character-level Understanding of Tokens of models. The original data sets was intended for zero-shot evaluation, providing only a test split. Thus, we follow \citet{sims2026stochastok} in generating a custom \cute data set for training and evaluating the smaller scale LLMs. We generate questions for seven subword tasks (contains letter, delete letter, insert letter, orthography, spelling, inverse spelling, substitute letter). For the training set we randomly sample from the \href{https://github.com/powerlanguage/word-lists/blob/master/1000-most-common-words.txt}{top 1k English words}. 

\textsc{Canonical:}
\begin{ttitems}
\item \tokensequence{{Spell},{ out},{ the},{ word},{ "},{ stood},{ ".},{ Answer},{:},{s},{ t},{ o},{ o},{ d}}
\item \tokensequence{{Is},{ there},{ a},{ "},{ e},{ "},{ in},{ "},{ salt},{ "?},{ Answer},{:},{No}}
\item \tokensequence{{Delete},{ every},{ instance},{ of},{ "},{ c},{ "},{ in},{ "},{ molecule},{ ".},{ Answer},{:},{m},{ole},{ule}}
\end{ttitems}

\subsection{General Benchmarks}

\textbf{\textsc{ARC-E}} \cite{clarc2018think} consists of natural, grade-school science questions with four answer options. \\

\textsc{Canonical:}
\begin{ttitems}
\item \tokensequence{{Which},{ factor},{ will},{ most},{ likely},{ cause},{ a},{ person},{ to},{ develop},{ a},{ fever},{?},{ a},{ bacterial},{ population},{ in},{ the},{ bloodstream}}
\item \tokensequence{{L},{ich},{ens},{ are},{ symb},{iotic},{ organisms},{ made},{ of},{ green},{ algae},{ and},{ fungi},{.},{ What},{ do},{ the},{ green},{ algae},{ supply},{ to},{ the},{ fungi},{ in},{ this},{ symb},{iotic},{ relationship},{?},{ food}}
\item \tokensequence{{When},{ a},{ switch},{ is},{ used},{ in},{ an},{ electrical},{ circuit},{,},{ the},{ switch},{ can},{ stop},{ and},{ start},{ the},{ flow},{ of},{ current},{.}}
\end{ttitems}

\textbf{\textsc{COPA}} \cite{roemmele2011choice} contains multiple-choice questions that test causaul reasoning with two answer options. \\

\textsc{Canonical:}
\begin{ttitems}
\item \tokensequence{{My},{ body},{ cast},{ a},{ shadow},{ over},{ the},{ grass},{.},{ cause},{?},{ The},{ sun},{ was},{ rising},{.}}
\item \tokensequence{{The},{ woman},{ tolerated},{ her},{ friend},{'s},{ difficult},{ behavior},{.},{ cause},{?},{ The},{ woman},{ knew},{ her},{ friend},{ was},{ going},{ through},{ a},{ hard},{ time},{.}}
\item \tokensequence{{The},{ women},{ met},{ for},{ coffee},{.},{ cause},{?},{ They},{ wanted},{ to},{ catch},{ up},{ with},{ each},{ other},{.}}
\end{ttitems}

\textbf{\textsc{CSQA}} \cite{talmor2019commonsenseqa} contains multiple-choice questions that test commonsense knowledge with five answer options. \\

\textsc{Canonical:}
\begin{ttitems}
\item \tokensequence{{The},{ sanctions},{ against},{ the},{ school},{ were},{ a},{ punishing},{ blow},{,},{ and},{ they},{ seemed},{ to},{ what},{ the},{ efforts},{ the},{ school},{ had},{ made},{ to},{ change},{?},{ ignore}}
\item \tokensequence{{Sam},{my},{ wanted},{ to},{ go},{ to},{ where},{ the},{ people},{ were},{.},{ Where},{ might},{ he},{ go},{?},{ populated},{ areas}}
\item \tokensequence{{To},{ locate},{ a},{ ch},{oker},{ not},{ located},{ in},{ a},{ jewelry},{ box},{ or},{ boutique},{ where},{ would},{ you},{ go},{?},{ jewelry},{ store}}
\end{ttitems}

\textbf{\textsc{HellaSwag}} \cite{zellers2019hellaswag} contains multiple-choice questions that test the most natural continuation given four answer options. \\

\textsc{Canonical:}
\begin{ttitems}
\item \tokensequence{{[},{-header},{-]},{ How},{ to},{ dry},{ pant},{yh},{ose},{ in},{ a},{ hurry},{ [},{-title},{-]},{ Lay},{ the},{ pant},{yh},{ose},{ flat},{ on},{ your},{ towel},{.},{ [},{-step},{-]},{ Use},{ a},{ dry},{ bath},{ towel},{ for},{ best},{ results},{.},{ The},{ thicker},{ the},{ towel},{ is},{,},{ the},{ more},{ effective},{ this},{ method},{ will},{ be},{.},{ If},{ you},{ do},{ not},{ have},{ a},{ thick},{ towel},{,},{ repeat},{ the},{ method},{ using},{ two},{ separate},{ towels},{.},{ [},{-title},{-]},{ Roll},{ pant},{yh},{ose},{ into},{ your},{ thick},{ towel},{.}}
\item \tokensequence{{[},{-header},{-]},{ How},{ to},{ make},{ german},{ style},{ meat},{lo},{af},{ [},{-title},{-]},{ Combine},{ beef},{,},{ pork},{,},{ onions},{,},{ beaten},{ eggs},{,},{ bread},{ crumbs},{,},{ 3},{ tablespoons},{ water},{,},{ parsley},{,},{ pap},{rika},{,},{ prepared},{ mustard},{ and},{ salt},{.},{ [},{-step},{-]},{ Mix},{ well},{.},{ [},{-title},{-]},{ Press},{ the},{ mixture},{ into},{ the},{ 8},{ "},{ baking},{ pan},{.},{ [},{-step},{-]},{ Place},{ the},{ 3},{ hard},{ cooked},{ eggs},{ in},{ a},{ row},{ down},{ the},{ middle},{ of},{ the},{ meat},{ mixture},{,},{ then},{ fold},{ up},{ both},{ sides},{,},{ so},{ you},{ have},{ a},{ loaf},{.},{ [},{-title},{-]},{ Pre},{-heat},{ the},{ oven},{ to},{ 375},{.}}
\item \tokensequence{{[},{-header},{-]},{ How},{ to},{ feel},{ full},{ without},{ eating},{ [},{-title},{-]},{ Chew},{ a},{ piece},{ of},{ gum},{.},{ [},{-step},{-]},{ Ch},{ewing},{ a},{ piece},{ of},{ gum},{ will},{ trigger},{ your},{ brain},{ and},{ stomach},{ to},{ believe},{ you},{ are},{ about},{ to},{ eat},{ or},{ feeling},{ full},{.},{ This},{ will},{ not},{ only},{ stimulate},{ your},{ mind},{ to},{ feel},{ full},{,},{ but},{ will},{ also},{ make},{ sure},{ that},{ your},{ mouth},{ is},{ too},{ busy},{ to},{ eat},{.},{ [},{-sub},{-steps},{-]},{ Make},{ sure},{ to},{ chew},{ sugar},{less},{ gum},{ so},{ that},{ you},{ don},{'t},{ get},{ unnecessary},{ calories},{.},{ Ch},{ewing},{ gum},{ can},{ even},{ burn},{ 11},{ calories},{ an},{ hour},{.}}
\end{ttitems}

\textbf{\textsc{PIQA}} \cite{bisk2020piqa} contains multiple-choice questions that test knowledge about the physical world with two answer options. \\

\textsc{Canonical:}
\begin{ttitems}
\item \tokensequence{{When},{ boiling},{ butter},{,},{ when},{ it},{'s},{ ready},{,},{ you},{ can},{ Pour},{ it},{ into},{ a},{ jar}}
\item \tokensequence{{To},{ permanently},{ attach},{ metal},{ legs},{ to},{ a},{ chair},{,},{ you},{ can},{ Weld},{ the},{ metal},{ together},{ to},{ get},{ it},{ to},{ stay},{ firmly},{ in},{ place}}
\item \tokensequence{{how},{ do},{ you},{ indent},{ something},{?},{ leave},{ a},{ space},{ before},{ starting},{ the},{ writing}}
\end{ttitems}

\textbf{\textsc{SocialIQA}} contains multiple-choice questions that test social commonsense knowledge with three answer options.\\

\textsc{Canonical:}
\begin{ttitems}
\item \tokensequence{{Choose},{ the},{ most},{ likely},{ answer},{.}}

\tokensequence{{Context},{:},{ After},{ weeks},{ apart},{ from},{ one},{ another},{ Austin},{ invited},{ their},{ boyfriend},{ to},{ dinner},{.}}

\tokensequence{{Question},{:},{ Why},{ did},{ Austin},{ do},{ this},{?}}

\tokensequence{{ reconnect}}
    
\item \tokensequence{{Choose},{ the},{ most},{ likely},{ answer},{.}}

\tokensequence{{Context},{:},{ Lee},{ spent},{ extra},{ time},{ to},{ explain},{ clearly},{ the},{ topic},{ purpose},{ to},{ the},{ students},{.}}

\tokensequence{{Question},{:},{ How},{ would},{ you},{ describe},{ Lee},{?}}

\tokensequence{{ careful}}

\item \tokensequence{{Choose},{ the},{ most},{ likely},{ answer},{.}}

\tokensequence{{Context},{:},{ The},{ man},{ and},{ woman},{ were},{ too},{ tired},{ to},{ cook},{,},{ so},{ Bailey},{ made},{ their},{ breakfast},{.}}

\tokensequence{{Question},{:},{ How},{ would},{ Bailey},{ feel},{ afterwards},{?}}

\tokensequence{{ helpful}}
\end{ttitems}

\FloatBarrier
\section{Pseudocodes} \label{app:pseudocodes}

\begin{algorithm}[h]
\caption{\stok: stochastic tokenisation via random pairwise splits}
\label{alg:stochastok}
\begin{algorithmic}[1]
\REQUIRE String $\bx\in\mcX^*$; token vocabulary $\mcV$; canonical tokenizer $\tau$;
split map $\mcS$; expansion proportion $\alpha\ge 0$; max splits $K_{\max}$
\STATE $\bv \leftarrow \tau(\bx)$ \COMMENT{canonical tokenisation}
\STATE $K \leftarrow \min\!\big(K_{\max}, \lceil \alpha \cdot |\bv| \rceil\big)$
\FOR{$k=1$ to $K$}
    \STATE $i \leftarrow \texttt{randomInteger}(1,|\bv|)$
    \STATE $t \leftarrow \bv_i$
    \IF{$\mcS(t)\neq \emptyset$}
        \STATE $(a,b) \leftarrow \texttt{randomChoice}(\mcS(t))$
        \STATE $\bv \leftarrow (v_1,\dots,v_{i-1},\, a, b,\, v_{i+1},\dots,v_{|v|})$
    \ENDIF
\ENDFOR
\STATE \textbf{return} $\bv$
\end{algorithmic}
\end{algorithm}

\begin{algorithm}[h]
\caption{Greedy adversarial tokenisation (radius-2 neighbourhood)}
\label{alg:adv_tok}
\begin{algorithmic}[1]
\REQUIRE string $\bx$, correct option $y$, scoring function $z_c(\bv)$, start tokenisation $\bv_0\in\mcT_\mcV(\bx)$, max steps $S$
\STATE $\bv \gets \bv_0$
\FOR{$s=1,\dots,S$}
    \STATE $\mathcal{N} \gets \mathrm{Ne}(\bv)=\mcT_\mcV^2(\bx,\bv)$
    \STATE $\bv^\star \gets \argmax_{\bu\in\mathcal{N}}\left(\max_{c\neq y} z_c(\bu)-z_y(\bu)\right)$
    \IF{$\max_{c\neq y} z_c(\bv^\star)-z_y(\bv^\star)\le \max_{c\neq y} z_c(\bv)-z_y(\bv)$}
        \STATE \textbf{break} \COMMENT{local optimum}
    \ENDIF
    \STATE $\bv \gets \bv^\star$
\ENDFOR
\STATE \textbf{return} $\bv$
\end{algorithmic}
\end{algorithm}

\FloatBarrier

\newpage
\section{Proof of \Cref{theo:upper_bound_robustness}}\label{app:proof_theorem}
Beyond the empirical results outlined in the previous section, we aim to theoretically characterise how tokenisation affects the robustness of the resulting model. In this perspective, we focus on evasion attacks, \ie, adversarial perturbations applied at test time. 
This threat model is particularly relevant in practical settings, where attackers typically have access only to the deployed model during inference. 
In this direction, let's consider a trained classifier $f\colon \mcV \to \mcY$ and let $v \in \mcV $ be an input with its associated label $y \in \mcY$, such that $f(v) = y$. The goal of an attacker is to craft a small additional perturbation to the input, such as to generate a point $\tilde{v}$ whose prediction $f(\tilde{v})$ is different from the original one. We note that in our setting, the function $f$ represents the LLM \emph{after} the tokeniser, comprising the embedding function and subsequent transformer blocks with self-attention.

As the generated adversarial perturbation should be ``similar'' to the original input, we consider a similarity budget $\epsilon$ (the maximum possible changes to the input), comprising the set of valid adversarial perturbations defined by a neighbourhood $\mcB(x, \epsilon)$. 
In our setting $\mcB(x, \epsilon) = \mcT^k_{\mcV}(x)$ with edit distance at most $k$. 
Now, the adversarial aim is to find, within $\mcB(x, \epsilon)$, points that not only satisfy the adversarial aim of flipping the classification but also result in the worst prediction. 
In this direction, given our model $f$, the model's \textit{adversarial risk} 
can be formulated as:
\begin{equation}\label{equation:robustness_definition_appendix}
\mathcal{R}_{\epsilon}[f] = \mathop{\mathbb{E}}_{x \in \mathcal{D}_{\mathcal{X}}} \left[\sup_{\tilde{x} \in \mathcal{B}(x, \epsilon)} d_{\mathcal{Y}}\left(f\left(\tilde{x}\right), f\left(x\right)\right)\right] .
\end{equation}
From a defence perspective, the objective is to ensure that the risk remains small, thereby keeping predictions within the same output region and making it harder to find a valid perturbation. This goal can be formulated as:
\begin{definition}[Adversarial Robustness]\label{def:adv_robustness_appendix}
    The classifier $f$ is said to be $(\epsilon, \gamma)$-\emph{robust} if its adversarial risk with respect to the classifier $f$ satisfies: $\mathcal{R}_{\epsilon}[f] \leq \gamma.$
\end{definition}

\Cref{def:adv_robustness_appendix} uses an upper-bound perspective ($\gamma$) as computing the exact risk is generally intractable. 
A smaller value of $\gamma$ indicates stronger robustness guarantees. Consequently, comparing the resulting bounds provides a principled way to assess and contrast the model's robustness. 
Without loss of generality, we consider a $1$-layer Transformer-based model in which all activation functions are assumed to be $1$-Lipschitz, which is the case for many common activations, \eg,  \texttt{tanh} and \texttt{sigmoid} \cite{virmaux18}. We additionally assume the output of the embedding function to be bounded in terms of its norm (\eg, $[0, 1]^{n \times d}$, with $n$ being the number of tokens and $d$ the hidden dimension). In this setting, the model's underlying robustness, $\gamma$, can be derived and is given by the following.

\begin{theorem*}
Let $f\colon \mathcal{X}\rightarrow\mathcal{Y}$ be a classifier following our considered problem setup, then $f$ is $(\epsilon, \gamma)$-\emph{robust}, with: 
    $$\gamma = \sqrt{2} \lVert W\lVert
    \left(\frac{d}{d-1}\right)^2
    C_1 C_2 \sqrt{k},$$

\begin{align*}
    \text{with } & C_1 = 1 + \lVert W_O \lVert \sqrt{H} \max_h \big[ \lVert W^{V, h} \lVert \big[\frac{4}{\sqrt{{d/H}}} \lVert W^{Q, h} \lVert \lVert W^{K, h} \lVert + 1 \big]  \big], \\
    & C_2 = \big(1 + \lVert W_{FFN} \lVert\big) \lVert W_{\text{out}} \lVert.
\end{align*}
\end{theorem*}

\begin{proof}

We recall our considered problem setup, in which $f$ is a classifier (a large-language model).
In the following analysis, we will ignore the reliance on the tokeniser $\tau$ and focus on the remaining function, \ie, 
\begin{equation}
    f = A \circ E ,
\end{equation}
where:
\begin{itemize}
\item $E:\cV^{*} \to \R^{d\times n}$ an embedding function, mapping from tokens to a continuous space.
    \item $A:\mathbb{R}^{d\times n} \to \mathbb{R}^{m}$ an attention-block, mapping from a set of tokens to a classification, for instance. Note that we consider the pooling operation within this block. 
\end{itemize}

Now recall that $\mcX$ denotes the space of raw inputs (\eg, strings or words), $\cV$ is the token vocabulary, and let $\cV^{*}$ be the set of finite token sequences where $n$ is the length of the sequence, which depends on the input's length and the considered tokenisation.
Further, let $d$ be the embedding dimension. 

We now consider that each space within our components is equipped with an appropriate metric or distance, \ie, $d_{\cV^*}$ (edit distance) on token sequences, and the L2 norm in Euclidean spaces.

Given a character sequence $\bx$, we denote:
\begin{equation}
    s_1 \in \mcV^*, \quad s_2 \in \mcT^k_{\mcV^*}(\bx, s_1) ,
\end{equation}
where $s_1$ is, for example, the canonical tokenisation of a string, and $s_2$ denotes an adversarially perturbed counterpart within the $ k$-neighbourhood of $s_1$.
Consequently, the adversarial perturbation induces at most $k$ token-level changes, which can be written as:
\begin{equation}
    d_{\mathcal V^*}(s_1, s_2) \leq k.
\end{equation}
The effect of the perturbation can first be shown within the embedding level, and we therefore start by writing the effect of this perturbation. 
Let $H(s) \in \{0,1\}^{\mid\mcV\mid \times n}$ denote the one-hot representation of a token sequence $s = (v_1,\dots,v_n)$, so that the embedding function can be written as
\begin{equation}
E(s) = W H(s),
\end{equation}
where $W \in \reals^{d \times \mid\mcV\mid }$ is the learned embedding matrix.

Each token substitution affects exactly one column of the matrix $H(s)$, replacing a one-hot vector by another one-hot vector. Therefore, if $s_1$ and $s_2$ differ in at most $k$ token positions, the matrix $H(s_1) - H(s_2)$ has at most $k$ non-zero columns, each with $\ell_2$-norm $\sqrt{2}$. We can consequently formulate the following:
\begin{equation}
    \lVert H(s_1) - H(s_2)\lVert \leq \sqrt{2k}.
\end{equation}

And consequently, we can also derive the following:
\begin{equation}
    \norm{E(s_1) - E(s_2)} = \norm{W(H(s_1) - H(s_2))} \leq \norm{W}  \norm{H(s_1) - H(s_2)} \leq \sqrt{2} \norm{W} \sqrt{k}.
\end{equation}

Now, let's consider the second part, which relates to the self-attention block. 
Let $z_e$ (corresponding to $s_1)$ and $\tilde{z_e}$ (the perturbed version corresponding to $s_2$) be the input to this specific part of the model. In practice an additional linear head is added at the end to get the final predictions, in our case we consider it to be $W_{\text{out}}$. 

Additionally, as specific in Section \ref{sec:theo_analysis}, we consider that the output of the embedding function to be bounded in terms of its norm (\eg, $[0, 1]^{n \times d}$, with $n$ being the number of tokens and $d$ the hidden dimension). Consequently, based on the previous assumptions and following previous work \cite{ennadir2025pool}, we have the following inequality:
\begin{equation}
    \lVert A(z_e) - A(\tilde{z_e}) \lVert \leq \lVert z_e - \tilde{z_e} \lVert \left(\frac{d}{d-1}\right)^2 C_1 C_2
\end{equation}
with
\begin{align}
    C_1 &= 1 +\lVert W_O \lVert \sqrt{H} \max_h \Biggl[ \Bigl\lVert W^{V,h} \Bigr\rVert \left(4 \frac{n}{\sqrt{d/H}} \Bigl\lVert W^{Q,h} \Bigr\rVert \Bigl \lVert W^{K,h} \Bigr\rVert + 1\right) \Biggr] \\
    C_2 &= (1 + \Bigl\lVert W_{FFN} \Bigr\rVert) \lVert W_{\text{out}} \lVert.
\end{align}

Hence, after simple substitution we have:
\begin{align}
    \lVert f(s_1) - f(s_2) \lVert &\leq  \lVert A(E(s_1)) - A(E(s_2)) \lVert \\
    & \leq \sqrt{2} \lVert W \lVert
    \left(\frac{d}{d-1}\right)^2
    C_1 C_2
    \sqrt{k}.
\end{align}

\end{proof}

\section{Distributional Properties of Stochastic Tokenisations} \label{app:distribution}

Let $\bx=(x_1,\dots,x_s)$ be a string sequence and $\bv^c=(v_1^c,\dots,v_m^c)$ its canonical tokenisation. We sample a non-canonical tokenisation $\bv\sim P(\cdot\mid\bx)$ from a stochastic tokenisation scheme, where $\bv=(v_1,\dots,v_p) \in \mcT_\mcV(\bx)$. We define a per-token split-count vector $\bS = (S_1, \dots, S_m)$ that describes the number of splits induced in each canonical token. The split vector $\bS$ is a random variable. In the following, we discuss the probability distribution of $\bS$ for each of our considered stochastic tokenisation schemes. 

\paragraph{\stok}
\stok applies $N$ token-level splits in an iterative sampling process (see \Cref{alg:stochastok}). At each step, a token $v_i$ is selected uniformly at random and replaced by a pair of subtokens $(t_1, t_2)\in\mcV^2$ with $t_1\circ t_2=v_1$. In the next step, a token from this extended token sequence is selected uniformly. Now, there are two possibilities to select a token that corresponds to the original token $v_i$. This procedure corresponds exactly to a Polya-urn model. Thus, the number of times each original token is split follows a \emph{Dirichlet–multinomial distribution} with symmetrical concentration parameter
\[
\bS \sim \distDirMult(N, \balpha), \quad \text{with } \balpha = \b1.
\]

\paragraph{\stokuni} We construct \stokuni to follow the per-token split-count distribution $\bS$ of \stok. We draw $\bS \sim \distDirMult(N, \b1)$, and sample the tokenisations for each canonical token $v_i$ with the specified split count $S_i$. Recall that each token $v_i^c\in\mcV$ corresponds to a character sequence, which we denote by $\tilde{\bx}_i \coloneqq \kappa(v_i^c)\in\mcX^*$.
Thus, we sample 
$
\bv_i \sim \distUnif \left( \mcT^{S_i+1}_\mcV\left(v^c_i\right) \right).$

\paragraph{\unik} \unik selects a non-canonical tokenisation uniformly at random from the set of all possible valid segmentations with edit distance $k$, \ie 
$\bv \sim \distUnif \left( \mcT^{k}_\mcV\left(\bx \right) \right).$ 

Next, we want to describe the distribution over the per-token split-count $\bS=(S_1,\dots,S_m)$ this scheme induces. To simplify, we consider only splits within canonical token boundaries. 
For a canonical token $v_i^c$, let $A_i(s)=|\mcT_\mcV^s(v^c_i)|$ denote the number of non-canonical tokenisations of $v_i^c$ with exactly $s$ splits. Using these counts, we construct a polynomial
\[
G_i(x) \;=\; \sum_{s\ge 0} A_i(s)\,x^{\,s}.
\]
The total number of tokenisations of $\bv^c$ with exactly $k$ splits can be computed via the product \[G(x) = \prod_{i=1}^m G_i(x)\] via the $k$-th coefficient $[x^k]G(x)$. 
To obtain the probability of canonical token $v^c_i$ having exactly $S_i=s$ splits, we define
\[
H_i(x) \;=\; \prod_{j \neq i} G_j(x),
\]
and obtain
\[
\Pr\left(S_i=s \mid \sum_j S_j=k\right)
=
\frac{A_i(s)\,[x^{k-s}]\,H_i(x)}{[x^k]\,G(x)}~.
\]

\paragraph{\uni} \uni samples tokenisations uniformly from all valid non-canonical tokenisations, $\bv\sim\distUnif(\mcT_\mcV(\bx))$. For simplification, assume an infinite vocabulary $\mcV$. Let $S_i$ be the number of splits within $v_i^c$. Then, $S_i$ follows a \textit{binomial} distribution
\[
S_i \sim \distBinom(|v_i^c|-1, \nicefrac{1}{2})
\]
as \uni corresponds to flipping a fair coin at every internal boundary.

\FloatBarrier
\subsection{Empirical distribution over Per-token Split-counts} \label{app:sec:splitcounts}

We test whether our theoretical distributions for the per-token split-count distributions match the empirical distribution for the different stochastic tokenisation schemes. 

We take an example input sequence $\bx=$(\texttt{revolution is a rapid, fundamental transformation of a society's class, state, ethnic or religious structures}). We sample $D=50,000$ draws of tokenisations $\bv\sim P(\cdot\mid\bx)$ and report the empirical distribution over the number of splits $S_1$ for the first token $v_1=\texttt{revolution}$.

\begin{figure}[h]
\begin{subfigure}[t]{0.48\linewidth}
    \centering
    \begin{tikzpicture}
\begin{axis}[
  height=\figureheight,
  width=\figurewidth,
  bar width=6pt,
  enlarge x limits=0.15,
  tick align=outside,
  tick pos=left,
  xtick style={color=black},
  ytick style={color=black},
  ymajorgrids,
  grid style={scGrey},
  xlabel={Number of splits $S_1$},
  ylabel={Probability (\%)},
xticklabel style={rotate=0},
  yticklabel style={rotate=0},
ymin=0, ymax=0.511538461538462,
  scaled y ticks=false,
  yticklabel={
    \pgfmathparse{100*\tick}\pgfmathprintnumber[fixed,precision=0]{\pgfmathresult}},
  legend cell align={left},
]

\addlegendimage{ybar,ybar legend,draw=none,fill=scBlue}
\addlegendentry{Simulation}

\draw[draw=none,fill=scBlue] (axis cs:-0.3,0) rectangle (axis cs:0.3,0.45782);
\draw[draw=none,fill=scBlue] (axis cs:0.7,0) rectangle (axis cs:1.3,0.25306);
\draw[draw=none,fill=scBlue] (axis cs:1.7,0) rectangle (axis cs:2.3,0.15838);
\draw[draw=none,fill=scBlue] (axis cs:2.7,0) rectangle (axis cs:3.3,0.08542);
\draw[draw=none,fill=scBlue] (axis cs:3.7,0) rectangle (axis cs:4.3,0.03138);
\draw[draw=none,fill=scBlue] (axis cs:4.7,0) rectangle (axis cs:5.3,0.01056);
\draw[draw=none,fill=scBlue] (axis cs:5.7,0) rectangle (axis cs:6.3,0.00266);
\draw[draw=none,fill=scBlue] (axis cs:6.7,0) rectangle (axis cs:7.3,0.00064);
\draw[draw=none,fill=scBlue] (axis cs:7.7,0) rectangle (axis cs:8.3,6e-05);
\draw[draw=none,fill=scBlue] (axis cs:8.7,0) rectangle (axis cs:9.3,2e-05);

\addplot[very thick, black, mark=+, mark size=3]
table[row sep=\\]{
0 0.487179487179487\\
1 0.256410256410257\\
2 0.131670131670132\\
3 0.0658350658350659\\
4 0.0319770319770321\\
5 0.015048015048015\\
6 0.00684000684000683\\
7 0.00299250299250299\\
8 0.00125492060975932\\
};
\addlegendentry{Theoretical}

\end{axis}
\end{tikzpicture}
    \caption{\stok}
    \label{fig:distribution_length_stochastok}
\end{subfigure}
\begin{subfigure}[t]{0.48\linewidth}
    \centering
    \begin{tikzpicture}
\begin{axis}[
  height=\figureheight,
  width=\figurewidth,
  bar width=6pt,
  enlarge x limits=0.15,
  tick align=outside,
  tick pos=left,
  xtick style={color=black},
  ytick style={color=black},
  ymajorgrids,
  grid style={scGrey},
  xlabel={Number of splits $S_1$},
  ylabel={Probability (\%)},
xticklabel style={rotate=0},
  yticklabel style={rotate=0},
ymin=0, ymax=0.511728,
  scaled y ticks=false,
  yticklabel={
    \pgfmathparse{100*\tick}\pgfmathprintnumber[fixed,precision=0]{\pgfmathresult}},
  legend cell align={left},
]

\addlegendimage{ybar,ybar legend,draw=none,fill=scBlue}
\addlegendentry{Simulation}

\draw[draw=none,fill=scBlue] (axis cs:-0.3,0) rectangle (axis cs:0.3,0.48736);
\draw[draw=none,fill=scBlue] (axis cs:0.7,0) rectangle (axis cs:1.3,0.2568);
\draw[draw=none,fill=scBlue] (axis cs:1.7,0) rectangle (axis cs:2.3,0.1301);
\draw[draw=none,fill=scBlue] (axis cs:2.7,0) rectangle (axis cs:3.3,0.06772);
\draw[draw=none,fill=scBlue] (axis cs:3.7,0) rectangle (axis cs:4.3,0.03188);
\draw[draw=none,fill=scBlue] (axis cs:4.7,0) rectangle (axis cs:5.3,0.01466);
\draw[draw=none,fill=scBlue] (axis cs:5.7,0) rectangle (axis cs:6.3,0.00632);
\draw[draw=none,fill=scBlue] (axis cs:6.7,0) rectangle (axis cs:7.3,0.00296);
\draw[draw=none,fill=scBlue] (axis cs:7.7,0) rectangle (axis cs:8.3,0.00128);
\draw[draw=none,fill=scBlue] (axis cs:8.7,0) rectangle (axis cs:9.3,0.00092);

\addplot[very thick, black, mark=+, mark size=3]
table[row sep=\\]{
0 0.487179487179487\\
1 0.256410256410257\\
2 0.131670131670132\\
3 0.0658350658350659\\
4 0.0319770319770321\\
5 0.015048015048015\\
6 0.00684000684000683\\
7 0.00299250299250299\\
8 0.00125492060975932\\
};
\addlegendentry{Theoretical}

\end{axis}
\end{tikzpicture}
    \caption{\stokuni}
    \label{fig:distribution_length_stochastok_uni}
\end{subfigure}
\\[1cm]
\vspace*{1cm}
\begin{subfigure}[t]{0.48\linewidth}
    \centering
    \begin{tikzpicture}
\begin{axis}[
  height=\figureheight,
  width=\figurewidth,
  bar width=6pt,
  enlarge x limits=0.15,
  tick align=outside,
  tick pos=left,
  xtick style={color=black},
  ytick style={color=black},
  ymajorgrids,
  grid style={scGrey},
  xlabel={Number of splits $S_1$},
  ylabel={Probability (\%)},
xticklabel style={rotate=0},
  yticklabel style={rotate=0},
ymin=0, ymax=0.67179,
  scaled y ticks=false,
  yticklabel={
    \pgfmathparse{100*\tick}\pgfmathprintnumber[fixed,precision=0]{\pgfmathresult}},
  legend cell align={left},
]

\draw[draw=none,fill=scBlue] (axis cs:-0.3,0) rectangle (axis cs:0.3,0.6398);
\addlegendimage{ybar,ybar legend,draw=none,fill=scBlue}
\addlegendentry{Simulation}

\draw[draw=none,fill=scBlue] (axis cs:0.7,0) rectangle (axis cs:1.3,0.067);
\draw[draw=none,fill=scBlue] (axis cs:1.7,0) rectangle (axis cs:2.3,0.0906);
\draw[draw=none,fill=scBlue] (axis cs:2.7,0) rectangle (axis cs:3.3,0.0914);
\draw[draw=none,fill=scBlue] (axis cs:3.7,0) rectangle (axis cs:4.3,0.0696);
\draw[draw=none,fill=scBlue] (axis cs:4.7,0) rectangle (axis cs:5.3,0.0314);
\draw[draw=none,fill=scBlue] (axis cs:5.7,0) rectangle (axis cs:6.3,0.0084);
\draw[draw=none,fill=scBlue] (axis cs:6.7,0) rectangle (axis cs:7.3,0.0016);
\draw[draw=none,fill=scBlue] (axis cs:7.7,0) rectangle (axis cs:8.3,0.0002);

\addplot[very thick, black, mark=+, mark size=3]
table[row sep=\\]{
0 0.534650371237025\\
1 0.0706293793255496\\
2 0.0991915408245849\\
3 0.118997610239912\\
4 0.102309618289508\\
5 0.0546286257842331\\
6 0.0164454219615871\\
7 0.00289582436475047\\
8 0.000243762433954433\\
};
\addlegendentry{Theoretical}

\end{axis}
\end{tikzpicture}
    \caption{\unik}
    \label{fig:distribution_length_unik}
\end{subfigure}
\begin{subfigure}[t]{0.48\linewidth}
    \centering
    \begin{tikzpicture}
\begin{axis}[
  height=\figureheight,
  width=\figurewidth,
  bar width=6pt,
  enlarge x limits=0.15,
  tick align=outside,
  tick pos=left,
  xtick style={color=black},
  ytick style={color=black},
  ymajorgrids,
  grid style={scGrey},
  xlabel={Number of splits $S_1$},
  ylabel={Probability (\%)},
xticklabel style={rotate=0},
  yticklabel style={rotate=0},
ymin=0, ymax=0.511728,
  scaled y ticks=false,
  yticklabel={
    \pgfmathparse{100*\tick}\pgfmathprintnumber[fixed,precision=0]{\pgfmathresult}},
  legend cell align={left},
]

\addlegendimage{ybar,ybar legend,draw=none,fill=scBlue}
\addlegendentry{Simulation}
\draw[draw=none,fill=scBlue] (axis cs:-0.3,0) rectangle (axis cs:0.3,0.00478);
\draw[draw=none,fill=scBlue] (axis cs:0.7,0) rectangle (axis cs:1.3,0.0042);
\draw[draw=none,fill=scBlue] (axis cs:1.7,0) rectangle (axis cs:2.3,0.01904);
\draw[draw=none,fill=scBlue] (axis cs:2.7,0) rectangle (axis cs:3.3,0.06432);
\draw[draw=none,fill=scBlue] (axis cs:3.7,0) rectangle (axis cs:4.3,0.16654);
\draw[draw=none,fill=scBlue] (axis cs:4.7,0) rectangle (axis cs:5.3,0.27416);
\draw[draw=none,fill=scBlue] (axis cs:5.7,0) rectangle (axis cs:6.3,0.27356);
\draw[draw=none,fill=scBlue] (axis cs:6.7,0) rectangle (axis cs:7.3,0.14596);
\draw[draw=none,fill=scBlue] (axis cs:7.7,0) rectangle (axis cs:8.3,0.04304);
\draw[draw=none,fill=scBlue] (axis cs:8.7,0) rectangle (axis cs:9.3,0.0044);

\addplot[very thick, black, mark=+, mark size=3]
table[row sep=\\]{
0 0\\
1 0.001953125\\
2 0.017578125\\
3 0.0703125\\
4 0.1640625\\
5 0.24609375\\
6 0.24609375\\
7 0.1640625\\
8 0.0703125\\
9 0.017578125\\
};
\addlegendentry{Theoretical}

\end{axis}
\end{tikzpicture}
    \caption{\uni}
    \label{fig:distribution_length_uni}
\end{subfigure}
\end{figure}
\FloatBarrier

\newpage
\subsection{Different Tokenisers and Their Effect on Stochastic Tokenisation Schemes}

GPT2's BPE tokeniser \cite{sennrich2016neural} has a vocabulary of 50,257 tokens, whereas Llama3's tokeniser has a vocabulary size of 128,k. The Llama3 tokeniser is BPE based with the difference that it ignores merge rules if a string is part of the vocabulary. 

\begin{figure*}[h]
\appheader{GPT2}
\centering

\begin{subfigure}[t]{0.48\linewidth}
\centering
\begin{tikzpicture}[baseline=(current bounding box.north)]
\begin{axis}[
  ybar,
  bar width=10pt,
  width=\linewidth,
  height=5cm,
  ymin=0,
  ylabel={Probability (\%)},
  scaled y ticks=false,
  yticklabel={
  \pgfmathparse{100*\tick}%
  \pgfmathprintnumber[fixed,precision=0]{\pgfmathresult}%
  },
  xlabel={},
  symbolic x coords={rev-olution},
  xtick=data,
  xticklabel style={font=\ttfamily\small, rotate=90, anchor=east},
  enlarge x limits=0.18,
  legend style={at={(0.5,1.02)}, anchor=south, legend columns=2},
]
\addplot+[bar shift=-6pt,fill=scRed, draw=scRed] table[
  col sep=comma,
  x=Segmentation,
  y=Frequency
] {
Segmentation,Frequency
rev-olution,1.0
};
\addplot+[bar shift=+6pt,fill=black!60, draw=black!60] table[
  col sep=comma,
  x=Segmentation,
  y=Frequency
] {
Segmentation,Frequency
rev-olution,1.0
};
\legend{\uni~~, \stok}
\end{axis}
\end{tikzpicture}
\label{fig:seg_hist_coarse_b}
\end{subfigure}
\hfill 
\begin{subfigure}[t]{0.48\linewidth}
\centering
\begin{tikzpicture}[baseline=(current bounding box.north)]
\begin{axis}[
  ybar,
  bar width=10pt,
  width=\linewidth,
  height=5cm,
  ymin=0,
  ylabel={Probability (\%)},
  scaled y ticks=false,
  yticklabel={
  \pgfmathparse{100*\tick}%
  \pgfmathprintnumber[fixed,precision=0]{\pgfmathresult}%
  },
  xlabel={},
  symbolic x coords={re-vol-ution,r-ev-olution,re-v-olution,rev-ol-ution},
  xtick=data,
  xticklabel style={font=\ttfamily\small, rotate=90, anchor=east},
  enlarge x limits=0.18,
  legend style={at={(0.5,1.02)}, anchor=south, legend columns=2},
]
\addplot+[bar shift=-6pt,fill=scRed, draw=scRed] table[
  col sep=comma,
  x=Segmentation,
  y=Frequency
] {
Segmentation,Frequency
re-vol-ution,0.2502744237102086
r-ev-olution,0.26344676180021953
re-v-olution,0.25850713501646544
rev-ol-ution,0.22777167947310648
};
\addplot+[bar shift=+6pt,fill=black!60, draw=black!60] table[
  col sep=comma,
  x=Segmentation,
  y=Frequency
] {
Segmentation,Frequency
re-vol-ution,0
rev-ol-ution,0.461441213653603
re-v-olution,0.27686472819216185
r-ev-olution,0.26169405815423513
};
\legend{\uni~~, \stok}
\end{axis}
\end{tikzpicture}
\label{fig:seg_hist_coarse}
\end{subfigure}
\par\bigskip

\begin{subfigure}[t]{0.48\linewidth}
\centering
\begin{tikzpicture}[baseline=(current bounding box.north)]
\begin{axis}[
  ybar,
  bar width=4pt,
  width=\linewidth,
  height=5cm,
  ymin=0,
  ylabel={Probability (\%)},
  scaled y ticks=false,
  yticklabel={
  \pgfmathparse{100*\tick}%
  \pgfmathprintnumber[fixed,precision=0]{\pgfmathresult}%
  },
  xlabel={},
  symbolic x coords={
    rev-o-l-ution,
    re-vo-lu-tion,
    rev-o-lu-tion,
    r-ev-ol-ution,
    r-e-v-olution,
    re-vo-l-ution,
    rev-ol-uti-on,
    re-vol-ut-ion,
    re-vol-u-tion,
    r-e-vol-ution,
    re-v-ol-ution,
    re-vol-uti-on,
    rev-ol-u-tion,
    rev-ol-ut-ion,
  },
  xtick=data,
  xticklabel style={font=\ttfamily\small, rotate=90, anchor=east},
  enlarge x limits=0.05,
  legend style={at={(0.5,1.02)}, anchor=south, legend columns=2},
]
\addplot+[bar shift=-2pt,fill=scRed, draw=scRed] table[
  col sep=comma,
  x=Segmentation,
  y=Frequency
] {
Segmentation,Frequency
rev-o-l-ution,0.054838709677419356
re-vo-lu-tion,0.07096774193548387
rev-o-lu-tion,0.06129032258064516
r-ev-ol-ution,0.09032258064516129
r-e-v-olution,0.07096774193548387
re-vo-l-ution,0.08064516129032258
rev-ol-uti-on,0.0935483870967742
re-vol-ut-ion,0.08387096774193549
re-vol-u-tion,0.08387096774193549
r-e-vol-ution,0.08709677419354839
re-v-ol-ution,0.05161290322580645
re-vol-uti-on,0.04838709677419355
rev-ol-u-tion,0.04838709677419355
rev-ol-ut-ion,0.07419354838709677
};
\addplot+[bar shift=+2pt,fill=black!60, draw=black!60] table[
  col sep=comma,
  x=Segmentation,
  y=Frequency,
] {
Segmentation,Frequency
r-ev-ol-ution,0.17307692307692307
rev-o-l-ution,0.21875
re-v-ol-ution,0.18509615384615385
rev-ol-ut-ion,0.052884615384615384
rev-ol-uti-on,0.04807692307692308
rev-ol-u-tion,0.05048076923076923
r-e-v-olution,0.27163461538461536
};
\legend{\uni~~, \stok}
\end{axis}
\end{tikzpicture}
\label{fig:seg_hist_fine_b}
\end{subfigure}
\hfill 
\begin{subfigure}[t]{0.48\linewidth}
\centering
\begin{tikzpicture}[baseline=(current bounding box.north)]
\begin{axis}[
  ybar,
  bar width=1.5pt,
  width=\linewidth,
  height=5cm,
  ymin=0,
  ylabel={Probability (\%)},
  scaled y ticks=false,
  yticklabel={
  \pgfmathparse{100*\tick}%
  \pgfmathprintnumber[fixed,precision=0]{\pgfmathresult}%
  },
  xlabel={},
  symbolic x coords={
        rev-ol-ut-io-n,
        rev-o-l-u-tion,
        rev-o-l-uti-on,
        r-ev-ol-uti-on,
        re-v-o-l-ution,
        r-e-vol-u-tion,
        rev-o-l-ut-ion,
        rev-ol-u-t-ion,
        rev-o-lu-ti-on,
        re-v-ol-uti-on,
        r-e-v-ol-ution,
        rev-ol-u-ti-on,
        r-ev-ol-ut-ion,
        re-vo-lu-t-ion,
        re-vo-l-uti-on,
        re-vol-uti-o-n,
        re-vol-u-ti-on,
        r-e-vol-ut-ion,
        re-v-o-lu-tion,
        re-v-ol-u-tion,
        r-ev-ol-u-tion,
        r-e-vo-l-ution,
        r-e-vol-uti-on,
        re-vol-u-t-ion,
        re-vol-ut-io-n,
        r-ev-o-l-ution,
        re-vo-l-ut-ion,
        re-vo-l-u-tion,
        r-e-vo-lu-tion,
        rev-ol-ut-i-on,
        r-ev-o-lu-tion,
        re-vol-ut-i-on,
        re-vo-lu-ti-on,
        rev-o-lu-t-ion,
        re-v-ol-ut-ion,
        rev-ol-uti-o-n,
  },
  xtick=data,
  xticklabel style={font=\ttfamily\small, rotate=90, anchor=east},
  enlarge x limits=0.01,
  legend style={at={(0.5,1.02)}, anchor=south, legend columns=2},
]
\addplot+[bar shift=-1pt,
  fill=scRed, draw=scRed] table[
  col sep=comma,
  x=Segmentation,
  y=Frequency
] {
Segmentation,Frequency
rev-ol-ut-io-n,0.03814064362336114
rev-o-l-u-tion,0.026221692491060787
rev-o-l-uti-on,0.026221692491060787
r-ev-ol-uti-on,0.027413587604290822
re-v-o-l-ution,0.025029797377830752
r-e-vol-u-tion,0.03456495828367104
rev-o-l-ut-ion,0.028605482717520857
rev-ol-u-t-ion,0.029797377830750895
rev-o-lu-ti-on,0.03098927294398093
re-v-ol-uti-on,0.03098927294398093
r-e-v-ol-ution,0.03814064362336114
rev-ol-u-ti-on,0.028605482717520857
r-ev-ol-ut-ion,0.026221692491060787
re-vo-lu-t-ion,0.033373063170441
re-vo-l-uti-on,0.03575685339690107
re-vol-uti-o-n,0.02264600715137068
re-vol-u-ti-on,0.03218116805721097
r-e-vol-ut-ion,0.04290822407628129
re-v-o-lu-tion,0.01907032181168057
re-v-ol-u-tion,0.03694874851013111
r-ev-ol-u-tion,0.025029797377830752
r-e-vo-l-ution,0.026221692491060787
r-e-vol-uti-on,0.01907032181168057
re-vol-u-t-ion,0.03098927294398093
re-vol-ut-io-n,0.01907032181168057
r-ev-o-l-ution,0.023837902264600714
re-vo-l-ut-ion,0.02264600715137068
re-vo-l-u-tion,0.0166865315852205
r-e-vo-lu-tion,0.03456495828367104
rev-ol-ut-i-on,0.027413587604290822
r-ev-o-lu-tion,0.025029797377830752
re-vol-ut-i-on,0.023837902264600714
re-vo-lu-ti-on,0.025029797377830752
rev-o-lu-t-ion,0.023837902264600714
re-v-ol-ut-ion,0.025029797377830752
rev-ol-uti-o-n,0.017878426698450536
};
\addplot+[bar shift=+1pt,fill=black!60, draw=black!60] table[
  col sep=comma,
  x=Segmentation,
  y=Frequency,
] {
Segmentation,Frequency
r-ev-ol-ut-ion,0.0375
re-v-ol-u-tion,0.06875
re-v-o-l-ution,0.1125
rev-o-l-uti-on,0.05625
r-e-v-ol-ution,0.225
rev-o-l-ut-ion,0.0375
r-ev-o-l-ution,0.125
rev-ol-ut-i-on,0.0125
rev-o-l-u-tion,0.0625
r-ev-ol-uti-on,0.0625
rev-ol-u-t-ion,0.05
r-ev-ol-u-tion,0.0375
rev-ol-uti-o-n,0.01875
rev-ol-u-ti-on,0.01875
rev-ol-ut-io-n,0.0125
re-v-ol-ut-ion,0.03125
re-v-ol-uti-on,0.03125
};
\legend{\uni~~, \stok}
\end{axis}
\end{tikzpicture}
\label{fig:seg_hist_fine}
\end{subfigure}

\caption{Histogram comparison of segmentation probabilities for two sampling schemes (\uni vs.\ \stok).}
\label{app:fig:seg_histograms}

\end{figure*}

\newpage

\begin{figure*}[h]
\appheader{Llama 3}
\centering

\begin{subfigure}[t]{0.48\linewidth}
\centering
\begin{tikzpicture}[baseline=(current bounding box.north)]
\begin{axis}[
  ybar,
  bar width=10pt,
  width=\linewidth,
  height=5cm,
  ymin=0,
  ylabel={Probability (\%)},
  scaled y ticks=false,
  yticklabel={
  \pgfmathparse{100*\tick}%
  \pgfmathprintnumber[fixed,precision=0]{\pgfmathresult}%
  },
  xlabel={},
  symbolic x coords={
    rev-olution,
    re-volution,
  },
  xtick=data,
  xticklabel style={font=\ttfamily\small, rotate=90, anchor=east},
  enlarge x limits=0.3,
  legend style={at={(0.5,1.02)}, anchor=south, legend columns=2},
]
\addplot+[bar shift=-6pt,fill=scRed, draw=scRed] table[
  col sep=comma,
  x=Segmentation,
  y=Frequency] {
Segmentation,Frequency
rev-olution,0.5012787723785166
re-volution,0.49872122762148335

};
\addplot+[bar shift=+6pt,fill=black!60, draw=black!60] table[
  col sep=comma,
  x=Segmentation,
  y=Frequency] {
Segmentation,Frequency
re-volution,0.4899436846339501
rev-olution,0.5100563153660499
};
\legend{\uni~~, \stok}
\end{axis}
\end{tikzpicture}
\label{fig:seg_hist_coarse_appendix_b}
\end{subfigure}
\hfill 
\begin{subfigure}[t]{0.48\linewidth}
\centering
\begin{tikzpicture}[baseline=(current bounding box.north)]
\begin{axis}[
  ybar,
  bar width=8pt,
  width=\linewidth,
  height=5cm,
  ymin=0,
  ylabel={Probability (\%)},
  scaled y ticks=false,
  yticklabel={
  \pgfmathparse{100*\tick}%
  \pgfmathprintnumber[fixed,precision=0]{\pgfmathresult}%
  },
  xlabel={},
  symbolic x coords={
    re-vol-ution,
    re-v-olution,
    r-ev-olution,
    rev-ol-ution,
    rev-olut-ion,
    r-e-volution,
    rev-olu-tion,
  },
  xtick=data,
  xticklabel style={font=\ttfamily\small, rotate=90, anchor=east},
  enlarge x limits=0.18,
  legend style={at={(0.5,1.02)}, anchor=south, legend columns=2},
]
\addplot+[bar shift=-5pt,fill=scRed, draw=scRed] table[
  col sep=comma,
  x=Segmentation,
  y=Frequency] {
Segmentation,Frequency
re-vol-ution,0.13700918964076858
re-v-olution,0.14118629908103592
r-ev-olution,0.13450292397660818
rev-ol-ution,0.14118629908103592
rev-olut-ion,0.14619883040935672
r-e-volution,0.14703425229741018
rev-olu-tion,0.15288220551378445
};
\addplot+[bar shift=+5pt,fill=black!60, draw=black!60] table[
  col sep=comma,
  x=Segmentation,
  y=Frequency] {
Segmentation,Frequency
re-v-olution,0.24655963302752293
r-ev-olution,0.14220183486238533
r-e-volution,0.3130733944954128
rev-olu-tion,0.07224770642201835
rev-olut-ion,0.05045871559633028
re-vol-ution,0.10665137614678899
rev-ol-ution,0.06880733944954129
};
\legend{\uni~~, \stok}
\end{axis}
\end{tikzpicture}
\label{fig:seg_hist_coarse_appendix}
\end{subfigure}
\par\bigskip
\begin{subfigure}[t]{0.48\linewidth}
\centering
\begin{tikzpicture}[baseline=(current bounding box.north)]
\begin{axis}[
  ybar,
  bar width=2pt,
  width=\linewidth,
  height=5cm,
  ymin=0,
  ylabel={Probability (\%)},
  scaled y ticks=false,
  yticklabel={
  \pgfmathparse{100*\tick}%
  \pgfmathprintnumber[fixed,precision=0]{\pgfmathresult}%
  },
  xlabel={},
  symbolic x coords={
    re-vo-l-ution,
    r-e-vol-ution,
    r-ev-olut-ion,
    re-v-olut-ion,
    re-vol-uti-on,
    rev-ol-u-tion,
    rev-olu-ti-on,
    rev-ol-ut-ion,
    rev-olut-io-n,
    re-vo-lu-tion,
    r-e-v-olution,
    rev-o-lu-tion,
    rev-olut-i-on,
    re-v-ol-ution,
    rev-ol-uti-on,
    re-vol-ut-ion,
    rev-olu-t-ion,
    rev-o-lut-ion,
    re-v-olu-tion,
    r-ev-olu-tion,
    re-vo-lut-ion,
    re-vol-u-tion,
    r-ev-ol-ution,
    rev-o-l-ution,
  },
  xtick=data,
  xticklabel style={font=\ttfamily\small, rotate=90, anchor=east},
  enlarge x limits=0.05,
  legend style={at={(0.5,1.02)}, anchor=south, legend columns=2},
]
\addplot+[bar shift=-1.5pt,fill=scRed, draw=scRed] table[
  col sep=comma,
  x=Segmentation,
  y=Frequency
] {
Segmentation,Frequency
re-vo-l-ution,0.03501855287569573
r-e-vol-ution,0.04962894248608534
r-ev-olut-ion,0.04638218923933209
re-v-olut-ion,0.040584415584415584
re-vol-uti-on,0.03965677179962894
rev-ol-u-tion,0.040352504638218926
rev-olu-ti-on,0.04359925788497217
rev-ol-ut-ion,0.03780148423005566
rev-olut-io-n,0.041280148423005564
re-vo-lu-tion,0.046150278293135436
r-e-v-olution,0.041743970315398886
rev-o-lu-tion,0.041743970315398886
rev-olut-i-on,0.0387291280148423
re-v-ol-ution,0.0364100185528757
rev-ol-uti-on,0.04406307977736549
re-vol-ut-ion,0.03455473098330241
rev-olu-t-ion,0.04638218923933209
rev-o-lut-ion,0.04151205936920223
re-v-olu-tion,0.04406307977736549
r-ev-olu-tion,0.03849721706864564
re-vo-lut-ion,0.04568645640074211
re-vol-u-tion,0.042439703153988866
r-ev-ol-ution,0.03896103896103896
rev-o-l-ution,0.044758812615955476
};
\addplot+[bar shift=+1.5pt,fill=black!60, draw=black!60] table[
  col sep=comma,
  x=Segmentation,
  y=Frequency,
] {
Segmentation,Frequency
r-e-vol-ution,0.15240641711229946
r-ev-olu-tion,0.0374331550802139
rev-ol-u-tion,0.045454545454545456
r-e-v-olution,0.3048128342245989
re-v-ol-ution,0.07754010695187166
re-v-olu-tion,0.05614973262032086
rev-ol-ut-ion,0.0213903743315508
r-ev-olut-ion,0.026737967914438502
rev-ol-uti-on,0.008021390374331552
rev-o-l-ution,0.034759358288770054
rev-o-lu-tion,0.029411764705882353
re-vo-l-ution,0.01871657754010695
rev-o-lut-ion,0.008021390374331552
rev-olu-t-ion,0.0213903743315508
rev-olut-io-n,0.016042780748663103
r-ev-ol-ution,0.034759358288770054
rev-olut-i-on,0.0053475935828877
rev-olu-ti-on,0.0106951871657754
re-v-olut-ion,0.053475935828877004
re-vol-u-tion,0.016042780748663103
re-vol-uti-on,0.0106951871657754
re-vol-ut-ion,0.0106951871657754
};
\legend{\uni~~, \stok}
\end{axis}
\end{tikzpicture}
\label{fig:seg_hist_fine_appendix_b}
\end{subfigure}
\hfill 
\begin{subfigure}[t]{0.48\linewidth}
\centering
\begin{tikzpicture}[baseline=(current bounding box.north)]
\begin{axis}[
  ybar,
  bar width=1pt,
  width=\linewidth,
  height=5cm,
  ymin=0,
  ylabel={Probability (\%)},
  scaled y ticks=false,
  yticklabel={
  \pgfmathparse{100*\tick}%
  \pgfmathprintnumber[fixed,precision=0]{\pgfmathresult}%
  },
  xlabel={},
  symbolic x coords={
    rev-o-l-ut-ion,
    r-ev-olu-t-ion,
    r-e-vol-uti-on,
    rev-ol-u-t-ion,
    re-vol-ut-i-on,
    rev-olu-ti-o-n,
    re-vo-l-uti-on,
    r-e-v-olut-ion,
    r-ev-ol-uti-on,
    r-e-v-olu-tion,
    re-vol-ut-io-n,
    re-v-o-l-ution,
    re-v-ol-u-tion,
    r-e-vol-u-tion,
    re-vo-lut-io-n,
    rev-olu-t-io-n,
    re-vo-l-u-tion,
    re-v-olu-ti-on,
    re-vol-u-ti-on,
    rev-ol-u-ti-on,
    re-vo-lut-i-on,
    r-e-vo-lu-tion,
    r-ev-olut-io-n,
    r-e-vo-l-ution,
    rev-o-lu-t-ion,
    re-v-ol-uti-on,
    r-ev-o-lu-tion,
    re-v-olu-t-ion,
    re-vo-lu-ti-on,
    rev-o-lut-i-on,
    re-v-olut-i-on,
    r-ev-o-l-ution,
    rev-o-lu-ti-on,
    rev-o-l-u-tion,
    rev-ol-ut-i-on,
    rev-olut-i-o-n,
    re-vol-u-t-ion,
    rev-o-lut-io-n,
    r-ev-ol-u-tion,
    rev-olu-t-i-on,
    re-vo-l-ut-ion,
    r-ev-olu-ti-on,
    re-vol-uti-o-n,
    r-e-v-ol-ution,
    rev-ol-uti-o-n,
    re-v-o-lu-tion,
    rev-ol-ut-io-n,
    r-ev-olut-i-on,
    r-e-vo-lut-ion,
    r-ev-o-lut-ion,
    r-ev-ol-ut-ion,
    r-e-vol-ut-ion,
    rev-o-l-uti-on,
    re-v-olut-io-n,
    re-vo-lu-t-ion,
    re-v-o-lut-ion,
    re-v-ol-ut-ion,
  },
  xtick=data,
  xticklabel style={font=\ttfamily\tiny, rotate=90, anchor=east},
  enlarge x limits=0.01,
  legend style={at={(0.5,1.02)}, anchor=south, legend columns=2},
  yticklabel style={
  /pgf/number format/.cd,
    fixed,
    precision=1
},
scaled y ticks=false,
]
\addplot+[bar shift=-1pt,fill=scRed, draw=scRed] table[
  col sep=comma,
  x=Segmentation,
  y=Frequency
] {
Segmentation,Frequency
rev-o-l-ut-ion,0.017851829812555786
r-ev-olu-t-ion,0.018149360309431716
r-e-vol-uti-on,0.01894277496776753
rev-ol-u-t-ion,0.017752652980263812
re-vol-ut-i-on,0.015669939502132302
rev-olu-ti-o-n,0.01953783596151939
re-vo-l-uti-on,0.017951006644847764
r-e-v-olut-ion,0.016860061489636022
r-ev-ol-uti-on,0.017752652980263812
r-e-v-olu-tion,0.015471585837548348
re-vol-ut-io-n,0.016562530992760092
re-v-o-l-ution,0.019438659129227414
re-v-ol-u-tion,0.017256768818803926
r-e-vol-u-tion,0.014777348011504512
re-vo-lut-io-n,0.01924030546464346
rev-olu-t-io-n,0.017256768818803926
re-vo-l-u-tion,0.019141128632351484
re-v-olu-ti-on,0.017157591986511952
re-vol-u-ti-on,0.015669939502132302
rev-ol-u-ti-on,0.018248537141723694
re-vo-lut-i-on,0.017256768818803926
r-e-vo-lu-tion,0.01428146385004463
r-ev-olut-io-n,0.01636417732817614
r-e-vo-l-ution,0.017455122483387882
rev-o-lu-t-ion,0.02013289695527125
re-v-ol-uti-on,0.018843598135475554
r-ev-o-lu-tion,0.01517405534067242
re-v-olu-t-ion,0.016165823663592184
re-vo-lu-ti-on,0.019041951800059506
rev-o-lut-i-on,0.016860061489636022
re-v-olut-i-on,0.018050183477139742
r-ev-o-l-ution,0.01606664683130021
rev-o-lu-ti-on,0.014678171179212536
rev-o-l-u-tion,0.01636417732817614
rev-ol-ut-i-on,0.01924030546464346
rev-olut-i-o-n,0.0186452444708916
re-vol-u-t-ion,0.018248537141723694
rev-o-lut-io-n,0.017554299315679856
r-ev-ol-u-tion,0.016959238321927996
rev-olu-t-i-on,0.016860061489636022
re-vo-l-ut-ion,0.017653476147971834
r-ev-olu-ti-on,0.015273232172964396
re-vol-uti-o-n,0.01606664683130021
r-e-v-ol-ution,0.017355945651095904
rev-ol-uti-o-n,0.019736189626103344
re-v-o-lu-tion,0.018843598135475554
rev-ol-ut-io-n,0.017752652980263812
r-ev-olut-i-on,0.017752652980263812
r-e-vo-lut-ion,0.016661707825052066
r-ev-o-lut-ion,0.02211643360111078
r-ev-ol-ut-ion,0.016165823663592184
r-e-vol-ut-ion,0.01894277496776753
rev-o-l-uti-on,0.016860061489636022
re-v-olut-io-n,0.018149360309431716
re-vo-lu-t-ion,0.017752652980263812
re-v-o-lut-ion,0.019934543290687296
re-v-ol-ut-ion,0.018050183477139742

};
\addplot+[bar shift=+1pt,fill=black!60, draw=black!60] table[
  col sep=comma,
  x=Segmentation,
  y=Frequency
] {
Segmentation,Frequency
re-v-olu-t-ion,0.051470588235294115
r-e-vo-l-ution,0.022058823529411766
rev-ol-u-t-ion,0.03676470588235294
rev-o-lut-i-on,0.007352941176470588
rev-olu-t-i-on,0.03676470588235294
r-e-v-ol-ution,0.125
r-e-v-olu-tion,0.0661764705882353
rev-o-l-u-tion,0.04411764705882353
r-e-vol-ut-ion,0.03676470588235294
r-e-vol-uti-on,0.014705882352941176
re-v-olu-ti-on,0.022058823529411766
re-v-ol-u-tion,0.04411764705882353
r-e-v-olut-ion,0.07352941176470588
re-v-olut-i-on,0.007352941176470588
re-v-o-l-ution,0.07352941176470588
re-v-o-lu-tion,0.007352941176470588
re-v-ol-ut-ion,0.029411764705882353
r-ev-olu-ti-on,0.014705882352941176
r-e-vol-u-tion,0.03676470588235294
rev-ol-uti-o-n,0.007352941176470588
rev-olu-t-io-n,0.007352941176470588
re-v-ol-uti-on,0.007352941176470588
rev-o-lu-t-ion,0.022058823529411766
rev-ol-ut-i-on,0.022058823529411766
re-v-olut-io-n,0.014705882352941176
r-ev-o-l-ution,0.014705882352941176
rev-olut-i-o-n,0.014705882352941176
rev-o-l-ut-ion,0.014705882352941176
re-vol-u-t-ion,0.007352941176470588
r-ev-olut-io-n,0.014705882352941176
re-v-o-lut-ion,0.007352941176470588
re-vo-l-u-tion,0.007352941176470588
r-ev-olu-t-ion,0.014705882352941176
r-ev-olut-i-on,0.007352941176470588
r-ev-ol-ut-ion,0.007352941176470588
r-ev-o-lut-ion,0.007352941176470588
rev-ol-ut-io-n,0.007352941176470588
r-ev-ol-u-tion,0.007352941176470588
re-vol-u-ti-on,0.007352941176470588
r-ev-ol-uti-on,0.007352941176470588
rev-ol-u-ti-on,0.007352941176470588
rev-o-lut-io-n,0.007352941176470588
rev-o-lu-ti-on,0.007352941176470588

};
\legend{\uni~~, \stok}
\end{axis}
\end{tikzpicture}
\label{fig:seg_hist_fine_appendix}
\end{subfigure}

\caption{Histogram comparison of segmentation probabilities for two sampling schemes (\uni vs.\ \stok).}
\label{app:fig:seg_histograms_appendix}

\end{figure*}

\FloatBarrier

\newpage

\newpage
\section{Additional Experiments and Ablations} \label{app:experiments}

\subsection{Additional Results for Stochastic Pre-Training} \label{app:pretraining}

\subsubsection{Language Modelling and General Benchmark Performance}

We report perplexity $\mathrm{PPL}(x_{1:T})=\exp\!\left(-\frac{1}{T}\sum_{t=1}^T \log p(x_t \mid x_{<t})\right)$ on a validation subset of OpenWebText for our pre-trained \tinyllm models in \Cref{tab:tinyllm_perplexity}. Stochastic pre-training leads to higher perplexity on \emph{canonical} validation data. This is consistent with the training objective: during stochastic pre-training, the model is trained to also predict \emph{non-canonical} token sequences. At the same time, perplexity is substantially reduced for non-canonically tokenised validation data.
\begin{table}[h]
\centering
\caption{Perplexity on an OpenWebText validation subset for canonically and non-canonically tokenised inputs.}
\label{tab:tinyllm_perplexity}
\begin{tabular}{lccc}
\toprule
$\alpha_{\text{pre}}$ & canonical & $\alpha_{\text{eval}}=0.1$ & $\alpha_{\text{eval}}=0.5$ \\
\midrule
0.0 & 29.77 & 98.08 & 474.72 \\
0.5 & 53.06 & 43.89 & 30.28 \\
\bottomrule
\end{tabular}
\end{table}

In addition, we report the multiple-choice accuracy of the \tinyllm models on standard benchmark datasets in \Cref{tab:tinyllm_general_benchmarks}. Increasing $\alpha_{\text{pre}}$ slightly worsens the fit to the canonical distribution, but improves robustness to non-canonical tokenisations. Importantly, during stochastic pre-training, the model is trained to predict non-canonical token sequences. This can reduce accuracy in standard MCQ evaluations, where we compare the log-likelihood of canonically tokenised answer continuations. In subsequent fine-tuning stages, however, we can perturb the input while keeping the prediction targets canonical. This encourages robustness to perturbed input tokenisations without requiring the model to generate non-canonical outputs.

\begin{table}[h]
\centering
\caption{MCQ accuracy of pre-trained \tinyllm models on standard benchmarks under canonical and non-canonical tokenisation at evaluation time.}
\label{tab:tinyllm_general_benchmarks}
\begin{tabular}{lcccccccc}
\toprule
& \multicolumn{2}{c}{\textsc{ARC-E}} & \multicolumn{2}{c}{\textsc{Blimp}} & \multicolumn{2}{c}{\textsc{COPA}} & \multicolumn{2}{c}{\textsc{HellaSwag}} \\
\cmidrule(lr){2-3} \cmidrule(lr){4-5} \cmidrule(lr){6-7} \cmidrule(lr){8-9}
$\alpha_{\text{pre}}$ & canon. & $\alpha_{\text{eval}}=0.5$ & canon. & $\alpha_{\text{eval}}=0.5$ & canon. & $\alpha_{\text{eval}}=0.5$ & canon. & $\alpha_{\text{eval}}=0.5$ \\
\midrule
0.0 & 0.3780 & 0.2720 & 0.8480 & 0.6900 & 0.6600 & 0.4900 & 0.3160 & 0.2740 \\
0.5 & 0.3300 & 0.3120 & 0.8300 & 0.8040 & 0.6100 & 0.5900 & 0.2960 & 0.3060 \\
\bottomrule
\end{tabular}
\end{table}

\subsubsection{Evaluation on TokSuite}
Lastly, we evaluate the \tinyllm models on \textsc{TokSuite} \cite{altintacs2025toksuite}, which probes a variety of tokenisation-sensitive tasks. Under \emph{canonical} evaluation, stochastic pre-training with $\alpha_{\text{pre}}=0.5$ leads to a small drop in overall accuracy compared to canonical pre-training, from $0.3746$ to $0.3623$. For individual subsets, the largest decreases are observed for \textsc{contractions} (from $0.5854$ to $0.4146$), \textsc{abbreviations} (from $0.3784$ to $0.2703$), \textsc{web search queries} (from $0.3514$ to $0.2432$), \textsc{similar words} (from $0.5405$ to $0.4595$), and \textsc{colloquial variants} (from $0.4250$ to $0.3500$). However, we also observe improvements on several subsets, including \textsc{spelled-out text} (from $0.2667$ to $0.4667$), \textsc{historical spelling} (from $0.3500$ to $0.5000$), \textsc{keyboard proximity errors} (from $0.2927$ to $0.4146$), and \textsc{date formats} (from $0.1667$ to $0.2778$).

Under \emph{stochastic} evaluation with \textsc{StochasTok} perturbations at $\alpha_{\text{eval}}=0.5$, however, stochastic pre-training improves overall accuracy from $0.2974$ to $0.3535$. The largest gains are observed for \textsc{character substitution} (from $0.2500$ to $0.4375$), \textsc{historical spelling} (from $0.3000$ to $0.4500$), \textsc{OCR errors} (from $0.2381$ to $0.3810$), \textsc{orthographic errors} (from $0.3250$ to $0.4375$), \textsc{spelled-out text} (from $0.4000$ to $0.5333$), and \textsc{spaced styling} (from $0.1000$ to $0.2250$). We still observe decreases on a few subsets, such as \textsc{colloquial variants} (from $0.4250$ to $0.3250$), \textsc{abbreviations} (from $0.3243$ to $0.2973$), and \textsc{scripted text} (from $0.2750$ to $0.2250$). 

Overall, this is consistent with our findings on the other benchmarks: stochastic pre-training slightly worsens the fit to the canonical distribution, but improves robustness to non-canonical tokenisations.

\newpage
\subsection{Additional Results for Stochastic Fine-Tuning}
\label{app:exp-pretraining-finetuning}

\subsubsection{Additional Results for \tinyllm} \label{app:add_results_tinyllm}
\begin{figure*}[ht]
    \centering

    {\textit{Data set}: \langgame \\ \textit{Model}: \tinyllm \par\medskip}

    \setlength{\figureheight}{0.2\linewidth}
    \setlength{\figurewidth}{.25\linewidth}

    \pgfplotsset{
        x tick label style={font=\scriptsize},
        y tick label style={rotate=90, font=\scriptsize},
        ylabel={\small Avg.\ test accuracy $\rightarrow$},
        xlabel={\small $\peval$},
        scale only axis,
        tick align=outside,
        tick pos=left,
        every axis plot/.append style={mark size=3pt},
        axis x line*=bottom,
        axis y line*=left,
        axis line style={draw=none},
        grid style={solid},
    }

    \begin{subfigure}[b]{0.33\textwidth}
        \centering
        \begin{tikzpicture}
\begin{axis}[
  height=\figureheight,
  width=\figurewidth,
  ybar,
  bar width=6pt,
  enlarge x limits=0.1,
  symbolic x coords={0,0.1,0.5,1,3},
  xtick=data,
  xtick style={color=black},
  ytick style={color=black},
  ymajorgrids,
ymin=0.1, ymax=1,
    yticklabel={
      \pgfmathparse{100*\tick}\pgfmathprintnumber[fixed,precision=0]{\pgfmathresult}\%
    },
  grid style={scGrey},
  legend cell align={left},
  legend style={fill opacity=1,legend columns=3,draw opacity=1,text opacity=1,at={(-0.05,1.03)},anchor=south west,draw=none,font=\scriptsize},
]

\addplot+[very thick, draw=black, fill=black, bar shift=-8pt]
coordinates {
  (0,0.411)
  ({0.1},0.3822)
  ({0.5},0.3084)
  (1,0.2803)
  (3,0.2622)
};
\addlegendentry{Canon.}

\addplot+[very thick, draw=scRed!90, fill=scRed!90, bar shift=0pt]
coordinates {
  (0,0.6660)
  ({0.1},0.6616)
  ({0.5},0.6246)
  (1,0.5630)
  (3,0.4388)
};
\addlegendentry{$\pfine=0.1$}

\addplot+[very thick, draw=scRed!60, fill=scRed!60, bar shift=8pt]
coordinates {
  (0,0.5920)
  ({0.1},0.5982)
  ({0.5},0.5854)
  (1,0.5768)
  (3,0.5078)
};
\addlegendentry{$\pfine=0.5$}

\end{axis}
\end{tikzpicture}
        \caption{$\ppre = 0$} \label{fig:app-langgame_tiny-p0}
    \end{subfigure}
    \hfill
    \begin{subfigure}[b]{0.33\textwidth}
        \centering
        \begin{tikzpicture}
\begin{axis}[
  height=\figureheight,
  width=\figurewidth,
  ybar,
  bar width=6pt,
  enlarge x limits=0.1,
  symbolic x coords={0,0.1,0.5,1,3},
  xtick=data,
  xtick style={color=black},
  ytick style={color=black},
  ymajorgrids,
ymin=0.1, ymax=1,
    yticklabel={
      \pgfmathparse{100*\tick}\pgfmathprintnumber[fixed,precision=0]{\pgfmathresult}\%
    },
  grid style={scGrey},
  legend cell align={left},
  legend style={fill opacity=1,legend columns=3,draw opacity=1,text opacity=1,at={(0.05,1.03)},anchor=south west,draw=none,font=\scriptsize},
]

\addplot+[very thick, draw=black, fill=black, bar shift=-8pt]
coordinates {
  (0,0.787)
  ({0.1},0.7629)
  ({0.5},0.6778)
  (1,0.593)
  (3,0.4352)
};
\addlegendentry{Canon.}

\addplot+[very thick, draw=scRed!90, fill=scRed!90, bar shift=0pt]
coordinates {
  (0,0.814)
  ({0.1},0.8172)
  ({0.5},0.8042)
  (1,0.792)
  (3,0.7452)
};
\addlegendentry{$\pfine=0.1$}

\addplot+[very thick, draw=scRed!60, fill=scRed!60, bar shift=8pt]
coordinates {
  (0,0.782)
  ({0.1},0.7852)
  ({0.5},0.7826)
  (1,0.7834)
  (3,0.78)
};
\addlegendentry{$\pfine=0.5$}

\end{axis}
\end{tikzpicture}
        \caption{$\ppre = 0.1$} \label{fig:app-langgame_tiny-p0.1}
    \end{subfigure}
    \hfill
    \begin{subfigure}[b]{0.33\textwidth}
        \centering
        \begin{tikzpicture}
\begin{axis}[
  height=\figureheight,
  width=\figurewidth,
  ybar,
  bar width=6pt,
  enlarge x limits=0.1,
  symbolic x coords={0,0.1,0.5,1,3},
  xtick=data,
  xtick style={color=black},
  ytick style={color=black},
  ymajorgrids,
ymin=0.1, ymax=1,
    yticklabel={
      \pgfmathparse{100*\tick}\pgfmathprintnumber[fixed,precision=0]{\pgfmathresult}\%
    },
  grid style={scGrey},
  legend cell align={left},
  legend style={fill opacity=1,legend columns=3,draw opacity=1,text opacity=1,at={(-0.05,1.03)},anchor=south west,draw=none,font=\scriptsize},
]

\addplot+[very thick, draw=black, fill=black, bar shift=-8pt]
coordinates {
  (0,0.750666666666667)
  ({0.1},0.769533333333333)
  ({0.5},0.729666666666667)
  (1,0.652666666666667)
  (3,0.493333333333333)
};
\addlegendentry{Canon.}

\addplot+[very thick, draw=scRed!90, fill=scRed!90, bar shift=0pt]
coordinates {
  (0,0.844)
  ({0.1},0.8494)
  ({0.5},0.8556)
  (1,0.8518)
  (3,0.8282)
};
\addlegendentry{$\pfine=0.1$}

\addplot+[very thick, fill=scRed!60, draw=scRed!60, bar shift=8pt]
coordinates {
  (0,0.804)
  ({0.1},0.8162)
  ({0.5},0.8396)
  (1,0.8356)
  (3,0.8238)
};
\addlegendentry{$\pfine=0.5$}

\end{axis}
\end{tikzpicture} \label{fig:app-langgame_tiny-p0.5}
        \caption{$\ppre = 0.5$}
    \end{subfigure}

    \caption{Average test accuracy of \tinyllm on \langgame under stochastic tokenisation during fine-tuning, for different levels of stochasticity during pre-training. Stochastic fine-tuning improves robustness to random perturbations during testing across pre-training settings.}
    \label{app:fig:avg_acc_tinyllm_langgame}
\end{figure*}

\begin{figure*}[ht]
    \centering

    {\textit{Data set}: \cute \\ \textit{Model}: \tinyllm \par\medskip}

    \setlength{\figureheight}{0.2\linewidth}
    \setlength{\figurewidth}{.25\linewidth}

    \pgfplotsset{
        x tick label style={font=\scriptsize},
        y tick label style={rotate=90, font=\scriptsize},
        ylabel={\small Avg.\ test accuracy $\rightarrow$},
        xlabel={\small $\peval$},
        scale only axis,
        tick align=outside,
        tick pos=left,
        every axis plot/.append style={mark size=3pt},
        axis x line*=bottom,
        axis y line*=left,
        axis line style={draw=none},
        grid style={solid},
    }

    \begin{subfigure}[b]{0.33\textwidth}
        \centering
        \begin{tikzpicture}
\begin{axis}[
  height=\figureheight,
  width=\figurewidth,
  ybar,
  bar width=6pt,
  enlarge x limits=0.1,
  symbolic x coords={0,0.1,0.5,1,3},
  xtick=data,
  xtick style={color=black},
  ytick style={color=black},
  ymajorgrids,
  ymin=0.1, ymax=1,
  yticklabel={
    \pgfmathparse{100*\tick}\pgfmathprintnumber[fixed,precision=0]{\pgfmathresult}\%
  },
  grid style={scGrey},
  legend cell align={left},
  legend style={fill opacity=1,legend columns=3,draw opacity=1,text opacity=1,at={(0.05,1.03)},anchor=south west,draw=none,font=\scriptsize},
]

\addplot+[very thick, draw=black, fill=black, bar shift=-8pt]
coordinates {
  (0,0.467)
  ({0.1},0.4226)
  ({0.5},0.3046)
  (1,0.2535)
  (3,0.2091)
};
\addlegendentry{Canon.}

\addplot+[very thick, draw=scRed!90, fill=scRed!90, bar shift=0pt]
coordinates {
  (0,0.592)
  ({0.1},0.5924)
  ({0.5},0.5816)
  (1,0.5616)
  (3,0.4786)
};
\addlegendentry{$\pfine=0.1$}

\addplot+[very thick, draw=scRed!60, fill=scRed!60, bar shift=8pt]
coordinates {
  (0,0.62)
  ({0.1},0.6234)
  ({0.5},0.627)
  (1,0.6274)
  (3,0.6182)
};
\addlegendentry{$\pfine=0.5$}

\end{axis}
\end{tikzpicture}
        \caption{$\ppre = 0$}
    \end{subfigure}
    \hfill
    \begin{subfigure}[b]{0.33\textwidth}
        \centering
        \begin{tikzpicture}
\begin{axis}[
  height=\figureheight,
  width=\figurewidth,
  ybar,
  bar width=6pt,
  enlarge x limits=0.1,
  symbolic x coords={0,0.1,0.5,1,3},
  xtick=data,
  xtick style={color=black},
  ytick style={color=black},
  ymajorgrids,
  ymin=0.1, ymax=1,
  yticklabel={
    \pgfmathparse{100*\tick}\pgfmathprintnumber[fixed,precision=0]{\pgfmathresult}\%
  },
  grid style={scGrey},
  legend cell align={left},
  legend style={fill opacity=1,legend columns=3,draw opacity=1,text opacity=1,at={(0.05,1.03)},anchor=south west,draw=none,font=\scriptsize},
]

\addplot+[very thick, draw=black, fill=black, bar shift=-8pt]
coordinates {
  (0,0.561)
  ({0.1},0.5472)
  ({0.5},0.4779)
  (1,0.407)
  (3,0.2865)
};
\addlegendentry{Canon.}

\addplot+[very thick, draw=scRed!90, fill=scRed!90, bar shift=0pt]
coordinates {
  (0,0.736)
  ({0.1},0.7372)
  ({0.5},0.7306)
  (1,0.7106)
  (3,0.6612)
};
\addlegendentry{$\pfine=0.1$}

\addplot+[very thick, draw=scRed!60, fill=scRed!60, bar shift=8pt]
coordinates {
  (0,0.79)
  ({0.1},0.7922)
  ({0.5},0.7926)
  (1,0.7934)
  (3,0.7878)
};
\addlegendentry{$\pfine=0.5$}

\end{axis}
\end{tikzpicture}
        \caption{$\ppre = 0.1$}
    \end{subfigure}
    \hfill
    \begin{subfigure}[b]{0.33\textwidth}
        \centering
        \begin{tikzpicture}
\begin{axis}[
  height=\figureheight,
  width=\figurewidth,
  ybar,
  bar width=6pt,
  enlarge x limits=0.1,
  symbolic x coords={0,0.1,0.5,1,3},
  xtick=data,
  xtick style={color=black},
  ytick style={color=black},
  ymajorgrids,
  ymin=0.1, ymax=1,
  yticklabel={
    \pgfmathparse{100*\tick}\pgfmathprintnumber[fixed,precision=0]{\pgfmathresult}\%
  },
  grid style={scGrey},
  legend cell align={left},
  legend style={fill opacity=1,legend columns=3,draw opacity=1,text opacity=1,at={(0.05,1.03)},anchor=south west,draw=none,font=\scriptsize},
]

\addplot+[very thick, draw=black, fill=black, bar shift=-8pt]
coordinates {
  (0,0.788)
  ({0.1},0.718333333333333)
  ({0.5},0.594266666666667)
  (1,0.500666666666667)
  (3,0.3468)
};
\addlegendentry{Canon.}

\addplot+[very thick, draw=scRed!90, fill=scRed!90, bar shift=0pt]
coordinates {
  (0,0.828)
  ({0.1},0.8256)
  ({0.5},0.8186)
  (1,0.8046)
  (3,0.7802)
};
\addlegendentry{$\pfine=0.1$}

\addplot+[very thick, draw=scRed!60, fill=scRed!60, bar shift=8pt]
coordinates {
  (0,0.862)
  ({0.1},0.861)
  ({0.5},0.8506)
  (1,0.8492)
  (3,0.8418)
};
\addlegendentry{$\pfine=0.5$}

\end{axis}
\end{tikzpicture}
        \caption{$\ppre = 0.5$}
    \end{subfigure}

    \caption{Average test accuracy of \tinyllm on \cute under stochastic tokenisation during fine-tuning, for different levels of stochasticity during pre-training. Stochastic fine-tuning improves robustness to random perturbations during testing across pre-training settings.}
    \label{app:fig:avg_acc_tinyllm_cute}
\end{figure*}

\begin{figure}[h!]
    \centering
    { \textit{Data set}: \langgame \\ \textit{Model}: \tinyllm}\par\medskip

    \begin{subfigure}{.85\textwidth}
      \appheader{Pre-training perturbation: $\alpha_{\mathrm{pre}}=0$}
      \resizebox{1\linewidth}{!}{%
      \pgfplotsset{
  colormap={scScheme}{
    color(0pt)=(white)
color(1pt)=(scCyan!60)
  }
}

\begin{tikzpicture}
\begin{groupplot}[
  group style={
    group size=3 by 1,
    horizontal sep=.5cm,
  },
  width=0.33\textwidth,
  tick label style={font=\small},   colormap name=scScheme,
  point meta min=0,
  point meta max=1,
  x grid style={gray!30},
  y grid style={gray!30},
  xlabel={$\pfine$},
  xmin=-0.5, xmax=3.5,
  ymin=-0.5, ymax=4.5,
  xtick={0,...,3},
  xticklabels={-, 0.1, 0.5, 1.0},
  ytick={0,...,4},
  yticklabels={
    Canonical,
    \uni,
    \stok,
    \stokuni,
    \unik,
  },
  y dir=reverse,
  tick align=outside,
  tick pos=left,
  every axis label/.append style={font=\small},
  every node near coord/.append style={
    font=\footnotesize, color=black, anchor=center,
    /pgf/number format/.cd, fixed, fixed zerofill
  },
]

\nextgroupplot[
  title={Canonical tokenisation},
]
\addplot [
  matrix plot,
  point meta=explicit,
  mesh/cols=4,
nodes near coords={\printmynumber\pgfplotspointmeta},
]
table [meta=Value] {
x y Value
0 0 0.575
1 0 nan
2 0 nan
3 0 nan
0 1 0.468
1 1 nan
2 1 nan
3 1 nan
0 2 nan
1 2 0.666
2 2 0.592
3 2 0.474
0 3 nan
1 3 0.676
2 3 0.566
3 3 0.496
0 4 nan
1 4 0.656
2 4 0.698
3 4 0.596
};

\nextgroupplot[
  title={Adversarial tokenisation \\ $v_0$ canonical},
  title style={align=center},
  colorbar style={
    ylabel={},
    y dir=normal,
  },
  yticklabels={
},
  y dir=reverse,
]
\addplot [
  matrix plot,
  point meta=explicit,
  mesh/cols=4,
nodes near coords={\printmynumber\pgfplotspointmeta},
]
table [meta=Value] {
x y Value
0 0 0.079
1 0 nan
2 0 nan
3 0 nan
0 1 0.094
1 1 nan
2 1 nan
3 1 nan
0 2 nan
1 2 0.018
2 2 0.088
3 2 0.134
0 3 nan
1 3 0.004
2 3 0.066
3 3 0.098
0 4 nan
1 4 0.004
2 4 0.044
3 4 0.098
};

\nextgroupplot[
  title={Adversarial tokenisation \\ $v_0$ random},
  title style={align=center},
colorbar style={
    ylabel={},
    y dir=normal,
  },
  yticklabels={
},
  y dir=reverse,
]
\addplot [
  matrix plot,
  point meta=explicit,
  mesh/cols=4,
nodes near coords={\printmynumber\pgfplotspointmeta},
]
table [meta=Value] {
x y Value
0 0 0.000
1 0 nan
2 0 nan
3 0 nan
0 1 0.058
1 1 nan
2 1 nan
3 1 nan
0 2 nan
1 2 0.000
2 2 0.010
3 2 0.028
0 3 nan
1 3 0.000
2 3 0.012
3 3 0.044
0 4 nan
1 4 0.000
2 4 0.012
3 4 0.026
};

\end{groupplot}
\end{tikzpicture}
    }
    \end{subfigure}\par\vspace{3mm}

    \begin{subfigure}{.85\textwidth}
      \appheader{Pre-training perturbation: $\alpha_{\mathrm{pre}}=0.1$}
      \resizebox{1\linewidth}{!}{%
      \pgfplotsset{
  colormap={scScheme}{
    color(0pt)=(white)
color(1pt)=(scCyan!60)
  }
}

\begin{tikzpicture}
\begin{groupplot}[
  group style={
    group size= 3 by 1,
    group name=rowA,
horizontal sep=0.5cm
  },
  width=0.33\textwidth,
  tick label style={font=\small},   colormap name=scScheme,
  point meta min=0,
  point meta max=1,
  x grid style={gray!30},
  y grid style={gray!30},
  xlabel={$\pfine$},
  xmin=-0.5, xmax=3.5,
  ymin=-0.5, ymax=4.5,
  xtick={0,...,3},
  xticklabels={-, 0.1, 0.5, 1.0},
  ytick={0,...,4},
colorbar style={
    ylabel={},
    y dir=normal,
  },
  yticklabels={
    Canonical,
    \uni,
    \stok,
    \stokuni,
    \unik,
  },
  y dir=reverse,
  tick align=outside,
  tick pos=left,
  every axis label/.append style={font=\small},
  every node near coord/.append style={
    font=\footnotesize, color=black, anchor=center,
    /pgf/number format/.cd, fixed, fixed zerofill
  },
]

\nextgroupplot[
  title={Canonical tokenisation},
]
\addplot [
  matrix plot,
  point meta=explicit,
  mesh/cols=4,
nodes near coords={\printmynumber\pgfplotspointmeta},
]
table [meta=Value] {
x y Value
0 0 0.808
1 0 nan
2 0 nan
3 0 nan
0 1 0.826
1 1 nan
2 1 nan
3 1 nan
0 2 nan
1 2 0.812
2 2 0.782
3 2 0.818
0 3 nan
1 3 0.816
2 3 0.786
3 3 0.806
0 4 nan
1 4 0.822
2 4 0.808
3 4 0.842
};

\nextgroupplot[
  title={Adversarial tokenisation \\ $v_0$ canonical},
  title style={align=center},
  colorbar style={
    ylabel={},
    y dir=normal,
  },
  yticklabels={
},
  y dir=reverse,
]
\addplot [
  matrix plot,
  point meta=explicit,
  mesh/cols=4,
nodes near coords={\printmynumber\pgfplotspointmeta},
]
table [meta=Value] {
x y Value
0 0 0.130
1 0 nan
2 0 nan
3 0 nan
0 1 0.524
1 1 nan
2 1 nan
3 1 nan
0 2 nan
1 2 0.170
2 2 0.344
3 2 0.370
0 3 nan
1 3 0.162
2 3 0.364
3 3 0.476
0 4 nan
1 4 0.186
2 4 0.410
3 4 0.564
};

\nextgroupplot[
  title={Adversarial tokenisation \\ $v_0$ random},
  title style={align=center},
colorbar style={
    ylabel={},
    y dir=normal,
  },
  yticklabels={
},
  y dir=reverse,
]
\addplot [
  matrix plot,
  point meta=explicit,
  mesh/cols=4,
nodes near coords={\printmynumber\pgfplotspointmeta},
]
table [meta=Value] {
x y Value
0 0 0.012
1 0 nan
2 0 nan
3 0 nan
0 1 0.474
1 1 nan
2 1 nan
3 1 nan
0 2 nan
1 2 0.042
2 2 0.084
3 2 0.136
0 3 nan
1 3 0.040
2 3 0.212
3 3 0.406
0 4 nan
1 4 0.048
2 4 0.312
3 4 0.506
};

\end{groupplot}

\end{tikzpicture}
    }
    \end{subfigure}\par\vspace{3mm}

    \begin{subfigure}{.85\textwidth}
      \appheader{Pre-training perturbation: $\alpha_{\mathrm{pre}}=0.5$}
      \resizebox{1\linewidth}{!}{%
      \pgfplotsset{
  colormap={scScheme}{
    color(0pt)=(white)
color(1pt)=(scCyan!60)
  }
}

\begin{tikzpicture}
\begin{groupplot}[
  group style={
    group size=3 by 1,
    horizontal sep=.5cm,
  },
  width=0.33\textwidth,
  tick label style={font=\small},   colormap name=scScheme,
  point meta min=0,
  point meta max=1,
  x grid style={gray!30},
  y grid style={gray!30},
  xlabel={$\pfine$},
  xmin=-0.5, xmax=3.5,
  ymin=-0.5, ymax=4.5,
  xtick={0,...,3},
  xticklabels={-, 0.1, 0.5, 1.0},
  ytick={0,...,4},
  yticklabels={
    Canonical,
    \uni,
    \stok,
    \stokuni,
    \unik,
  },
  y dir=reverse,
  tick align=outside,
  tick pos=left,
  every axis label/.append style={font=\small},
  every node near coord/.append style={
    font=\footnotesize, color=black, anchor=center,
    /pgf/number format/.cd, fixed, fixed zerofill
  },
]

\nextgroupplot[
  title={Canonical tokenisation},
]
\addplot [
  matrix plot,
  point meta=explicit,
  mesh/cols=4,
nodes near coords={\printmynumber\pgfplotspointmeta},
]
table [meta=Value] {
x y Value
0 0 0.847
1 0 nan
2 0 nan
3 0 nan
0 1 0.768
1 1 nan
2 1 nan
3 1 nan
0 2 nan
1 2 0.844
2 2 0.804
3 2 0.820
0 3 nan
1 3 0.836
2 3 0.824
3 3 0.806
0 4 nan
1 4 0.846
2 4 0.852
3 4 0.800
};

\nextgroupplot[
  title={Adversarial tokenisation \\ $v_0$ canonical},
  title style={align=center},
  colorbar style={
    ylabel={},
    y dir=normal,
  },
  yticklabels={
},
  y dir=reverse,
]
\addplot [
  matrix plot,
  point meta=explicit,
  mesh/cols=4,
  nodes near coords=\printmynumber\pgfplotspointmeta,
]
table [meta=Value] {
x y Value
0 0 0.299
1 0 nan
2 0 nan
3 0 nan
0 1 0.418
1 1 nan
2 1 nan
3 1 nan
0 2 nan
1 2 0.302
2 2 0.378
3 2 0.430
0 3 nan
1 3 0.288
2 3 0.474
3 3 0.454
0 4 nan
1 4 0.346
2 4 0.504
3 4 0.504
};

\nextgroupplot[
  title={Adversarial tokenisation \\ $v_0$ random},
  title style={align=center},
colorbar style={
    ylabel={},
    y dir=normal,
  },
  yticklabels={
},
  y dir=reverse,
]
\addplot [
  matrix plot,
  point meta=explicit,
  mesh/cols=4,
  nodes near coords=\printmynumber\pgfplotspointmeta,
]
table [meta=Value] {
x y Value
0 0 0.042
1 0 nan
2 0 nan
3 0 nan
0 1 0.406
1 1 nan
2 1 nan
3 1 nan
0 2 nan
1 2 0.142
2 2 0.264
3 2 0.290
0 3 nan
1 3 0.146
2 3 0.424
3 3 0.446
0 4 nan
1 4 0.204
2 4 0.494
3 4 0.526
};

\end{groupplot}
\end{tikzpicture}
    }
    \end{subfigure}

    \caption{Effect of perturbation strength $\ppre$, $\pfine$, and different training schemes on adversarial robustness for the \langgame data set.}
    \label{fig:adv_robustness_tiny_langgame_ablation}
\end{figure}

\newpage

\begin{figure}[h!]
    \centering
{ \textit{Data set}: \cute \\ \textit{Model}: \tinyllm}\par\medskip
    \begin{subfigure}{.85\textwidth}
      \appheader{Pre-training perturbation: $\alpha_{\mathrm{pre}}=0$}
      \resizebox{1\linewidth}{!}{%
      \pgfplotsset{
  colormap={scScheme}{
    color(0pt)=(white)
color(1pt)=(scCyan!60)
  }
}

\begin{tikzpicture}
\begin{groupplot}[
  group style={
    group size=3 by 1,
    horizontal sep=.5cm,
  },
  width=0.33\textwidth,
  tick label style={font=\small},   colormap name=scScheme,
  point meta min=0,
  point meta max=1,
  x grid style={gray!30},
  y grid style={gray!30},
  xlabel={$\pfine$},
  xmin=-0.5, xmax=3.5,
  ymin=-0.5, ymax=4.5,
  xtick={0,...,3},
  xticklabels={-, 0.1, 0.5, 1.0},
  ytick={0,...,4},
  yticklabels={
    Canonical,
    \uni,
    \stok,
    \stokuni,
    \unik,
  },
  y dir=reverse,
  tick align=outside,
  tick pos=left,
  every axis label/.append style={font=\small},
  every node near coord/.append style={
    font=\footnotesize, color=black, anchor=center,
    /pgf/number format/.cd, fixed, fixed zerofill
  },
]

\nextgroupplot[
  title={Canonical tokenisation},
]
\addplot [
  matrix plot,
  point meta=explicit,
  mesh/cols=4,
nodes near coords={\printmynumber\pgfplotspointmeta},
]
table [meta=Value] {
x y Value
0 0 0.616
1 0 nan
2 0 nan
3 0 nan
0 1 0.681
1 1 nan
2 1 nan
3 1 nan
0 2 nan
1 2 0.662
2 2 0.678
3 2 0.713
0 3 nan
1 3 0.646
2 3 0.701
3 3 0.701
0 4 nan
1 4 0.641
2 4 0.717
3 4 0.743
};

\nextgroupplot[
  title={Adversarial tokenisation \\ $v_0$ canonical},
  title style={align=center},
  colorbar style={
    ylabel={},
    y dir=normal,
  },
  yticklabels={
},
  y dir=reverse,
]
\addplot [
  matrix plot,
  point meta=explicit,
  mesh/cols=4,
nodes near coords={\printmynumber\pgfplotspointmeta},
]
table [meta=Value] {
x y Value
0 0 0.070
1 0 nan
2 0 nan
3 0 nan
0 1 0.395
1 1 nan
2 1 nan
3 1 nan
0 2 nan
1 2 0.135
2 2 0.311
3 2 0.392
0 3 nan
1 3 0.142
2 3 0.322
3 3 0.382
0 4 nan
1 4 0.109
2 4 0.319
3 4 0.432
};

\nextgroupplot[
  title={Adversarial tokenisation \\ $v_0$ random},
  title style={align=center},
colorbar style={
    ylabel={},
    y dir=normal,
  },
  yticklabels={
},
  y dir=reverse,
]
\addplot [
  matrix plot,
  point meta=explicit,
  mesh/cols=4,
nodes near coords={\printmynumber\pgfplotspointmeta},
]
table [meta=Value] {
x y Value
0 0 0.015
1 0 nan
2 0 nan
3 0 nan
0 1 0.382
1 1 nan
2 1 nan
3 1 nan
0 2 nan
1 2 0.069
2 2 0.249
3 2 0.334
0 3 nan
1 3 0.062
2 3 0.270
3 3 0.341
0 4 nan
1 4 0.058
2 4 0.266
3 4 0.384
};

\end{groupplot}
\end{tikzpicture}
    }
    \end{subfigure}\par\vspace{3mm}
    \begin{subfigure}{.85\textwidth}
      \appheader{Pre-training perturbation: $\alpha_{\mathrm{pre}}=0.1$}
      \resizebox{1\linewidth}{!}{%
      \pgfplotsset{
  colormap={scScheme}{
    color(0pt)=(white)
color(1pt)=(scCyan!60)
  }
}

\begin{tikzpicture}
\begin{groupplot}[
  group style={
    group size=3 by 1,
    horizontal sep=.5cm,
  },
  width=0.33\textwidth,
  tick label style={font=\small},   colormap name=scScheme,
  point meta min=0,
  point meta max=1,
  x grid style={gray!30},
  y grid style={gray!30},
  xlabel={$\pfine$},
  xmin=-0.5, xmax=3.5,
  ymin=-0.5, ymax=4.5,
  xtick={0,...,3},
  xticklabels={-, 0.1, 0.5, 1.0},
  ytick={0,...,4},
  yticklabels={
    Canonical,
    \uni,
    \stok,
    \stokuni,
    \unik,
  },
  y dir=reverse,
  tick align=outside,
  tick pos=left,
  every axis label/.append style={font=\small},
  every node near coord/.append style={
    font=\footnotesize, color=black, anchor=center,
    /pgf/number format/.cd, fixed, fixed zerofill
  },
]

\nextgroupplot[
  title={Canonical tokenisation},
]
\addplot [
  matrix plot,
  point meta=explicit,
  mesh/cols=4,
nodes near coords={\printmynumber\pgfplotspointmeta},
]
table [meta=Value] {
x y Value
0 0 0.726
1 0 nan
2 0 nan
3 0 nan
0 1 0.790
1 1 nan
2 1 nan
3 1 nan
0 2 nan
1 2 0.743
2 2 0.805
3 2 0.803
0 3 nan
1 3 0.757
2 3 0.810
3 3 0.805
0 4 nan
1 4 0.753
2 4 0.797
3 4 0.822
};

\nextgroupplot[
  title={Adversarial tokenisation \\ $v_0$ canonical},
  title style={align=center},
  colorbar style={
    ylabel={},
    y dir=normal,
  },
  yticklabels={
},
  y dir=reverse,
]
\addplot [
  matrix plot,
  point meta=explicit,
  mesh/cols=4,
nodes near coords={\printmynumber\pgfplotspointmeta},
]
table [meta=Value] {
x y Value
0 0 0.270
1 0 nan
2 0 nan
3 0 nan
0 1 0.592
1 1 nan
2 1 nan
3 1 nan
0 2 nan
1 2 0.357
2 2 0.554
3 2 0.599
0 3 nan
1 3 0.365
2 3 0.570
3 3 0.607
0 4 nan
1 4 0.333
2 4 0.577
3 4 0.615
};

\nextgroupplot[
  title={Adversarial tokenisation \\ $v_0$ random},
  title style={align=center},
colorbar style={
    ylabel={},
    y dir=normal,
  },
  yticklabels={
},
  y dir=reverse,
]
\addplot [
  matrix plot,
  point meta=explicit,
  mesh/cols=4,
nodes near coords={\printmynumber\pgfplotspointmeta},
]
table [meta=Value] {
x y Value
0 0 0.085
1 0 nan
2 0 nan
3 0 nan
0 1 0.571
1 1 nan
2 1 nan
3 1 nan
0 2 nan
1 2 0.230
2 2 0.482
3 2 0.552
0 3 nan
1 3 0.254
2 3 0.521
3 3 0.574
0 4 nan
1 4 0.208
2 4 0.527
3 4 0.575
};

\end{groupplot}
\end{tikzpicture}
    }
    \end{subfigure}\par\vspace{3mm}

    \begin{subfigure}{.85\textwidth}
      \appheader{Pre-training perturbation: $\alpha_{\mathrm{pre}}=0.5$}
      \resizebox{1\linewidth}{!}{%
      \pgfplotsset{
  colormap={scScheme}{
    color(0pt)=(white)
color(1pt)=(scCyan!60)
  }
}

\begin{tikzpicture}
\begin{groupplot}[
  group style={
    group size=3 by 1,
    horizontal sep=.5cm,
  },
  width=0.33\textwidth,
  tick label style={font=\small},   colormap name=scScheme,
  point meta min=0,
  point meta max=1,
  x grid style={gray!30},
  y grid style={gray!30},
  xlabel={$\pfine$},
  xmin=-0.5, xmax=3.5,
  ymin=-0.5, ymax=4.5,
  xtick={0,...,3},
  xticklabels={-, 0.1, 0.5, 1.0},
  ytick={0,...,4},
  yticklabels={
    Canonical,
    \uni,
    \stok,
    \stokuni,
    \unik,
  },
  y dir=reverse,
  tick align=outside,
  tick pos=left,
  every axis label/.append style={font=\small},
  every node near coord/.append style={
    font=\footnotesize, color=black, anchor=center,
    /pgf/number format/.cd, fixed, fixed zerofill
  },
]

\nextgroupplot[
  title={Canonical tokenisation},
]
\addplot [
  matrix plot,
  point meta=explicit,
  mesh/cols=4,
nodes near coords={\printmynumber\pgfplotspointmeta},
]
table [meta=Value] {
x y Value
0 0 0.791
1 0 nan
2 0 nan
3 0 nan
0 1 0.834
1 1 nan
2 1 nan
3 1 nan
0 2 nan
1 2 0.834
2 2 0.859
3 2 0.865
0 3 nan
1 3 0.806
2 3 0.864
3 3 0.864
0 4 nan
1 4 0.818
2 4 0.847
3 4 0.861
};

\nextgroupplot[
  title={Adversarial tokenisation \\ $v_0$ canonical},
  title style={align=center},
  colorbar style={
    ylabel={},
    y dir=normal,
  },
  yticklabels={
},
  y dir=reverse,
]
\addplot [
  matrix plot,
  point meta=explicit,
  mesh/cols=4,
nodes near coords={\printmynumber\pgfplotspointmeta},
]
table [meta=Value] {
x y Value
0 0 0.211
1 0 nan
2 0 nan
3 0 nan
0 1 0.647
1 1 nan
2 1 nan
3 1 nan
0 2 nan
1 2 0.522
2 2 0.669
3 2 0.675
0 3 nan
1 3 0.531
2 3 0.681
3 3 0.675
0 4 nan
1 4 0.488
2 4 0.669
3 4 0.706
};

\nextgroupplot[
  title={Adversarial tokenisation \\ $v_0$ random},
  title style={align=center},
colorbar style={
    ylabel={},
    y dir=normal,
  },
  yticklabels={
},
  y dir=reverse,
]
\addplot [
  matrix plot,
  point meta=explicit,
  mesh/cols=4,
nodes near coords={\printmynumber\pgfplotspointmeta},
]
table [meta=Value] {
x y Value
0 0 0.094
1 0 nan
2 0 nan
3 0 nan
0 1 0.641
1 1 nan
2 1 nan
3 1 nan
0 2 nan
1 2 0.440
2 2 0.615
3 2 0.645
0 3 nan
1 3 0.451
2 3 0.645
3 3 0.650
0 4 nan
1 4 0.408
2 4 0.630
3 4 0.672
};

\end{groupplot}
\end{tikzpicture}
    }
    \end{subfigure}

    \caption{Effect of perturbation strength $\ppre$, $\pfine$, and different training schemes on adversarial robustness for the \cute data set.}
    \label{fig:adv_robustness}
\end{figure}

\FloatBarrier
\subsubsection{Additional Results for \gptxl} \label{app:exp_diff_models}

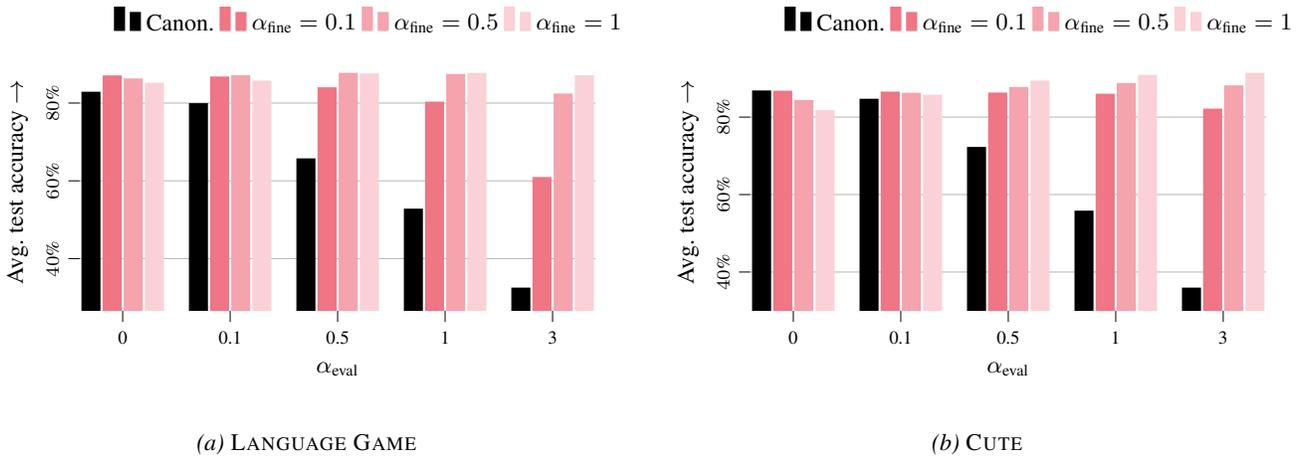
\begin{figure}[h]
    \centering
    \setlength{\figureheight}{0.2\linewidth}
    \setlength{\figurewidth}{.4\linewidth}

    \pgfplotsset{
        x tick label style={font=\scriptsize},
        y tick label style={rotate=90, font=\scriptsize},
        ylabel={\small Avg.\ test accuracy $\rightarrow$},
        xlabel={\small $\peval$},
        scale only axis,
        tick align=outside,
        tick pos=left,
        every axis plot/.append style={mark size=3pt},
        axis x line*=bottom,
        axis y line*=left,
        axis line style={draw=none},
        grid style={solid},
    }

    \begin{subfigure}[b]{0.48\textwidth}
        \centering
        \begin{tikzpicture}
\begin{axis}[
  height=\figureheight,
  width=\figurewidth,
  ybar,
  bar width=6pt,
  enlarge x limits=0.1,
  symbolic x coords={0,0.1,0.5,1,3},
  xtick=data,
  xtick style={color=black},
  ytick style={color=black},
  ymajorgrids,
  yticklabel={
    \pgfmathparse{100*\tick}\pgfmathprintnumber[fixed,precision=0]{\pgfmathresult}\%
  },
  grid style={scGrey},
  legend cell align={left},
  legend style={fill opacity=1,legend columns=4,draw opacity=1,text opacity=1,at={(0.05,1.03)},anchor=south west,draw=none,font=\footnotesize},
]

\addplot+[very thick, draw=black, fill=black, bar shift=-12pt]
coordinates {
  (0,0.825)
  ({0.1},0.7956)
  ({0.5},0.65365)
  (1,0.52435)
  (3,0.32135)
};
\addlegendentry{Canon.}

\addplot+[very thick, draw=scRed!90, fill=scRed!90, bar shift=-4pt]
coordinates {
  (0,0.867)
  ({0.1},0.8642)
  ({0.5},0.8365)
  (1,0.7996)
  (3,0.6057)
};
\addlegendentry{$\pfine=0.1$}

\addplot+[very thick, draw=scRed!60, fill=scRed!60, bar shift=4pt]
coordinates {
  (0,0.859)
  ({0.1},0.8674)
  ({0.5},0.8736)
  (1,0.8708)
  (3,0.8205)
};
\addlegendentry{$\pfine=0.5$}

\addplot+[very thick, draw=scRed!30, fill=scRed!30, bar shift=12pt]
coordinates {
  (0,0.848)
  ({0.1},0.8536)
  ({0.5},0.872)
  (1,0.8733)
  (3,0.8673)
};
\addlegendentry{$\pfine=1$}

\end{axis}
\end{tikzpicture}
        \caption{\langgame}
        \label{fig:app-gpt2xl-langgame}
    \end{subfigure}
    \hfill
    \begin{subfigure}[b]{0.48\textwidth}
        \centering
        \begin{tikzpicture}
\begin{axis}[
  height=\figureheight,
  width=\figurewidth,
  ybar,
  bar width=6pt,
  enlarge x limits=0.1,
  symbolic x coords={0,0.1,0.5,1,3},
  xtick=data,
  xtick style={color=black},
  ytick style={color=black},
  ymajorgrids,
  yticklabel={
    \pgfmathparse{100*\tick}\pgfmathprintnumber[fixed,precision=0]{\pgfmathresult}\%
  },
  grid style={scGrey},
  legend cell align={left},
  legend style={fill opacity=1,legend columns=4,draw opacity=1,text opacity=1,at={(0.05,1.03)},anchor=south west,draw=none,font=\footnotesize},
]

\addplot+[very thick, draw=black, fill=black, bar shift=-12pt]
coordinates {
  (0,0.864666666666667)
  ({0.1},0.843266666666667)
  ({0.5},0.718866666666667)
  (1,0.554333333333333)
  (3,0.355666666666667)
};
\addlegendentry{Canon.}

\addplot+[very thick, draw=scRed!90, fill=scRed!90, bar shift=-4pt]
coordinates {
  (0,0.864)
  ({0.1},0.8618)
  ({0.5},0.8594)
  (1,0.8562)
  (3,0.8178)
};
\addlegendentry{$\pfine=0.1$}

\addplot+[very thick, draw=scRed!60, fill=scRed!60, bar shift=4pt]
coordinates {
  (0,0.84)
  ({0.1},0.8588)
  ({0.5},0.8736)
  (1,0.884)
  (3,0.878)
};
\addlegendentry{$\pfine=0.5$}

\addplot+[very thick, draw=scRed!30, fill=scRed!30, bar shift=12pt]
coordinates {
  (0,0.814)
  ({0.1},0.8536)
  ({0.5},0.89)
  (1,0.9048)
  (3,0.9102)
};
\addlegendentry{$\pfine=1$}

\end{axis}
\end{tikzpicture}
        \caption{\cute}
        \label{fig:app-gpt2xl-cute}
    \end{subfigure}

    \caption{Average accuracy of \gptxl for stochastic tokenisation during evaluation and fine-tuning on \langgame and \cute. Stochastic fine-tuning improves robustness to random perturbations during evaluation.}
    \label{fig:app_finetuning_robustness_gpt2xl}
\end{figure}

\begin{figure}[h!]
    \centering
    \begin{subfigure}{.85\textwidth}\centering
    \appheader{Data set: \langgame}
      \pgfplotsset{
  colormap={scScheme}{
    color(0pt)=(white)
color(1pt)=(scCyan!60)
  }
}

\begin{tikzpicture}
\begin{groupplot}[
  group style={
    group size=3 by 1,
    horizontal sep=.5cm,
  },
  width=0.33\textwidth,
  tick label style={font=\small},   colormap name=scScheme,
  point meta min=0,
  point meta max=1,
  x grid style={gray!30},
  y grid style={gray!30},
  xlabel={$\pfine$},
  xmin=-0.5, xmax=3.5,
  ymin=-0.5, ymax=4.5,
  xtick={0,...,3},
  xticklabels={-, 0.1, 0.5, 1.0},
  ytick={0,...,4},
  yticklabels={
    Canonical,
    \uni,
    \stok,
    \stokuni,
    \unik,
  },
  y dir=reverse,
  tick align=outside,
  tick pos=left,
  every axis label/.append style={font=\small},
  every node near coord/.append style={
    font=\footnotesize, color=black, anchor=center,
    /pgf/number format/.cd, fixed, fixed zerofill
  },
]

\nextgroupplot[
  title={Canonical tokenisation},
]
\addplot [
  matrix plot,
  point meta=explicit,
  mesh/cols=4,
nodes near coords={\printmynumber\pgfplotspointmeta},
]
table [meta=Value] {
x y Value
0 0 0.831
1 0 nan
2 0 nan
3 0 nan
0 1 0.498
1 1 nan
2 1 nan
3 1 nan
0 2 nan
1 2 0.870
2 2 0.865
3 2 0.851
0 3 nan
1 3 0.864
2 3 0.853
3 3 0.855
0 4 nan
1 4 0.866
2 4 0.870
3 4 0.832
};

\nextgroupplot[
  title={Adversarial tokenisation \\ $v_0$ canonical},
  title style={align=center},
  colorbar style={
    ylabel={},
    y dir=normal,
  },
  yticklabels={
},
  y dir=reverse,
]
\addplot [
  matrix plot,
  point meta=explicit,
  mesh/cols=4,
nodes near coords={\printmynumber\pgfplotspointmeta},
]
table [meta=Value] {
x y Value
0 0 0.073
1 0 nan
2 0 nan
3 0 nan
0 1 0.036
1 1 nan
2 1 nan
3 1 nan
0 2 nan
1 2 0.029
2 2 0.176
3 2 0.299
0 3 nan
1 3 0.019
2 3 0.203
3 3 0.319
0 4 nan
1 4 0.026
2 4 0.220
3 4 0.330
};

\end{groupplot}
\end{tikzpicture}
    \end{subfigure}\par\vspace{3mm}
    \begin{subfigure}{.85\textwidth}\centering
    \appheader{Data set: \cute}
      \pgfplotsset{
  colormap={scScheme}{
    color(0pt)=(white)
color(1pt)=(scCyan!60)
  }
}

\begin{tikzpicture}
\begin{groupplot}[
  group style={
    group size=3 by 1,
    horizontal sep=.5cm,
  },
  width=0.33\textwidth,
  tick label style={font=\small},   colormap name=scScheme,
  point meta min=0,
  point meta max=1,
  x grid style={gray!30},
  y grid style={gray!30},
  xlabel={$\pfine$},
  xmin=-0.5, xmax=3.5,
  ymin=-0.5, ymax=4.5,
  xtick={0,...,3},
  xticklabels={-, 0.1, 0.5, 1.0},
  ytick={0,...,4},
  yticklabels={
    Canonical,
    \uni,
    \stok,
    \stokuni,
    \unik,
  },
  y dir=reverse,
  tick align=outside,
  tick pos=left,
  every axis label/.append style={font=\small},
  every node near coord/.append style={
    font=\footnotesize, color=black, anchor=center,
    /pgf/number format/.cd, fixed, fixed zerofill
  },
]

\nextgroupplot[
  title={Canonical tokenisation},
]
\addplot [
  matrix plot,
  point meta=explicit,
  mesh/cols=4,
nodes near coords={\printmynumber\pgfplotspointmeta},
]
table [meta=Value] {
x y Value
0 0 0.846
1 0 nan
2 0 nan
3 0 nan
0 1 0.775
1 1 nan
2 1 nan
3 1 nan
0 2 nan
1 2 0.843
2 2 0.837
3 2 0.806
0 3 nan
1 3 0.846
2 3 0.819
3 3 0.810
0 4 nan
1 4 0.845
2 4 0.848
3 4 0.834
};

\nextgroupplot[
  title={Adversarial tokenisation \\ $v_0$ canonical},
  title style={align=center},
  colorbar style={
    ylabel={},
    y dir=normal,
  },
  yticklabels={
},
  y dir=reverse,
]
\addplot [
  matrix plot,
  point meta=explicit,
  mesh/cols=4,
nodes near coords={\printmynumber\pgfplotspointmeta},
]
table [meta=Value] {
x y Value
0 0 0.101
1 0 nan
2 0 nan
3 0 nan
0 1 0.440
1 1 nan
2 1 nan
3 1 nan
0 2 nan
1 2 0.373
2 2 0.503
3 2 0.483
0 3 nan
1 3 0.365
2 3 0.517
3 3 0.506
0 4 nan
1 4 0.280
2 4 0.539
3 4 0.570
};

\end{groupplot}
\end{tikzpicture}
    \end{subfigure}\par\vspace{3mm}
\caption{Accuracy of \gptxl under canonical and adversarial tokenisation on \langgame and \cute. Stochastic fine-tuning consistently improves robustness to perturbations during evaluation, with highest gains for methods that address support and sampling bias (\stokuni, \unik.}
\end{figure}
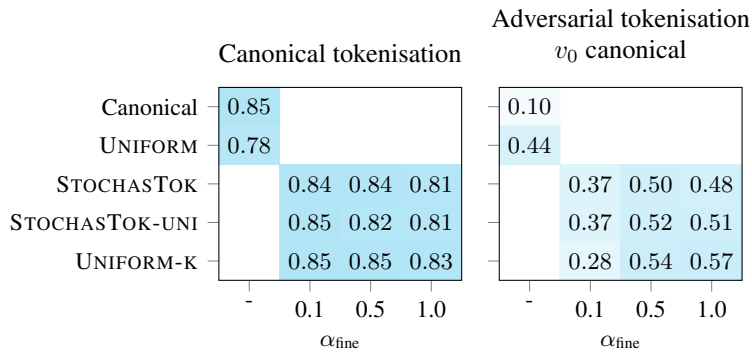

\newpage
\FloatBarrier

\subsubsection{Ablation on Tokenisation Support}\label{app:tokenisation_support}

\stok induces a distribution over tokenisations that is biased and with incomplete support. We showed that addressing both issues with \stokuni results in increased robustness to adversarial tokenisations. In this ablation, we study which of the two extensions (uniformity or support) is the main driver of the improvements. Thus, we construct two additional stochastic tokenisation schemes. On the one hand, we use the same \stok algorithm as in \Cref{alg:stochastok}, but build the dictionary of possible expansions with 2-way and 3-way merges. This still results in a biased distribution, but increases the support of possible tokenisations. On the other hand, we restrict \stokuni to sample only tokenisations that are reachable with \stok, \ie achievable via 2-way merges. 

We test the robustness to adversarial tokenisation attacks for three different attack settings. Each attack consists of $10$ iterations. First, we vary the initial tokenisation $\bv_0$: canonical $\bv_0=\tau(\bx)$, or uniform $\bv_0\sim\distUnif(\mcT_\mcV(\bx))$. Second, we increase the neighbourhood of tokenisations for the greedy search in each iteration step by increasing the edit distance $d_k=\{2,4\}$. 
\Cref{fig:app:tokenisation-support-barplot-langgame} shows the results for 10 LoRA fine-tuned \llamaoneb models on \langgame with our stochastic tokenisation schemes for different expansion proportions $\pfine$. 
The first case ($d_k=2,\ \bv_0=\tau(\bx)$) corresponds to the easiest setting. Within the tokenisation neighbourhood with edit distance $d_k=2$, all methods perform equally well. However, when increasing the neighbourhood to $d_k=4$ \textit{or} starting from a random initial tokenisation, methods with increased support exhibit higher robustness. The same trend, albeit to a smaller extend is visible for the \cute dataset in \Cref{fig:app:tokenisation-support-barplot-cute}. 

\begin{figure}[h]
\centering

\begin{subfigure}{1\linewidth}
    \centering
    \input{plotting/ablation/support_langgame}
    \caption{\langgame dataset}  \label{fig:app:tokenisation-support-barplot-langgame}
\end{subfigure}
\hfill \vspace*{1cm}
\begin{subfigure}{1\linewidth}
    \centering
    \input{plotting/ablation/support_cute}
    \caption{\cute dataset} \label{fig:app:tokenisation-support-barplot-cute}
\end{subfigure}

\caption{Adversarial accuracy (\%) of 10 LoRA fine-tuned \llamaoneb models versus fine-tuning stochasticity $\pfine$ for three attack settings (defined by neighbourhood radius $d_k$ and initial tokenisation $\bv_0$). Methods with larger support over $\mcT_\mcV(\bx)$ (\stok (2- and 3-way merges), \stokuni (all merges)) are more robust. Increasing uniformity on the same support (\stokuni (2-way merges)) yields no significant improvements.}
\label{fig:app:tokenisation-support-barplot}

\end{figure}

\FloatBarrier

\subsubsection{Comparison to BPE-dropout} \label{app:bpedropout}

BPE-dropout \cite{provilkov2020bpe} is a popular stochastic tokenisation method for BPE tokenisers. It introduces stochasticity in the tokenisation process by randomly skipping merge rules with a specified dropout probability $p_{drop}$.  

However, a direct comparison to our methods is not straightforward: The perturbation strength is controlled indirectly via the dropout rate $p_{\text{drop}}$, and the resulting number of induced splits further depends on the input string and tokeniser's merge structure. In contrast, the methods considered in the main text allow explicit control over the number of induced splits. 

To illustrate this, we apply BPE-dropout to the example string from \Cref{app:sec:splitcounts}, namely \emph{``revolution is a rapid, fundamental transformation of a society's class, state, ethnic or religious structures''}. We sample $50{,}000$ stochastic tokenisations with BPE-dropout using the Llama 3 tokeniser and track the resulting non-canonical segmentations. 
First, \Cref{tab:bpedropout_splits} reports the statistics of the total number of induced splits for different values of $p_{drop}$. The effective perturbation strength varies substantially even for a fixed value of $p_{\text{drop}}$. Second, \Cref{app:fig:seg_histograms_appendix} summarises the distribution over the non-canonical tokenisations of the first token (\emph{``revolution''}). This reflects the findings of \cite{cognetta2024distributional} that BPE-dropout has a strong sampling bias, which may lead to certain non-canonical tokenisations to be underrepresented.

\begin{table}[h]
\centering
\caption{Induced number of splits under BPE-dropout for different dropout rates $p_{\text{drop}}$.}
\label{tab:bpedropout_splits}
\begin{tabular}{lrrrr}
\toprule
$p_{\text{drop}}$ & mean & std & min & max \\
\midrule
0.10 & 0.75  & 0.90 & 0  & 6  \\
0.25 & 3.01  & 1.89 & 0  & 14 \\
0.50 & 22.41 & 4.62 & 6  & 44 \\
0.75 & 53.06 & 6.13 & 29 & 78 \\
\bottomrule
\end{tabular}
\end{table}

\begin{figure}[h]
\appheader{Llama 3}
\centering

\begin{subfigure}[t]{0.48\linewidth}
\centering
\begin{tikzpicture}[baseline=(current bounding box.north)]
\begin{axis}[
  ybar,
  bar width=10pt,
  width=\linewidth,
  height=5cm,
  ymin=0,
  ylabel={Probability (\%)},
  scaled y ticks=false,
  yticklabel={
  \pgfmathparse{100*\tick}%
  \pgfmathprintnumber[fixed,precision=0]{\pgfmathresult}%
  },
  xlabel={},
  symbolic x coords={
    rev-olution,
    re-volution,
  },
  xtick=data,
  xticklabel style={font=\ttfamily\small, rotate=90, anchor=east},
  enlarge x limits=0.3,
  legend style={at={(0.5,1.02)}, anchor=south, legend columns=2},
]
\addplot+[bar shift=-6pt,fill=scRed, draw=scRed] table[
  col sep=comma,
  x=Segmentation,
  y=Frequency] {
Segmentation,Frequency
rev-olution,0.5012787723785166
re-volution,0.49872122762148335

};
\addplot+[bar shift=+6pt,fill=black!60, draw=black!60] table[
  col sep=comma,
  x=Segmentation,
  y=Frequency] {
Segmentation,Frequency
rev-olution,0.8716309948474039
re-volution,0.1283690051525961
};
\legend{\uni~~, \bpedropout}
\end{axis}
\end{tikzpicture}
\label{fig:seg_hist_coarse_appendix_b}
\end{subfigure}
\hfill 
\begin{subfigure}[t]{0.48\linewidth}
\centering
\begin{tikzpicture}[baseline=(current bounding box.north)]
\begin{axis}[
  ybar,
  bar width=8pt,
  width=\linewidth,
  height=5cm,
  ymin=0,
  ylabel={Probability (\%)},
  scaled y ticks=false,
  yticklabel={
  \pgfmathparse{100*\tick}%
  \pgfmathprintnumber[fixed,precision=0]{\pgfmathresult}%
  },
  xlabel={},
  symbolic x coords={
    re-vol-ution,
    re-v-olution,
    r-ev-olution,
    rev-ol-ution,
    rev-olut-ion,
    r-e-volution,
    rev-olu-tion,
  },
  xtick=data,
  xticklabel style={font=\ttfamily\small, rotate=90, anchor=east},
  enlarge x limits=0.18,
  legend style={at={(0.5,1.02)}, anchor=south, legend columns=2},
]
\addplot+[bar shift=-5pt,fill=scRed, draw=scRed] table[
  col sep=comma,
  x=Segmentation,
  y=Frequency] {
Segmentation,Frequency
re-vol-ution,0.13700918964076858
re-v-olution,0.14118629908103592
r-ev-olution,0.13450292397660818
rev-ol-ution,0.14118629908103592
rev-olut-ion,0.14619883040935672
r-e-volution,0.14703425229741018
rev-olu-tion,0.15288220551378445
};
\addplot+[bar shift=+5pt,fill=black!60, draw=black!60] table[
  col sep=comma,
  x=Segmentation,
  y=Frequency] {
Segmentation,Frequency
re-vol-ution,0.6384131801430739
re-v-olution,0.1331021027530891
rev-ol-ution,0.13115109473227834
rev-olut-ion,0.06741816605246044
r-ev-olution,0.019076522870149576
rev-olu-tion,0.010622154779969651
r-e-volution,0.00021677866897897247
};
\legend{\uni~~, \bpedropout}
\end{axis}
\end{tikzpicture}
\label{fig:seg_hist_coarse_appendix}
\end{subfigure}
\par\bigskip
\begin{subfigure}[t]{0.48\linewidth}
\centering
\begin{tikzpicture}[baseline=(current bounding box.north)]
\begin{axis}[
  ybar,
  bar width=2pt,
  width=\linewidth,
  height=5cm,
  ymin=0,
  ylabel={Probability (\%)},
  scaled y ticks=false,
  yticklabel={
  \pgfmathparse{100*\tick}%
  \pgfmathprintnumber[fixed,precision=0]{\pgfmathresult}%
  },
  xlabel={},
  symbolic x coords={
    re-vo-l-ution,
    r-e-vol-ution,
    r-ev-olut-ion,
    re-v-olut-ion,
    re-vol-uti-on,
    rev-ol-u-tion,
    rev-olu-ti-on,
    rev-ol-ut-ion,
    rev-olut-io-n,
    re-vo-lu-tion,
    r-e-v-olution,
    rev-o-lu-tion,
    rev-olut-i-on,
    re-v-ol-ution,
    rev-ol-uti-on,
    re-vol-ut-ion,
    rev-olu-t-ion,
    rev-o-lut-ion,
    re-v-olu-tion,
    r-ev-olu-tion,
    re-vo-lut-ion,
    re-vol-u-tion,
    r-ev-ol-ution,
    rev-o-l-ution,
  },
  xtick=data,
  xticklabel style={font=\ttfamily\small, rotate=90, anchor=east},
  enlarge x limits=0.05,
  legend style={at={(0.5,1.02)}, anchor=south, legend columns=2},
]
\addplot+[bar shift=-1.5pt,fill=scRed, draw=scRed] table[
  col sep=comma,
  x=Segmentation,
  y=Frequency
] {
Segmentation,Frequency
re-vo-l-ution,0.03501855287569573
r-e-vol-ution,0.04962894248608534
r-ev-olut-ion,0.04638218923933209
re-v-olut-ion,0.040584415584415584
re-vol-uti-on,0.03965677179962894
rev-ol-u-tion,0.040352504638218926
rev-olu-ti-on,0.04359925788497217
rev-ol-ut-ion,0.03780148423005566
rev-olut-io-n,0.041280148423005564
re-vo-lu-tion,0.046150278293135436
r-e-v-olution,0.041743970315398886
rev-o-lu-tion,0.041743970315398886
rev-olut-i-on,0.0387291280148423
re-v-ol-ution,0.0364100185528757
rev-ol-uti-on,0.04406307977736549
re-vol-ut-ion,0.03455473098330241
rev-olu-t-ion,0.04638218923933209
rev-o-lut-ion,0.04151205936920223
re-v-olu-tion,0.04406307977736549
r-ev-olu-tion,0.03849721706864564
re-vo-lut-ion,0.04568645640074211
re-vol-u-tion,0.042439703153988866
r-ev-ol-ution,0.03896103896103896
rev-o-l-ution,0.044758812615955476
};
\addplot+[bar shift=+1.5pt,fill=black!60, draw=black!60] table[
  col sep=comma,
  x=Segmentation,
  y=Frequency,
] {
Segmentation,Frequency
re-vo-l-ution,0.5598714170908117
re-vo-lu-tion,0.14920975087061344
re-vo-lut-ion,0.11438521296544334
re-vol-ut-ion,0.04098580230377712
re-v-ol-ution,0.029734797749799088
rev-ol-ut-ion,0.029734797749799088
rev-o-l-ution,0.012590409858023038
re-v-olut-ion,0.008840075006697026
rev-o-lu-tion,0.008572193945888026
r-ev-ol-ution,0.006429145459416019
rev-ol-u-tion,0.006429145459416019
rev-olut-i-on,0.005893383337798017
rev-olu-t-ion,0.005625502276989017
re-vol-u-tion,0.004553978033753014
rev-ol-uti-on,0.0032145727297080095
rev-olut-io-n,0.0032145727297080095
re-vol-uti-on,0.0029466916688990086
rev-o-lut-ion,0.0029466916688990086
rev-olu-ti-on,0.0029466916688990086
re-v-olu-tion,0.0008036431824270024
r-ev-olut-ion,0.0005357621216180017
r-e-v-olution,0.0002678810608090008
r-e-vol-ution,0.0002678810608090008
};
\legend{\uni~~, \bpedropout}
\end{axis}
\end{tikzpicture}
\label{fig:seg_hist_fine_appendix_b}
\end{subfigure}
\hfill 
\begin{subfigure}[t]{0.48\linewidth}
\centering
\begin{tikzpicture}[baseline=(current bounding box.north)]
\begin{axis}[
  ybar,
  bar width=1pt,
  width=\linewidth,
  height=5cm,
  ymin=0,
  ylabel={Probability (\%)},
  scaled y ticks=false,
  yticklabel={
  \pgfmathparse{100*\tick}%
  \pgfmathprintnumber[fixed,precision=0]{\pgfmathresult}%
  },
  xlabel={},
  symbolic x coords={
    rev-o-l-ut-ion,
    r-ev-olu-t-ion,
    r-e-vol-uti-on,
    rev-ol-u-t-ion,
    re-vol-ut-i-on,
    rev-olu-ti-o-n,
    re-vo-l-uti-on,
    r-e-v-olut-ion,
    r-ev-ol-uti-on,
    r-e-v-olu-tion,
    re-vol-ut-io-n,
    re-v-o-l-ution,
    re-v-ol-u-tion,
    r-e-vol-u-tion,
    re-vo-lut-io-n,
    rev-olu-t-io-n,
    re-vo-l-u-tion,
    re-v-olu-ti-on,
    re-vol-u-ti-on,
    rev-ol-u-ti-on,
    re-vo-lut-i-on,
    r-e-vo-lu-tion,
    r-ev-olut-io-n,
    r-e-vo-l-ution,
    rev-o-lu-t-ion,
    re-v-ol-uti-on,
    r-ev-o-lu-tion,
    re-v-olu-t-ion,
    re-vo-lu-ti-on,
    rev-o-lut-i-on,
    re-v-olut-i-on,
    r-ev-o-l-ution,
    rev-o-lu-ti-on,
    rev-o-l-u-tion,
    rev-ol-ut-i-on,
    rev-olut-i-o-n,
    re-vol-u-t-ion,
    rev-o-lut-io-n,
    r-ev-ol-u-tion,
    rev-olu-t-i-on,
    re-vo-l-ut-ion,
    r-ev-olu-ti-on,
    re-vol-uti-o-n,
    r-e-v-ol-ution,
    rev-ol-uti-o-n,
    re-v-o-lu-tion,
    rev-ol-ut-io-n,
    r-ev-olut-i-on,
    r-e-vo-lut-ion,
    r-ev-o-lut-ion,
    r-ev-ol-ut-ion,
    r-e-vol-ut-ion,
    rev-o-l-uti-on,
    re-v-olut-io-n,
    re-vo-lu-t-ion,
    re-v-o-lut-ion,
    re-v-ol-ut-ion,
  },
  xtick=data,
  xticklabel style={font=\ttfamily\tiny, rotate=90, anchor=east},
  enlarge x limits=0.01,
  legend style={at={(0.5,1.02)}, anchor=south, legend columns=2},
  yticklabel style={
  /pgf/number format/.cd,
    fixed,
    precision=1
},
scaled y ticks=false,
]
\addplot+[bar shift=-1pt,fill=scRed, draw=scRed] table[
  col sep=comma,
  x=Segmentation,
  y=Frequency
] {
Segmentation,Frequency
rev-o-l-ut-ion,0.017851829812555786
r-ev-olu-t-ion,0.018149360309431716
r-e-vol-uti-on,0.01894277496776753
rev-ol-u-t-ion,0.017752652980263812
re-vol-ut-i-on,0.015669939502132302
rev-olu-ti-o-n,0.01953783596151939
re-vo-l-uti-on,0.017951006644847764
r-e-v-olut-ion,0.016860061489636022
r-ev-ol-uti-on,0.017752652980263812
r-e-v-olu-tion,0.015471585837548348
re-vol-ut-io-n,0.016562530992760092
re-v-o-l-ution,0.019438659129227414
re-v-ol-u-tion,0.017256768818803926
r-e-vol-u-tion,0.014777348011504512
re-vo-lut-io-n,0.01924030546464346
rev-olu-t-io-n,0.017256768818803926
re-vo-l-u-tion,0.019141128632351484
re-v-olu-ti-on,0.017157591986511952
re-vol-u-ti-on,0.015669939502132302
rev-ol-u-ti-on,0.018248537141723694
re-vo-lut-i-on,0.017256768818803926
r-e-vo-lu-tion,0.01428146385004463
r-ev-olut-io-n,0.01636417732817614
r-e-vo-l-ution,0.017455122483387882
rev-o-lu-t-ion,0.02013289695527125
re-v-ol-uti-on,0.018843598135475554
r-ev-o-lu-tion,0.01517405534067242
re-v-olu-t-ion,0.016165823663592184
re-vo-lu-ti-on,0.019041951800059506
rev-o-lut-i-on,0.016860061489636022
re-v-olut-i-on,0.018050183477139742
r-ev-o-l-ution,0.01606664683130021
rev-o-lu-ti-on,0.014678171179212536
rev-o-l-u-tion,0.01636417732817614
rev-ol-ut-i-on,0.01924030546464346
rev-olut-i-o-n,0.0186452444708916
re-vol-u-t-ion,0.018248537141723694
rev-o-lut-io-n,0.017554299315679856
r-ev-ol-u-tion,0.016959238321927996
rev-olu-t-i-on,0.016860061489636022
re-vo-l-ut-ion,0.017653476147971834
r-ev-olu-ti-on,0.015273232172964396
re-vol-uti-o-n,0.01606664683130021
r-e-v-ol-ution,0.017355945651095904
rev-ol-uti-o-n,0.019736189626103344
re-v-o-lu-tion,0.018843598135475554
rev-ol-ut-io-n,0.017752652980263812
r-ev-olut-i-on,0.017752652980263812
r-e-vo-lut-ion,0.016661707825052066
r-ev-o-lut-ion,0.02211643360111078
r-ev-ol-ut-ion,0.016165823663592184
r-e-vol-ut-ion,0.01894277496776753
rev-o-l-uti-on,0.016860061489636022
re-v-olut-io-n,0.018149360309431716
re-vo-lu-t-ion,0.017752652980263812
re-v-o-lut-ion,0.019934543290687296
re-v-ol-ut-ion,0.018050183477139742

};
\addplot+[bar shift=+1pt,fill=black!60, draw=black!60] table[
  col sep=comma,
  x=Segmentation,
  y=Frequency
] {
Segmentation,Frequency
re-vo-l-ut-ion,0.26255380200860834
re-vo-lu-t-ion,0.22812051649928264
re-v-ol-ut-ion,0.13199426111908177
rev-ol-u-t-ion,0.08177905308464849
re-v-o-l-ution,0.03730272596843615
rev-o-lu-t-ion,0.03443328550932568
re-vo-lu-ti-on,0.02582496413199426
rev-ol-u-ti-on,0.024390243902439025
rev-ol-ut-i-on,0.021520803443328552
rev-o-l-ut-ion,0.018651362984218076
re-vol-u-ti-on,0.015781922525107604
re-vol-u-t-ion,0.01291248206599713
r-ev-o-l-ution,0.010043041606886656
re-vol-ut-io-n,0.010043041606886656
r-ev-ol-ut-ion,0.00860832137733142
re-vol-ut-i-on,0.00860832137733142
rev-ol-ut-io-n,0.00860832137733142
r-e-vo-l-ution,0.007173601147776184
re-v-olu-t-ion,0.007173601147776184
re-v-ol-u-tion,0.005738880918220947
re-v-olut-i-on,0.005738880918220947
r-ev-o-lut-ion,0.0028694404591104736
re-v-o-lu-tion,0.0028694404591104736
re-v-o-lut-ion,0.0028694404591104736
re-v-olut-io-n,0.0028694404591104736
re-vo-l-u-tion,0.0028694404591104736
re-vo-lut-io-n,0.0028694404591104736
r-e-v-ol-ution,0.0014347202295552368
r-e-v-olut-ion,0.0014347202295552368
r-e-vol-ut-ion,0.0014347202295552368
r-ev-ol-u-tion,0.0014347202295552368
r-ev-ol-uti-on,0.0014347202295552368
r-ev-olu-t-ion,0.0014347202295552368
re-v-ol-uti-on,0.0014347202295552368
re-v-olu-ti-on,0.0014347202295552368
re-vo-lut-i-on,0.0014347202295552368
rev-o-l-u-tion,0.0014347202295552368
rev-o-lu-ti-on,0.0014347202295552368

};
\legend{\uni~~, \bpedropout}
\end{axis}
\end{tikzpicture}
\label{fig:seg_hist_fine_appendix}
\end{subfigure}

\caption{Histogram comparison of segmentation probabilities for two sampling schemes (\uni vs.\ \bpedropout with $p_{drop}=0.5$).}
\label{app:fig:seg_histograms_appendix}

\end{figure}
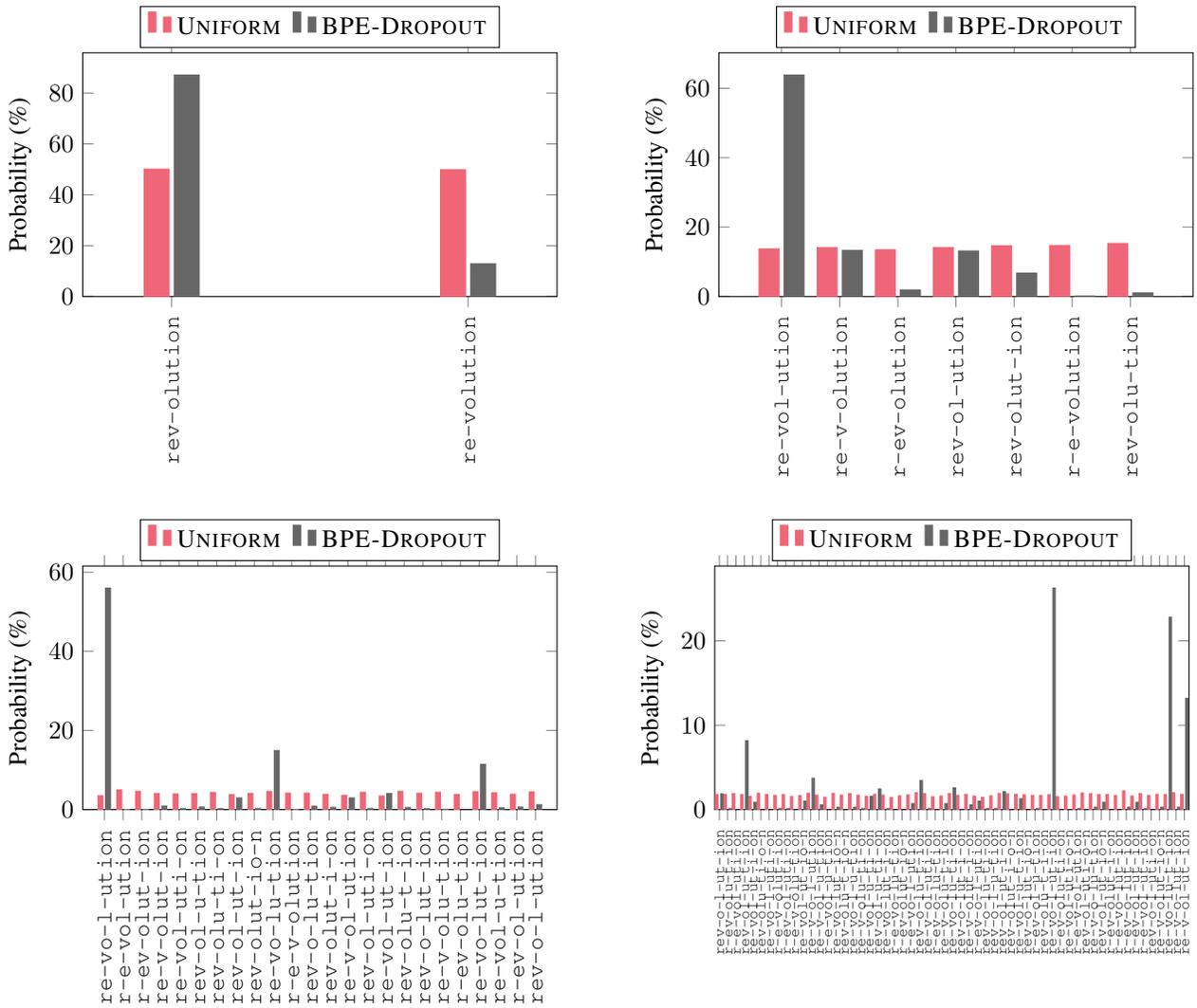

Next, we empirically evaluate the performance of BPE-dropout compared to our methods for \llamaoneb fine-tuned on \langgame and \cute. \Cref{fig:app_finetuning_robustness_bpedropout_avg} shows average accuracy under canonical and stochastic evaluation with \stok tokenisation. We observe that \bpedropout improves robustness to non-canonical tokenisations compared to canonical fine-tuning. \Cref{fig:app_finetuning_robustness_bpedropout_adv} reports robustness under adversarial tokenisation. The greedy search for the adversarial tokenisation is performed for 10 iterations, $d_k=2$ neighbours, starting from canonical initial tokenisations $v_0$. Here, BPE-dropout again improves robustness over canonical fine-tuning, while more uniform stochastic tokenisation schemes remain competitive or stronger at comparable perturbation strengths. As above, note that a direct comparison between $p_{\text{drop}}$ and $\alpha_{\text{fine}}$ is not possible. In \Cref{fig:avg_robustness_qwen} and \Cref{fig:app_finetuning_robustness_bpedropout_adv_qwen} we repeated the experiments for the Qwen-3 tokeniser and LoRA fine-tune a \qwenzerosixb model.

\begin{figure}[h]
    \centering
    \setlength{\figureheight}{0.2\linewidth}
    \setlength{\figurewidth}{.4\linewidth}

    \pgfplotsset{
        x tick label style={font=\scriptsize},
        y tick label style={rotate=90, font=\scriptsize},
        ylabel={\small Avg.\ test accuracy $\rightarrow$},
        xlabel={\small $\peval$},
        scale only axis,
        tick align=outside,
        tick pos=left,
        every axis plot/.append style={mark size=3pt},
        axis x line*=bottom,
        axis y line*=left,
        axis line style={draw=none},
        grid style={solid},
    }

    \begin{subfigure}[b]{0.48\textwidth}
        \centering
        \begin{tikzpicture}
\begin{axis}[
  height=\figureheight,
  width=\figurewidth,
  ybar,
  bar width=6pt,
  enlarge x limits=0.12,
  symbolic x coords={0,0.1,0.5,1},
  xtick=data,
  xtick style={color=black},
  ytick style={color=black},
  ymajorgrids,
  ymin=0.6,
  ymax=.99,
  yticklabel={
    \pgfmathparse{100*\tick}\pgfmathprintnumber[fixed,precision=0]{\pgfmathresult}\%
  },
  grid style={scGrey},
  legend cell align={left},
  legend style={
    fill opacity=1,
    legend columns=2,
    draw opacity=1,
    text opacity=1,
    at={(0.2,1.03)},
    anchor=south west,
    draw=none
  },
]

\addplot+[very thick, draw=black, fill=black, bar shift=-12pt]
coordinates { (0,0.952) ({0.1},0.947) ({0.5},0.901) (1,0.813) };
\addlegendentry{$p_{\mathrm{drop}}=0.0$}

\addplot+[very thick, draw=scBlue!90!black, fill=scBlue!90!black, bar shift=-4pt]
coordinates { (0,0.956) ({0.1},0.949) ({0.5},0.935) (1,0.904) };
\addlegendentry{$p_{\mathrm{drop}}=0.1$}

\addplot+[very thick, draw=scBlue!60, fill=scBlue!60, bar shift=4pt]
coordinates { (0,0.958) ({0.1},0.953) ({0.5},0.953) (1,0.940) };
\addlegendentry{$p_{\mathrm{drop}}=0.25$}

\addplot+[very thick, draw=scBlue!30, fill=scBlue!30, bar shift=12pt]
coordinates { (0,0.956) ({0.1},0.955) ({0.5},0.966) (1,0.969) };
\addlegendentry{$p_{\mathrm{drop}}=0.5$}

\end{axis}
\end{tikzpicture}
        \caption{\langgame}
    \end{subfigure}
    \hfill
    \begin{subfigure}[b]{0.48\textwidth}
        \centering
        \begin{tikzpicture}
\begin{axis}[
  height=\figureheight,
  width=\figurewidth,
  ybar,
  bar width=6pt,
  enlarge x limits=0.12,
  symbolic x coords={0,0.1,0.5,1},
  xtick=data,
  xtick style={color=black},
  ytick style={color=black},
  ymajorgrids,
  ymin=0.6,
  ymax=.99,
  yticklabel={
    \pgfmathparse{100*\tick}\pgfmathprintnumber[fixed,precision=0]{\pgfmathresult}\%
  },
  grid style={scGrey},
  legend cell align={left},
  legend style={
    fill opacity=1,
    legend columns=2,
    draw opacity=1,
    text opacity=1,
    at={(0.2,1.03)},
    anchor=south west,
    draw=none
  },
]

\addplot+[very thick, draw=black, fill=black, bar shift=-12pt]
coordinates { (0,0.893) ({0.1},0.884) ({0.5},0.817) (1,0.716) };
\addlegendentry{$p_{\mathrm{drop}}=0.0$}

\addplot+[very thick, draw=scBlue!90!black, fill=scBlue!90!black, bar shift=-4pt]
coordinates { (0,0.892) ({0.1},0.892) ({0.5},0.893) (1,0.875) };
\addlegendentry{$p_{\mathrm{drop}}=0.1$}

\addplot+[very thick, draw=scBlue!60, fill=scBlue!60, bar shift=4pt]
coordinates { (0,0.909) ({0.1},0.909) ({0.5},0.907) (1,0.901) };
\addlegendentry{$p_{\mathrm{drop}}=0.25$}

\addplot+[very thick, draw=scBlue!30, fill=scBlue!30, bar shift=12pt]
coordinates { (0,0.903) ({0.1},0.911) ({0.5},0.918) (1,0.925) };
\addlegendentry{$p_{\mathrm{drop}}=0.5$}

\end{axis}
\end{tikzpicture}
        \caption{\cute}
    \end{subfigure}

\caption{Average accuracy of \llamaoneb fine-tuned with \bpedropout and evaluated on \langgame under canonical and \stok tokenisation. \bpedropout improves robustness to non-canonical tokenisations. However, choosing an appropriate dropout rate $p_{\text{drop}}$ is less intuitive, since the resulting number of splits depends on the underlying input text and its merge structure.}    \label{fig:app_finetuning_robustness_bpedropout_avg}
\end{figure}
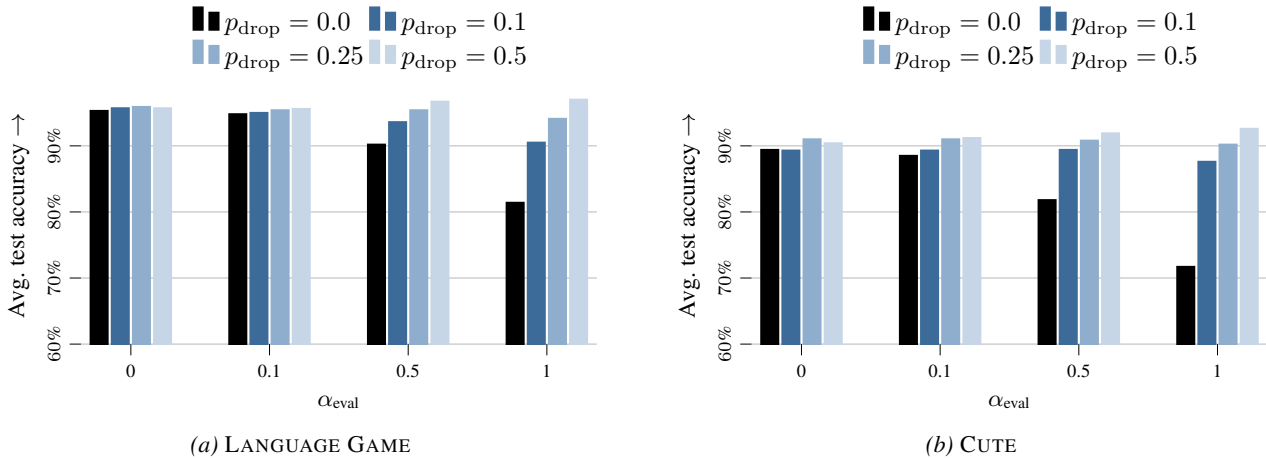

\begin{figure}[h]
    \centering
    \begin{subfigure}{.85\textwidth}\centering
    \appheader{Data set: \langgame}
      \pgfplotsset{
  colormap={scScheme}{
    color(0pt)=(white)
color(1pt)=(scCyan!60)
  }
}

\begin{tikzpicture}
\begin{groupplot}[
  group style={
    group size=3 by 1,
    horizontal sep=.5cm,
  },
  width=0.33\textwidth,
  colormap name=scScheme,
  point meta min=0,
  point meta max=1,
  xlabel={$\pfine/p_{drop}$},
  xmin=-0.5, xmax=3.5,
  ymin=-0.5, ymax=3.5,
  xtick={0,1,2,3},
  xticklabels={-, 0.1, 0.25, 0.5},
  ytick={0,1,2,3},
  yticklabels={
    Canonical,
    \unik,
    \bpedropout,
    \stokuni
  },
 y dir=reverse,
  tick align=outside,
  tick pos=left,
  every axis label/.append style={font=\small},
  every node near coord/.append style={
    font=\footnotesize, color=black, anchor=center,
    /pgf/number format/.cd, fixed, fixed zerofill
  },
]
\nextgroupplot[title={Canonical tokenisation }]
\addplot [
  matrix plot,
  point meta=explicit,
  mesh/cols=4,
nodes near coords={\printmynumber\pgfplotspointmeta},
]
table [meta=Value] {
x y Value
0 0 0.948
1 0 nan
2 0 nan
3 0 nan
0 1 nan
1 1 0.954
2 1 0.962
3 1 0.964
0 2 nan
1 2 0.956
2 2 0.960
3 2 0.956
0 3 nan
1 3 0.954
2 3 0.962
3 3 0.950
};

\nextgroupplot[
  title={Adversarial tokenisation \\ $v_0$ canonical},
  title style={align=center},
  colorbar style={
    ylabel={},
    y dir=normal,
  },
  yticklabels={},
]
\addplot [
  matrix plot,
  point meta=explicit,
  mesh/cols=4,
nodes near coords={\printmynumber\pgfplotspointmeta},
]
table [meta=Value] {
x y Value
0 0 0.102
1 0 nan
2 0 nan
3 0 nan
0 1 nan
1 1 0.804
2 1 0.732
3 1 0.764
0 2 nan
1 2 0.418
2 2 0.532
3 2 0.706
0 3 nan
1 3 0.616
2 3 0.636
3 3 0.696
};

\end{groupplot}
\end{tikzpicture}
    \end{subfigure}\par\vspace{3mm}
    \begin{subfigure}{.85\textwidth}\centering
    \appheader{Data set: \cute}
      \pgfplotsset{
  colormap={scScheme}{
    color(0pt)=(white)
color(1pt)=(scCyan!60)
  }
}

\begin{tikzpicture}
\begin{groupplot}[
  group style={
    group size=3 by 1,
    horizontal sep=.5cm,
  },
  width=0.33\textwidth,
  colormap name=scScheme,
  point meta min=0,
  point meta max=1,
  xlabel={$\pfine / p_{drop}$},
  xmin=-0.5, xmax=3.5,
  ymin=-0.5, ymax=3.5,
  xtick={0,1,2,3},
  xticklabels={-, 0.1, 0.25, 0.5},
  ytick={0,1,2,3},
  yticklabels={
    Canonical,
    \unik,
    \bpedropout,
    \stokuni
  },
  y dir=reverse,
  tick align=outside,
  tick pos=left,
  every axis label/.append style={font=\small},
  every node near coord/.append style={
    font=\footnotesize, color=black, anchor=center,
    /pgf/number format/.cd, fixed, fixed zerofill
  },
]

\nextgroupplot[title={Canonical tokenisation }]
\addplot [
  matrix plot,
  point meta=explicit,
  mesh/cols=4,
nodes near coords={\printmynumber\pgfplotspointmeta},
]
table [meta=Value] {
x y Value
0 0 0.900
1 0 nan
2 0 nan
3 0 nan
0 1 nan
1 1 0.892
2 1 0.913
3 1 0.914
0 2 nan
1 2 0.894
2 2 0.907
3 2 0.897
0 3 nan
1 3 0.905
2 3 0.906
3 3 0.910
};

\nextgroupplot[
  title={Adversarial tokenisation \\ $v_0$ canonical},
  title style={align=center},
  colorbar style={
    ylabel={},
    y dir=normal,
  },
  yticklabels={},
]
\addplot [
  matrix plot,
  point meta=explicit,
  mesh/cols=4,
nodes near coords={\printmynumber\pgfplotspointmeta},
]table [meta=Value] {
x y Value
0 0 0.303
1 0 nan
2 0 nan
3 0 nan
0 1 nan
1 1 0.708
2 1 0.783
3 1 0.817
0 2 nan
1 2 0.677
2 2 0.755
3 2 0.812
0 3 nan
1 3 0.712
2 3 0.794
3 3 0.785
};

\end{groupplot}
\end{tikzpicture}
    \end{subfigure}\par\vspace{3mm}
\caption{Accuracy of \llamaoneb under canonical and adversarial tokenisation on \langgame and \cute for stochastic fine-tuning with \bpedropout, \stokuni, and \unik. All stochastic fine-tuning methods improve robustness compared to canonical fine-tuning. Note that $\pfine$ and $p_{drop}$ cannot be compared directly as they induce different number of splits. }
\label{fig:app_finetuning_robustness_bpedropout_adv}
\end{figure}

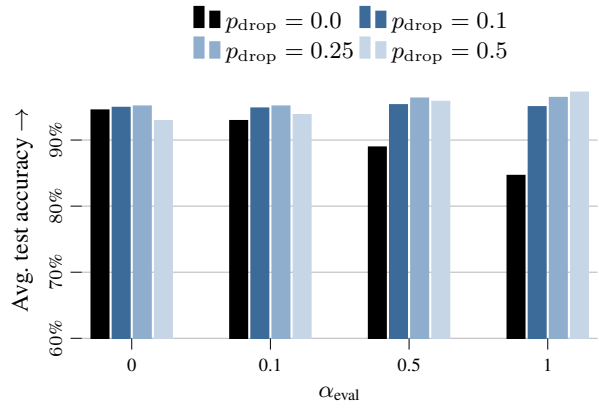
\begin{figure}[h]
    \centering

    \centering\footnotesize
    \setlength{\figureheight}{0.2\linewidth}
    \setlength{\figurewidth}{0.4\linewidth}

    \pgfplotsset{
        x tick label style={font=\scriptsize},
        y tick label style={rotate=90, font=\scriptsize},
        ylabel={\small Avg.\ test accuracy $\rightarrow$},
        xlabel={\small $\peval$},
        scale only axis,
        tick align=outside,
        tick pos=left,
        every axis plot/.append style={mark size=3pt},
        axis x line*=bottom,
        axis y line*=left,
        axis line style={draw=none},
        grid style={solid},
    }
    \begin{subfigure}{.85\textwidth}\centering
    \appheader{Data set: \langgame}

    \begin{tikzpicture}
\begin{axis}[
  height=\figureheight,
  width=\figurewidth,
  ybar,
  bar width=6pt,
  enlarge x limits=0.12,
  symbolic x coords={0,0.1,0.5,1},
  xtick=data,
  xtick style={color=black},
  ytick style={color=black},
  ymajorgrids,
  ymin=0.6,
  ymax=.99,
  yticklabel={
    \pgfmathparse{100*\tick}\pgfmathprintnumber[fixed,precision=0]{\pgfmathresult}\%
  },
  grid style={scGrey},
  legend cell align={left},
  legend style={
    fill opacity=1,
    legend columns=2,
    draw opacity=1,
    text opacity=1,
    at={(0.2,1.03)},
    anchor=south west,
    draw=none
  },
]

\addplot+[very thick, draw=black, fill=black, bar shift=-12pt]
coordinates { (0,0.944) ({0.1},0.928) ({0.5},0.888) (1,0.845) };
\addlegendentry{$p_{\mathrm{drop}}=0.0$}

\addplot+[very thick, draw=scBlue!90!black, fill=scBlue!90!black, bar shift=-4pt]
coordinates { (0,0.948) ({0.1},0.947) ({0.5},0.952) (1,0.949) };
\addlegendentry{$p_{\mathrm{drop}}=0.1$}

\addplot+[very thick, draw=scBlue!60, fill=scBlue!60, bar shift=4pt]
coordinates { (0,0.950) ({0.1},0.950) ({0.5},0.962) (1,0.963) };
\addlegendentry{$p_{\mathrm{drop}}=0.25$}

\addplot+[very thick, draw=scBlue!30, fill=scBlue!30, bar shift=12pt]
coordinates { (0,0.928) ({0.1},0.937) ({0.5},0.957) (1,0.971) };
\addlegendentry{$p_{\mathrm{drop}}=0.5$}

\end{axis}
\end{tikzpicture}
    \end{subfigure}\par\vspace{3mm}
    
    \caption{Average perturbation robustness of \qwenzerosixb fine-tuned on \langgame. }
    \label{fig:avg_robustness_qwen}
    \vspace*{-1em}
\end{figure}

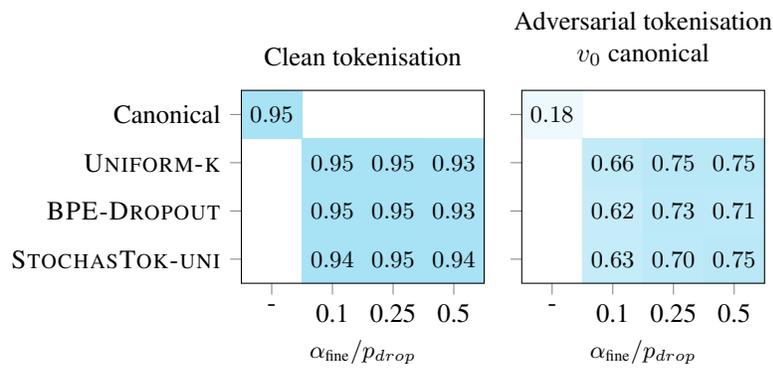
\begin{figure}[h]
    \centering
    \begin{subfigure}{.85\textwidth}\centering
      \pgfplotsset{
  colormap={scScheme}{
    color(0pt)=(white)
color(1pt)=(scCyan!60)
  }
}

\begin{tikzpicture}
\begin{groupplot}[
  group style={
    group size=3 by 1,
    horizontal sep=.5cm,
  },
  width=0.33\textwidth,
  colormap name=scScheme,
  point meta min=0,
  point meta max=1,
  xlabel={$\pfine/p_{drop}$},
  xmin=-0.5, xmax=3.5,
  ymin=-0.5, ymax=3.5,
  xtick={0,1,2,3},
  xticklabels={-, 0.1, 0.25, 0.5},
  ytick={0,1,2,3},
  yticklabels={
    Canonical,
    \unik,
    \bpedropout,
    \stokuni
  },
 y dir=reverse,
  tick align=outside,
  tick pos=left,
  every axis label/.append style={font=\small},
  every node near coord/.append style={
    font=\footnotesize, color=black, anchor=center,
    /pgf/number format/.cd, fixed, fixed zerofill
  },
]
\nextgroupplot[title={Clean tokenisation}]
\addplot [
  matrix plot,
  point meta=explicit,
  mesh/cols=4,
  nodes near coords={\printmynumber\pgfplotspointmeta},
]
table [meta=Value] {
x y Value
0 0 0.946
1 0 nan
2 0 nan
3 0 nan
0 1 nan
1 1 0.950
2 1 0.948
3 1 0.930
0 2 nan
1 2 0.948
2 2 0.950
3 2 0.928
0 3 nan
1 3 0.936
2 3 0.946
3 3 0.942
};

\nextgroupplot[
  title={Adversarial tokenisation \\ $v_0$ canonical},
  title style={align=center},
  colorbar style={
    ylabel={},
    y dir=normal,
  },
  yticklabels={},
]
\addplot [
  matrix plot,
  point meta=explicit,
  mesh/cols=4,
  nodes near coords={\printmynumber\pgfplotspointmeta},
]
table [meta=Value] {
x y Value
0 0 0.180
1 0 nan
2 0 nan
3 0 nan
0 1 nan
1 1 0.656
2 1 0.748
3 1 0.752
0 2 nan
1 2 0.620
2 2 0.726
3 2 0.708
0 3 nan
1 3 0.628
2 3 0.700
3 3 0.748
};

\end{groupplot}
\end{tikzpicture}
    \end{subfigure}\par\vspace{3mm}

\caption{Canonical and adversarial perturbation robustness of \qwenzerosixb fine-tuned on \langgame.  }
\label{fig:app_finetuning_robustness_bpedropout_adv_qwen}
\end{figure}

\FloatBarrier
\subsubsection{Additional Results for Gemma 3}
\label{app:larger_vocab}

\begin{figure}[h]
    \centering

    \centering\footnotesize
    \setlength{\figureheight}{0.2\linewidth}
    \setlength{\figurewidth}{0.4\linewidth}

    \pgfplotsset{
        x tick label style={font=\scriptsize},
        y tick label style={rotate=90, font=\scriptsize},
        ylabel={\small Avg.\ test accuracy $\rightarrow$},
        xlabel={\small $\peval$},
        scale only axis,
        tick align=outside,
        tick pos=left,
        every axis plot/.append style={mark size=3pt},
        axis x line*=bottom,
        axis y line*=left,
        axis line style={draw=none},
        grid style={solid},
    }

    \begin{tikzpicture}
\begin{axis}[
  height=\figureheight,
  width=\figurewidth,
  ybar,
  bar width=6pt,
  enlarge x limits=0.1,
  symbolic x coords={0,0.1,0.5,1},
  xtick=data,
  xtick style={color=black},
  ytick style={color=black},
  ymajorgrids,
  ymin=0.1, ymax=1,
  yticklabel={
    \pgfmathparse{100*\tick}\pgfmathprintnumber[fixed,precision=0]{\pgfmathresult}\%
  },
  grid style={scGrey},
  legend cell align={left},
  legend style={fill opacity=1,legend columns=4,draw opacity=1,text opacity=1,at={(0.05,1.03)},anchor=south west,draw=none,font=\scriptsize},
]

\addplot+[very thick, draw=black, fill=black, bar shift=-12pt]
coordinates {
  (0,0.904)
  ({0.1},0.890)
  ({0.5},0.767)
  (1,0.635)
};
\addlegendentry{$\pfine=0.0$}

\addplot+[very thick, draw=scRed!90, fill=scRed!90, bar shift=-4pt]
coordinates {
  (0,0.920)
  ({0.1},0.918)
  ({0.5},0.890)
  (1,0.814)
};
\addlegendentry{$\pfine=0.1$}

\addplot+[very thick, draw=scRed!60, fill=scRed!60, bar shift=4pt]
coordinates {
  (0,0.904)
  ({0.1},0.915)
  ({0.5},0.917)
  (1,0.876)
};
\addlegendentry{$\pfine=0.2$}

\addplot+[very thick, draw=scRed!30, fill=scRed!30, bar shift=12pt]
coordinates {
  (0,0.898)
  ({0.1},0.912)
  ({0.5},0.924)
  (1,0.918)
};
\addlegendentry{$\pfine=0.5$}

\end{axis}
\end{tikzpicture}

    \caption{Average perturbation robustness of \gemmaoneb (vocabulary size of 262k) fine-tuned and evaluated with \stok on \langgame. }
    \label{fig:gemma_avg_robustness}
    \vspace*{-1em}
\end{figure}
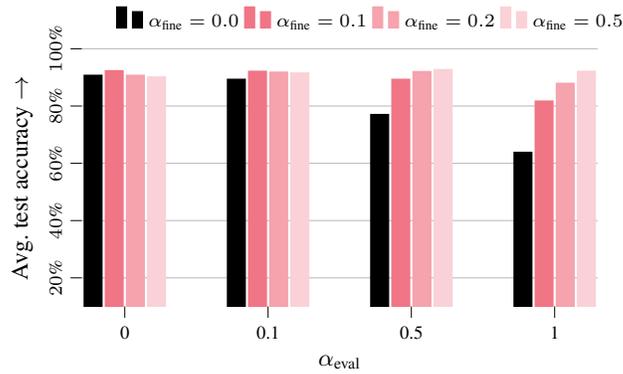

\FloatBarrier
\subsubsection{Ablation on Embedding Distances} \label{app:embedding_distances}

We measure the mean normalised distance between the hidden representations of canonical and non-canonical tokenisations of the same word across layers. We use the 1,000 most frequent English words and uniformly sample 50 non-canonical tokenisations for each word. \Cref{fig:app:embeddings} shows the results for \llamaoneb after canonical and stochastic fine-tuning on \langgame. While canonical fine-tuning does not reduce these distances compared to the zero-shot model, stochastic fine-tuning consistently yields smaller distances, especially in deeper layers. This indicates reduced local sensitivity to tokenisation changes and is consistent with the improved robustness to adversarial tokenisation observed in the main paper.

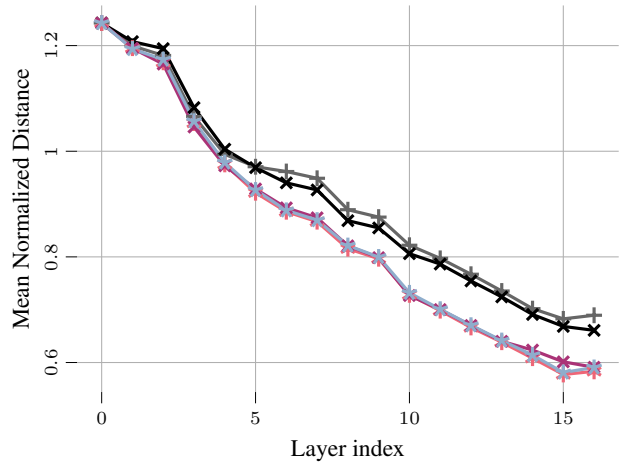
\begin{figure}[h]
\centering\footnotesize
\setlength{\figureheight}{0.3\linewidth}
\setlength{\figurewidth}{0.3\linewidth}

\pgfplotsset{
    x tick label style={font=\scriptsize}, 
    y tick label style={rotate=90, font=\scriptsize}, 
    ylabel={\small Avg.\ normalised $L^2$ distance},
    xlabel={\small Layer index},
    scale only axis,
    tick align=outside,
    tick pos=left,
    every axis plot/.append style={mark size=3pt},
    axis x line*=bottom,
    axis y line*=left,
    axis line style={draw=none},
    grid style={solid},
    }

    \begin{tikzpicture}

\definecolor{crimson2143940}{RGB}{214,39,40}
\definecolor{darkgrey176}{RGB}{176,176,176}
\definecolor{darkorange25512714}{RGB}{255,127,14}
\definecolor{forestgreen4416044}{RGB}{44,160,44}
\definecolor{lightgrey204}{RGB}{204,204,204}
\definecolor{steelblue31119180}{RGB}{31,119,180}

\begin{axis}[
height=\figureheight,
tick align=outside,
tick pos=left,
width=1.4\figurewidth,
x grid style={darkgrey176},
xlabel={Layer index},
xmajorgrids,
xmin=-0.8, xmax=16.8,
xtick style={color=black},
y grid style={darkgrey176},
ylabel={Mean Normalized Distance},
ymajorgrids,
ymin=0.544115109614192, ymax=1.2760232851211,
ytick style={color=black}
]

\addplot [very thick, black!60, mark=+, mark size=3, mark options={solid}]
table {0 1.24275473168896
1 1.1994502856492
2 1.18132154620082
3 1.06580915502899
4 0.994086667959609
5 0.970669055469738
6 0.961605629804653
7 0.948818279744436
8 0.889304043797881
9 0.875279104874344
10 0.822037937943555
11 0.797211967657938
12 0.767143938800746
13 0.735403699258887
14 0.702033488292879
15 0.682648545683333
16 0.689477823572047
};
\label{plt:plt-distance-zero}

\addplot [very thick, black, mark=x, mark size=3, mark options={solid}]
table {0 1.24275473168896
1 1.20753966754084
2 1.19430839099152
3 1.08264113038717
4 1.00392055409712
5 0.968978777529078
6 0.940147914789058
7 0.926698693087845
8 0.868561993230382
9 0.854808284167699
10 0.80604757280316
11 0.786358546941446
12 0.754412852745057
13 0.724249813167277
14 0.690966909560374
15 0.66835123967626
16 0.660883864678152
};
\label{plt:plt-distance-canon}

\addplot [very thick, scRed, mark=asterisk, mark size=3, mark options={solid}]
table {0 1.24275473168896
1 1.19519255047422
2 1.16844725606458
3 1.05334605986874
4 0.975111982029811
5 0.921468879661238
6 0.885510961758417
7 0.866686761796036
8 0.815457337367003
9 0.796534382854753
10 0.728161535126798
11 0.69820197223967
12 0.666419278888861
13 0.638472247472088
14 0.608251627790481
15 0.577383663046324
16 0.583277909744289
};
\label{plt:plt-distance-stok}
 
\addplot [very thick, scPurple, mark=x, mark size=3, mark options={solid}]
table {0 1.24275473168896
1 1.19491398775558
2 1.1650336802025
3 1.04589182897999
4 0.972488717329734
5 0.928632978310314
6 0.892809718909716
7 0.873968594208942
8 0.820914912451555
9 0.79850428820794
10 0.726875138370681
11 0.700208728558152
12 0.670100534970903
13 0.640134202285812
14 0.623593838515263
15 0.601290568170964
16 0.591005443479242
};
\label{plt:plt-distance-stokuni}

\addplot [very thick, scBlue!60, mark=star, mark size=3, mark options={solid}]
table {0 1.24275473168896
1 1.19362558153507
2 1.17346355115114
3 1.05681803563399
4 0.978663868034483
5 0.926609573372812
6 0.888089731716792
7 0.869611444163437
8 0.821451911150879
9 0.799861603859144
10 0.731889973094953
11 0.700968964374885
12 0.670631269868678
13 0.641616042401287
14 0.614124808349489
15 0.581110588139456
16 0.5906775328291
};
\label{plt:plt-distance-unik}

\end{axis}

\end{tikzpicture}
    \caption{Canonical fine-tuning (\ref{plt:plt-distance-canon}) does not affect the distances of alternative tokenisations compared to the zero-shot \llamaoneb (\ref{plt:plt-distance-zero}). Stochastic tokenisation (\stok \ref{plt:plt-distance-stok}, \stokuni \ref{plt:plt-distance-stokuni}, \unik \ref{plt:plt-distance-unik}) reduce distances in deeper layers. 
    }
    \label{fig:app:embeddings}
\end{figure}

\newpage

\subsection{Additional results for in-context learning}

\subsubsection{Example Prompt for Language Game}

\textit{Canonical context + perturbed question ($\alpha_{\text{query}}=1$).}

\begin{quote}
    \ttfamily
    \tokensequence{{Below},{\;are},{\;solved},{\;examples},{.},{\;For},{\;the},{\;final},{\;question},{,},{\;predict},{\;the},{\;correct},{\;option.}}

    \vspace{0.5em}
    \tokensequence{{Which},{\;choice},{\;has},{\;the},{\;most},{\;letter},{\;'},{b},{'s},{?},{\;These},{\;are},{\;the},{\;possible},{\;option},{\;words},{:},{\;[},{\;up},{,},{\;baby},{,},{\;rain},{,},{\;glad},{].},{\;Answer},{:},{\;},{\;baby}}

    \vspace{0.5em}
    \tokensequence{{Which},{ option},{ ends},{ with},{ 'o'},{?},{ The},{ possible},{ option},{ strings},{:},{ [},{ no},{,},{ hard},{,},{ by},{,},{ happen},{].},{ Answer},{:},{ },{ no}}

    \vspace{0.5em}
    \tokensequence{{What},{ word},{ is},{ the},{ shortest},{?},{ The},{ available},{ option},{ strings},{ are},{:},{ [},{ row},{,},{ suffix},{,},{ gather},{,},{ energy},{].},{ Answer},{:},{ },{ row}}

    \vspace{0.5em}
    {\itshape
    \tokensequence{{T},{h},{e},{ possible},{ option},{ words},{ a},{re},{:},{ [},{ straight},{,},{ ca},{r},{r},{y},{,},{ di},{ffer},{,},{ h},{alf},{].},{ Which},{ of},{ the},{ p},{ossible},{ option},{ words},{ has},{ the},{ most},{ let},{te},{r},{ '},{t},{'s},{?},{ Ans},{we},{r},{:}}}
\end{quote}

\textit{Perturbed context ($\alpha_{\text{context}}=1$) + perturbed question ($\alpha_{\text{query}}=1$).}

\begin{quote}
    \ttfamily
    \tokensequence{{Below},{ are},{ solved},{ examples},{.},{ For},{ the},{ final},{ question},{,},{ predict},{ the},{ correct},{ option.}}

    \vspace{0.5em}
    \tokensequence{{Wh},{i},{ch},{ choice},{ has},{ the},{ most},{ le},{t},{t},{er},{ '},{b},{'s},{?},{ These},{ are},{ the},{ possible},{ option},{ words},{:},{ [},{ up},{,},{ baby},{,},{ rain},{,},{ g},{lad},{].},{ A},{n},{sw},{er},{:},{ },{ baby}}

    \vspace{0.5em}
    \tokensequence{{Which},{ option},{ end},{s},{ with},{ 'o'},{?},{ The},{ possible},{ option},{ s},{tring},{s},{:},{ [},{ no},{,},{ hard},{,},{ b},{y},{,},{ happen},{].},{ Answer},{:},{ },{ no}}

    \vspace{0.5em}
    \tokensequence{{Wh},{at},{ word},{ i},{s},{ the},{ shortest},{?},{ The},{ available},{ option},{ string},{s},{ are},{:},{ [},{ r},{o},{w},{,},{ suffix},{,},{ g},{ather},{,},{ e},{n},{ergy},{].},{ Ans},{wer},{:},{ },{ row}}

    \vspace{0.5em}
    {\itshape
    \tokensequence{{T},{h},{e},{ possible},{ option},{ words},{ a},{re},{:},{ [},{ straight},{,},{ ca},{r},{r},{y},{,},{ di},{ffer},{,},{ h},{alf},{].},{ Which},{ of},{ the},{ p},{ossible},{ option},{ words},{ has},{ the},{ most},{ let},{te},{r},{ '},{t},{'s},{?},{ Ans},{we},{r},{:}}}
\end{quote}

\begin{table*}[h!]
\centering\small
\renewcommand{\arraystretch}{1.1}
\setlength{\tabcolsep}{4pt}

\renewcommand{\HeatVMax}{0.3} %
\renewcommand{\HeatMinInt}{0} %
\renewcommand{\HeatSpanInt}{70} %

\caption{Accuracy (\%) of \llamaeightb under canonical versus uniformly sampled query tokenisation across data sets, with $K=10$ few-shot examples. 
}
\label{tab:app:uniformity_avg_icl}
\resizebox{\textwidth}{!}{%

\begin{tabular}{@{}l ll *{7}{cc} @{}}
\toprule
& \multirow{2}{*}{\shortstack{\textbf{Context}\\\textbf{tokenisation}}} &
  \multirow{2}{*}{\textbf{$\picl$}} &
  \multicolumn{2}{c}{\textsc{Language Game}} &
  \multicolumn{2}{c}{\textsc{ARC-E}} &
  \multicolumn{2}{c}{\textsc{COPA}} &
  \multicolumn{2}{c}{\textsc{CSQA}} &
  \multicolumn{2}{c}{\textsc{HellaSwag}} &
  \multicolumn{2}{c}{\textsc{PIQA}} &
  \multicolumn{2}{c}{\textsc{SocialIQA}} \\
\cmidrule(l){4-5}\cmidrule(l){6-7}\cmidrule(l){8-9}\cmidrule(l){10-11}\cmidrule(l){12-13}\cmidrule(l){14-15}\cmidrule(l){16-17}
& & &
\textbf{canon.} & \textbf{unif.} &
\textbf{canon.} & \textbf{unif.} &
\textbf{canon.} & \textbf{unif.} &
\textbf{canon.} & \textbf{unif.} &
\textbf{canon.} & \textbf{unif.} &
\textbf{canon.} & \textbf{unif.} &
\textbf{canon.} & \textbf{unif.} \\
\midrule

& Canonical & -  & \PairDelta{0.839}{0.703} & \PairDelta{0.781}{0.713} & \PairDelta{0.920}{0.846} & \PairDelta{0.683}{0.630} & \PairDelta{0.548}{0.502} & \PairDelta{0.814}{0.795} & \PairDelta{0.524}{0.451} \\ \midrule
& \stok & 0.1  & \PairDelta{0.846}{0.756} & \PairDelta{0.767}{0.717} & \PairDelta{0.920}{0.860} & \PairDelta{0.688}{0.634} & \PairDelta{0.550}{0.509} & \PairDelta{0.817}{0.796} & \PairDelta{0.528}{0.463} \\
& \stok & 0.5  & \PairDelta{0.859}{0.802} & \PairDelta{0.775}{0.732} & \PairDelta{0.911}{0.868} & \PairDelta{0.678}{0.641} & \PairDelta{0.543}{0.509} & \PairDelta{0.819}{0.797} & \PairDelta{0.537}{0.474} \\
& \stok & 1.0  & \PairDelta{0.877}{0.818} & \PairDelta{0.774}{0.740} & \PairDelta{0.913}{0.870} & \PairDelta{0.671}{0.639} & \PairDelta{0.544}{0.509} & \PairDelta{0.819}{0.800} & \PairDelta{0.538}{0.475} \\ \midrule
& \stokuni & 0.1  & \PairDelta{0.851}{0.756} & \PairDelta{0.764}{0.715} & \PairDelta{0.920}{0.858} & \PairDelta{0.686}{0.635} & \PairDelta{0.551}{0.508} & \PairDelta{0.816}{0.795} & \PairDelta{0.530}{0.462} \\
& \stokuni & 0.5  & \PairDelta{0.868}{0.804} & \PairDelta{0.771}{0.731} & \PairDelta{0.911}{0.870} & \PairDelta{0.674}{0.637} & \PairDelta{0.542}{0.508} & \PairDelta{0.820}{0.798} & \PairDelta{0.536}{0.473} \\
& \stokuni & 1.0  & \PairDelta{0.871}{0.813} & \PairDelta{0.773}{0.738} & \PairDelta{0.916}{0.873} & \PairDelta{0.667}{0.639} & \PairDelta{0.542}{0.510} & \PairDelta{0.819}{0.801} & \PairDelta{0.534}{0.477} \\ \midrule
& \unik & 0.1  & \PairDelta{0.851}{0.745} & \PairDelta{0.773}{0.715} & \PairDelta{0.920}{0.852} & \PairDelta{0.687}{0.634} & \PairDelta{0.548}{0.505} & \PairDelta{0.814}{0.795} & \PairDelta{0.527}{0.462} \\
& \unik & 0.5  & \PairDelta{0.862}{0.790} & \PairDelta{0.771}{0.725} & \PairDelta{0.914}{0.871} & \PairDelta{0.680}{0.637} & \PairDelta{0.544}{0.510} & \PairDelta{0.818}{0.798} & \PairDelta{0.530}{0.478} \\
& \unik & 1.0  & \PairDelta{0.872}{0.808} & \PairDelta{0.771}{0.737} & \PairDelta{0.911}{0.872} & \PairDelta{0.676}{0.640} & \PairDelta{0.543}{0.509} & \PairDelta{0.819}{0.800} & \PairDelta{0.528}{0.489} \\ \midrule
& \uni & -  & \PairDelta{0.850}{0.817} & \PairDelta{0.776}{0.752} & \PairDelta{0.902}{0.870} & \PairDelta{0.623}{0.638} & \PairDelta{0.541}{0.510} & \PairDelta{0.819}{0.804} & \PairDelta{0.514}{0.498} \\
\bottomrule
\end{tabular}
} 
\end{table*}

\renewcommand{\HeatVMax}{1}
\newcommand{\HeatMaxMixInt}{100}
\newcommand{\HeatMute}{90}

\newcommand{\HeatCellFewshot}[1]{%
  \begingroup
  \pgfmathsetmacro{\d}{#1}%
  \pgfmathsetmacro{\x}{max(0,min(\d,\HeatVMax))}%
  \pgfmathsetmacro{\t}{\x/\HeatVMax}%
  \pgfmathparse{(\t < 0.5) ? 1 : 0}%
  \let\isLower\pgfmathresult
    \pgfmathsetmacro{\p}{(\t)}%
    \pgfmathtruncatemacro{\I}{round(\HeatMaxMixInt*\p)}%
    \edef\HeatColorSpec{scLightCyan!\I}%
  \edef\HeatColorSpec{\HeatColorSpec}%
  \expandafter\cellcolor\expandafter{\HeatColorSpec}%
  \pgfmathparse{100*(\d)}%
  \pgfmathprintnumber[fixed,fixed zerofill,precision=\DeltaPrec]{\pgfmathresult}%
  \endgroup
}

\newcommand{\AdvRowShort}[3]{#1 & \HeatCellFewshot{#2} & \HeatCellFewshot{#3} \\}

\begin{table}[h!]
\centering\small

\caption{ Accuracy (\%) on \langgame\ using \llamaeightb\ under canonical and adversarial tokenisations, with $5$ context examples.
}
\label{tab:adv_icl_langgame}
   
    \begin{tabular}{ccccc|c}
        \toprule
         & \textsc{Canonical}& \stok & \stokuni & \unik & \uni \\ \midrule
         Canonical & \HeatCellFewshot{0.838}&  \HeatCellFewshot{0.892}&  \HeatCellFewshot{0.896}&  \HeatCellFewshot{0.886}& \HeatCellFewshot{0.84}\\
         Adversarial & \HeatCellFewshot{0.188}&  \HeatCellFewshot{0.504}&  \HeatCellFewshot{0.510}&  \HeatCellFewshot{0.492}& \HeatCellFewshot{0.454} \\
    \bottomrule

    \end{tabular}
\end{table}

\newpage

\newpage

\end{document}